\documentclass[journal]{IEEEtran}

\usepackage{amsmath}
\usepackage{amssymb}
\usepackage{cases}
\usepackage{graphicx}
\usepackage{nicematrix}

\usepackage[sort,compress]{cite}
\usepackage{booktabs}

\usepackage{amsthm}

\newtheorem{theorem}{Theorem}

\newtheorem{corollary}{Corollary}

\theoremstyle{definition}

\newtheorem{remark}{{Remark}}

\usepackage{threeparttable}    
\usepackage{makecell}

\usepackage{arydshln} 
\usepackage{mdframed} 

\usepackage{setspace}
\usepackage[]{xcolor}
\usepackage{subcaption}
\usepackage{wrapfig}
\usepackage{multirow}
\usepackage{mathtools}
\usepackage{empheq}
\usepackage{color}
\usepackage{bbm}
\usepackage{xspace}

\usepackage[shortlabels]{enumitem}

\usepackage{tikz} 
\usepackage{tkz-graph}
\usepackage{tkz-berge}
\usetikzlibrary{backgrounds,fit,shapes,snakes,arrows,shapes.geometric,positioning,automata}
\usetikzlibrary{intersections,patterns,shapes.misc}
\usetikzlibrary{decorations.pathmorphing}
\usepackage{pgfplots}
\pgfplotsset{compat=1.16}

\tikzstyle{block} = [rectangle, rounded corners, minimum width=3cm, minimum height=1cm,text centered, draw=black, fill=red!30]
\tikzstyle{new} = [rectangle, rounded corners, minimum width=1cm, minimum
height=1cm,text centered, draw=black, fill=blue!10!white, dashed]
\tikzstyle{arrow} = [thick,->,>=stealth]
\usetikzlibrary{calc, quotes}
\usetikzlibrary{arrows.meta}

\usepackage[linesnumbered,ruled,vlined]{algorithm2e}

\usepackage[]{algpseudocode}

\usepackage{fontawesome}

\usepackage{soul}
\usepackage{comment}
\usepackage[]{url}
\usepackage[normalem]{ulem}
\useunder{\uline}{\ul}{}

\usepackage[pdftex, colorlinks=true, linkcolor=blue, citecolor = blue, urlcolor = cyan, filecolor=black, pagebackref=false, hypertexnames=false]{hyperref}
\usepackage{cleveref}
\crefname{problem}{Problem}{Problems}
\crefname{example}{Example}{Examples}
\crefname{section}{Sec.}{Secs.}
\Crefname{section}{Section}{Sections}
\Crefname{table}{Table}{Tables}
\crefname{table}{Table}{Tabs.}
\crefname{figure}{Fig.}{Figs.}
\crefname{algorithm}{Algorithm}{Algorithms}
\crefname{remark}{Remark}{Remarks}
\crefname{assumption}{Assumption}{Assumption}
\crefname{theorem}{Theorem}{Theorems}
\crefname{proposition}{Proposition}{Propositions}
\crefname{lemma}{Lemma}{Lemmas}
\crefname{corollary}{Corollary}{Corollaries}
\crefname{assumption}{Assumption}{Assumptions}
\crefname{definition}{Definition}{Definitions}

\usepackage{bbding}

\usepackage{nicematrix} 

\def\bfa{\mathbf{a}}
\def\bfb{\mathbf{b}}

\def\bff{\mathbf{f}}
\def\bfg{\mathbf{g}}

\def\bfn{\mathbf{n}}
\def\bfp{\mathbf{p}}

\def\bfs{\mathbf{s}}

\def\bfu{\mathbf{u}}
\def\bfv{\mathbf{v}}

\def\bfy{\mathbf{y}}
\def\bfz{\mathbf{z}}

\def\bfA{\mathbf{A}}

\def\bfC{\mathbf{C}}
\def\bfD{\mathbf{D}}
\def\bfE{\mathbf{E}}
\def\bfF{\mathbf{F}}
\def\bfG{\mathbf{G}}
\def\bfH{\mathbf{H}}
\def\bfI{\mathbf{I}}
\def\bfJ{\mathbf{J}}
\def\bfK{\mathbf{K}}
\def\bfL{\mathbf{L}}
\def\bfM{\mathbf{M}}
\def\bfN{\mathbf{N}}

\def\bfP{\mathbf{P}}
\def\bfQ{\mathbf{Q}}
\def\bfR{\mathbf{R}}
\def\bfS{\mathbf{S}}
\def\bfT{\mathbf{T}}

\def\bfX{\mathbf{X}}

\def\bfZo{\mathbf{0}}

\def\bfvarepsilon {\boldsymbol{\varepsilon}}
\def\bfepsilon {\boldsymbol{\epsilon}}

\def\bfomega{\boldsymbol{\omega}}

\def\bftheta{\boldsymbol{\theta}}

\def\bfPhi{\boldsymbol{\Phi}}

\def\scaleF#1{\scalebox{1.}{$#1$}}

\DeclareMathOperator{\Exp}{Exp}
\DeclareMathOperator{\Log}{Log}

\DeclareMathOperator{\se}{\mathfrak{se}}

\newcommand{\bmat}{\begin{bmatrix}}
\newcommand{\emat}{\end{bmatrix}}

\providecommand{\optional}[1]{{}}

\providecommand{\techreport}[1]{{}}  

\usepackage{tabularx}
\usepackage{etoc}

\def\rfD{\mathrm{D}}

\def\GSD{\mathbf{G}}
\UseRawInputEncoding

\def\revised #1{{#1}}
\def\revised #1{        {#1}}

\def\SO3{{\mathbf{SO}(3)}}

\begin{document}
\title{Equivariant Filter Transformations for Consistent and Efficient Visual--Inertial Navigation}


\author{Chungeng Tian,~\IEEEmembership{Graduate Student Member,~IEEE}, Fenghua He,~\IEEEmembership{Member,~IEEE}, and Ning Hao,~\IEEEmembership{Member,~IEEE}
\thanks{ The authors are with School of Astronautics, Harbin Institute of Technology, Harbin, 150000, China. {(email: tcghit@outlook.com; hefenghua@hit.edu.cn; haoning@hit.edu.cn).} Corresponding author: {Fenghua He}.}%
}

\maketitle
\thispagestyle{plain}
\pagestyle{plain}

\begin{abstract}
This paper presents an equivariant filter (EqF) transformation approach for visual--inertial navigation. 
By establishing analytical links between EqFs with different symmetries, the proposed approach enables systematic consistency design and efficient implementation.
First, we formalize the mapping from the global system state to the local error-state and prove that it induces a \revised{first-order linear transformation relationship} between the error-states of any two EqFs. Second, we derive transformation laws for the associated linearized error-state systems and unobservable subspaces. These results yield a general consistency design principle: for any unobservable system, a consistent EqF with a state-independent unobservable subspace can be synthesized by transforming the local coordinate chart, thereby avoiding ad hoc symmetry analysis. 
Third, to mitigate the computational burden arising from the non-block-diagonal Jacobians required for consistency, 
we propose two efficient implementation strategies. These strategies exploit the Jacobians of a simpler EqF with block-diagonal structure to accelerate covariance operations while preserving consistency.
Extensive Monte Carlo simulations and real-world experiments validate the proposed approach in terms of both accuracy and runtime. The code is available at \url{https://github.com/HITCSC/T-EqF}.
\end{abstract}



\section{Introduction}

\IEEEPARstart
{V}{isual}--Inertial Navigation Systems (VINS), which fuse high-frequency inertial measurements with visual observations, have become a core technology for autonomous robots operating in GPS-denied environments\cite{jin2020camera,genevaOpenVINSResearchPlatform2020,duan2022stereo,wang2025intelligent,wang2025givlslam}. \revised{Among filter-based nonlinear estimators, Equivariant Filters (EqFs), which exploit system symmetries to lift the estimation problem from the state manifold to a suitable symmetry group, have recently attracted increasing attention\cite{vangoorEquivariantFilterEqF2023, ge2022equivariant,fornasier2022EquivariantFilterDesign,tao2025equivariant}. This view unifies classical error-state Kalman filter (ESKF) \cite{eskf,trawny2005indirect,chang2017indirect} and invariant extended Kalman filter (I-EKF) \cite{barrauInvariantExtendedKalman2017,ge2026difference, shi2023invariant,he2025invariant,xia2025invariantekfbased} constructions and provides a general route to geometric estimator design.
Because EqF relies on equivariant properties rather than requiring the original system model to be defined on a Lie group, it provides a flexible framework for designing estimators for complex navigation systems.}



Despite this progress, two critical problems remain. First, universal design criteria for consistent EqFs are still lacking. \revised{In VINS, state-dependent unobservable subspaces may cause the estimator to acquire spurious information along unobservable directions, leading to overconfident estimates \cite{huangAnalysisImprovementConsistency2008}. Second, symmetries or transformations introduced for consistency may destroy the block-diagonal Jacobian structure used by efficient VINS filters, causing expensive covariance propagation when many landmarks are maintained \cite{yangDecoupledRightInvariant2022}. Although recent works \cite{tian2025teskf,chen2024visualinertial, tian2025unobservable} address this trade-off for specific estimators, a general principle for arbitrary EqFs is still missing.}

To address these challenges, this paper develops an EqF transformation framework for consistent and efficient VINS. The central idea is to characterize each EqF by its \textit{global-local map}, which maps the system state in a neighborhood of the current estimate to the local error-state. We prove that, for any two EqFs of the same system, the composition of their global-local maps induces a transformation relationship between their error-states. This result provides a unified mechanism for transferring filter structures across EqFs: the associated linearized dynamics, measurement Jacobians, and unobservable subspaces follow explicit transformation laws.
Building on this relationship, we turn EqF consistency design into a coordinate-transformation problem: by transforming the local coordinate chart, the unobservable subspace can be made state-independent, yielding consistent EqFs without ad hoc symmetry selection. 
The same transformation viewpoint also reveals how the covariance operations of a more consistent target EqF can be related to those of a structurally simple auxiliary EqF.
To reduce the computational burden arising from non-block-diagonal Jacobians, we further propose two implementation strategies 
that use this transformation relationship to exploit the block-diagonal Jacobians of the auxiliary EqF while preserving the consistency of the target EqF.

The main contributions of this paper are summarized as follows:
\begin{itemize}
\item A universal transformation relationship, with analytical expressions, that links error-states and transfers the associated Jacobians and observability properties across arbitrary EqFs.
\item A systematic methodology for designing consistent EqFs, applicable to general unobservable systems.
\item Two efficient EqF implementations that achieve runtime comparable to ESKF.
\end{itemize}

\revised{
\section{Related Work}}
\label{sec:related_work}

\subsection{\revised{Equivariant and Symmetry-Based Filtering}}
\revised{ Filter-based VINS commonly uses ESKF formulations, which represent attitude errors through local perturbations on the special orthogonal group \cite{eskf}. I-EKF further exploits invariant errors on the special Euclidean group to improve the structure of the linearized dynamics and the observability properties of navigation estimators\cite{barrauInvariantExtendedKalman2017}. Recent work \cite{barrauGeometryNavigationProblems2023} extends the invariant filtering framework using the two-frame group to better handle states in the body frame.
 EqF generalizes the symmetry-based view by using equivariant system properties rather than requiring the original model itself to be a Lie-group system\cite{vangoorEquivariantFilterEqF2023}.
 EqF remains applicable even when the system dynamics do not strictly satisfy the group-affine property \cite{fornasier2022EquivariantFilterDesign}.
 Moreover, if the system exhibits output equivariance, the EqF formulation can further reduce the linearization error of the measurement model \cite{vangoor2023eqvio}.
 Recent work \cite{fornasier2025equivariant} shows that ESKF and I-EKF, including the two-frame extension \cite{barrauGeometryNavigationProblems2023}, can be cast as EqFs. Building on this view, this paper derives explicit transformation relationships between arbitrary EqFs, clarifying how independently presented formulations are connected.}







\subsection{\revised{Inconsistency Issue in VINS}}
\revised{Estimator inconsistency in VINS has been widely attributed to the mismatch between the true unobservable directions and those of the linearized system\cite{huangAnalysisImprovementConsistency2008}. First-estimate Jacobian methods \cite{huangAnalysisImprovementConsistency2008,chenFEJ2ConsistentVisualInertial2022} and observability-constrained methods \cite{heschConsistencyAnalysisImprovement2014} enforce the correct null space by modifying the Jacobians, but may sacrifice Jacobian optimality. Another line of work seeks to make the unobservable subspace state-independent so that the observability properties are preserved under linearization \cite{hao2025transformationbased,song2024affine}. This can be achieved through symmetry selection (e.g., I-EKF \cite{barrauEKFSLAMAlgorithmConsistency2016}, Robocentric formulation \cite{huai2022robocentric}, and EqVIO \cite{vangoor2023eqvio}) or error-state transformations \cite{tian2025teskf}.
However, a general design principle for choosing appropriate symmetries to achieve consistency is still lacking. Moreover, current transformation-based methods are limited to improving the consistency of ESKF.}

\subsection{\revised{Efficient Implementation for Consistent VINS}}
\revised{The computational issue considered in this paper is therefore not generic VINS acceleration \cite{pengSqrt2025}, but rather the bottleneck introduced by consistency-improving methods. In VINS, symmetries or transformations used to maintain a state-independent unobservable subspace may destroy the block-diagonal propagation structure by coupling landmark errors with navigation errors. To mitigate this issue, \cite{yangDecoupledRightInvariant2022} combines invariant errors with first-estimate Jacobians to decouple landmark uncertainty, but this comes at the cost of Jacobian optimality and estimation accuracy. Recent works \cite{chen2024visualinertial,tian2025teskf} identify the relationship between the error-states of ESKF and I-EKF and exploit the block-diagonal propagation of ESKF to accelerate I-EKF.
Our work is closely related to \cite{tian2025teskf} in consistency design and efficient computation, but extends the transformation idea from ESKF/I-EKF to arbitrary EqFs and offers more choices for efficient implementation.}

\section{Preliminaries}
\label{sec:preliminaries}%
\subsection{Manifolds and Lie Groups}

For a smooth manifold $\mathcal{M}$, let $\mathrm{T}_{\xi}\mathcal{M}$ denote the tangent space at $\xi \in \mathcal{M}$. Given a differentiable mapping $h: \mathcal{M} \to \mathcal{N}$ between two smooth manifolds, its differential at $\xi' \in \mathcal{M}$ is denoted by the linear map \begin{equation}
\mathrm{D}_{\xi|\xi'} h(\xi): \mathrm{T}_{\xi'}\mathcal{M} \to \mathrm{T}_{h(\xi')}\mathcal{N},
\end{equation} which pushes forward a tangent vector $v \in \mathrm{T}_{\xi'}\mathcal{M}$ to $\mathrm{D}_{\xi|\xi'} h(\xi)[v] \in \mathrm{T}_{h(\xi')}\mathcal{N}$.

A Lie group $\mathbf{G}$ is a smooth manifold endowed with a group structure. For any $X, Y \in \mathbf{G}$, group multiplication is denoted by $XY$, the inverse of $X$ by $X^{-1}$, and the identity element by $I$. The tangent space at the identity, $\mathfrak{g} := \mathrm{T}_I \mathbf{G}$, is the Lie algebra of $\mathbf{G}$. As an $n$-dimensional vector space, $\mathfrak{g}$ is isomorphic to $\mathbb{R}^n$ via the \emph{wedge} operator $(\cdot)^\wedge: \mathbb{R}^n \to \mathfrak{g}$, with its inverse denoted by the \emph{vee} operator $(\cdot)^\vee: \mathfrak{g} \to \mathbb{R}^n$. The group adjoint map $\mathrm{Ad}_X: \mathfrak{g} \to \mathfrak{g}$ is defined as the differential of the conjugation mapping at the identity, i.e., \begin{equation}
\mathrm{Ad}_X [v] := \mathrm{D}_{Y|I} (X Y X^{-1}) [v].
\end{equation} The corresponding adjoint matrix $\mathbf{Ad}_X^{\vee} \in  \mathbb{R}^{n\times n}$ is defined such that $\mathbf{Ad}_X^{\vee} \mathbf{v} := (\mathrm{Ad}_X [\mathbf{v}^\wedge])^\vee$ for any $\mathbf{v} \in \mathbb{R}^n$.

\subsection{Right Group Actions}
A right group action of a Lie group $\mathbf{G}$ on a smooth manifold $\mathcal{M}$ is a smooth map $\phi:  \mathbf{G}\times \mathcal{M} \to \mathcal{M}$ that satisfies \begin{align}
\phi(I,\xi) &= \xi, \\
\phi(X_2,\phi( X_1,\xi)) &= \phi(X_1 X_2,\xi),
\end{align} for all $\xi \in \mathcal{M}$ and $X_1, X_2 \in \mathbf{G}$. For a fixed $X \in \mathbf{G}$, the partial map $\phi_X: \mathcal{M} \to \mathcal{M}$ is defined as $\phi_X(\xi) := \phi(X,\xi)$. Similarly, for a fixed $\xi \in \mathcal{M}$, the partial map $\phi_\xi: \mathbf{G} \to \mathcal{M}$ is defined as $\phi_\xi(X) := \phi(X,\xi)$.

\section{Transformation between EqFs}
\label{sec:transformation}
In this section, we present the transformation relationships between EqFs constructed from different symmetries. We first briefly introduce the VINS model and the EqF-based estimator, then derive the transformation between the error-states of different EqFs, and finally show that this relationship extends to the linearized error-state system.
\subsection{System Model}
Consider a robot equipped with an \revised{inertial measurement unit (IMU)} and a camera moving in a 3D environment. The system state to be estimated is
           \begin{align}
                \xi= (\bfR, \bfv, \bfp, \bfb_{\omega},\bfb_a, \bff_1,\dots,\bff_m),
        \end{align}
where $\bfR \in \SO3$, $\bfv \in \mathbb{R}^3$, and $\bfp \in \mathbb{R}^3$ are the orientation, velocity, and position of the IMU in the world frame, respectively;
$\bfb_{\omega} \in \mathbb{R}^3$ and $\bfb_a \in \mathbb{R}^3$ are the gyroscope bias and accelerometer bias, respectively;
 $\bff_i \in \mathbb{R}^3$ is the position of the $i$-th landmark in the world frame.

The system dynamics are given by:
\begin{equation}
        \left\{
        \begin{array}{ll}
                \dot{\bfR}  = \bfR [\bfomega_m - \bfb_{\omega}- \bfn_{\omega}]_\times, \quad\quad &\dot{\bfb}_{\omega}  = \bfn_{\text{w}\omega}, \\
                \dot{\bfv}  = \bfR (\bfa_m - \bfb_a - \bfn_a) + \bfg, \quad \quad & \dot{\bfb}_a  = \bfn_{\text{w}a},  \\
                \dot{\bfp}  = \bfv, \quad  \quad  &  \dot{\bff}_i  = \bfZo, 
        \end{array} \right.
\label{equ:system_dynamics}
\end{equation}
where $\bfomega_m$, $\bfa_m$ are the raw angular velocity and acceleration, respectively; $\bfg$ is the gravity vector; $\bfn = [\bfn_{\omega}^\top, \bfn_a^\top, \bfn_{\text{w}\omega}^\top, \bfn_{\text{w}a}^\top]^\top$
        is the process noise, with $\bfn_{\omega}$, $\bfn_a$ being the measurement noise of gyroscope and accelerometer, and $\bfn_{\text{w}\omega}$, $\bfn_{\text{w}a}$ being the random walk noise of gyroscope and accelerometer biases.

When the camera observes the $i$-th landmark, the measurement model is given by
\begin{equation}
        \bfz_i = h(^{I}\bfp_{f_i}) + \bfepsilon_i,
        \label{equ:measurement_model}
\end{equation}
where $^{I}\bfp_{f_i} = \bfR^\top (\bff_i - \bfp)$ and
$h = \pi \circ \Upsilon$. $\Upsilon: \mathbb{R}^3 \to \mathbb{R}^3$ transforms the landmark position from the IMU frame to the camera frame, and $\pi: \mathbb{R}^3 \to \mathbb{R}^2$ is the projection function. By stacking all visual measurements, we obtain the measurement vector $
        \bfz = \begin{bmatrix}
                \cdots &
                \bfz_i^\top
                \cdots  
        \end{bmatrix}^\top
$, and the corresponding noise vector $\bfepsilon = \begin{bmatrix}
                \cdots &
                \bfepsilon_i^\top & 
                \cdots 
        \end{bmatrix}^\top$.
        
\subsection{Equivariant Filter} \label{sec:sd_eqf}
We briefly review the SD-EqF, formulated on the semi-direct bias group
$\GSD:= \mathbf{SE}_{2}(3)\ltimes \mathfrak{se}(3) \times \mathbb{R}^{3m}$,
to illustrate the fundamental components of the EqF framework, including the symmetry group, group action, global error, and local coordinate.

\textbf{Symmetry group:}
An element of $\GSD$ is $X = (C, \gamma, p)$, where
$C =(R,a,b) \in \mathbf{SE}_{2}(3)$, $\gamma \in \se(3)$, and
$p \in \mathbb{R}^{3m}$. For any $X_1, X_2 \in \GSD$, the group multiplication is defined as: 
\begin{equation}
X_1 X_2 = (C_1 C_2, \gamma_1 + \mathrm{Ad}_{\Gamma ({C_1})}[\gamma_2], p_1 + p_2),
\end{equation}
where $\Gamma({C_1}) = (R,a) \in \mathbf{SE}(3)$. The identity is $(I,0,0)$ and
$X^{-1} = (C^{-1}, -\mathrm{Ad}_{\Gamma(C)^{-1}}[\gamma], -p)$.

\textbf{Group action:} To define the group action, we partition the system state as $\xi = (\bfA, \bfb, \bff)$, where $\bfA := (\bfR, \bfv, \bfp) \in \mathbf{SE}_2(3)$, $\bfb := (\bfb_{\omega}, \bfb_a) \in \mathbb{R}^{6}$, and $\bff := (\bff_1, \dots, \bff_m) \in \mathbb{R}^{3m}$. The group action $\phi: \GSD \times \mathcal{M} \to \mathcal{M}$ is then defined as: \begin{equation}
\phi(X, \xi) = (\bfA C, \mathbf{Ad}^{\lor}_{\Gamma (C)^{-1}}(\bfb - \gamma^{\lor}), \bff + p).
\end{equation} Given an arbitrary fixed state $\xi^{\circ } \in \mathcal{M}$, state estimation is equivalent to estimating an element $\hat X \in \bfG$, such that the estimated system state is expressed as: \begin{equation}
\hat \xi = \phi(\hat X, \xi^{\circ } ).
\end{equation} In this work, the origin of SD-EqF is chosen as $\xi^{\circ }  = (\bfI, \bf0, \bf0)$.

\textbf{Global error:} Let $\xi$ and $X$ denote the true state and its group representation, while $\hat{\xi}$ and $\hat{X}$ denote their estimated counterparts. The equivariant error is defined as: \begin{equation}
e = \phi(\hat{X}^{-1}, \xi).
\end{equation} Notably, $e = \xi^{\circ } $ if and only if $\hat X = X$. Thus, the global error resides in a neighborhood of $\xi^{\circ } $, i.e., $e \in \mathcal{U}(\xi^{\circ } ) \subset \mathcal{M}$.

\textbf{Local coordinate:} To facilitate estimator design, the global error $e$ is mapped to a vector space via a local coordinate chart $\vartheta: \mathcal{U}(\xi^{\circ }) \to \mathbb{R}^N$: \begin{equation}
\boldsymbol{\varepsilon} = \vartheta(e),
\end{equation}
 satisfying $\vartheta(\xi^{\circ } ) = \mathbf{0}$. 
In this paper, the local coordinate of the global error is referred to as the \textit{error-state} of EqF, as it serves a functionally analogous role to the error-state in ESKF.
For the SD-EqF, the local coordinate is explicitly given by: \begin{equation}
\vartheta(e) = \begin{pmatrix} \log(e_{\mathbf{A}})^\lor, & e_{\mathbf{b}}, & e_{\bff} \end{pmatrix}.
\end{equation}

The linearized error-state system of an EqF plays a crucial role in estimator design, as it governs uncertainty propagation and the computation of state corrections. Rather than deriving it separately through the input action and kinematic lift for each specific EqF \cite{vangoorEquivariantFilterEqF2023}, the next subsection shows that it can be efficiently derived by leveraging the transformation relationship between different EqFs.

\subsection{Transformation between Error-States}
This subsection elaborates on the transformation relationships between different EqFs. We first derive the transformation matrix between the error-states of two arbitrary EqFs and then use it to derive the linearized error-state system of one EqF from that of another.

\begin{figure}[!th]
    \centering
    \resizebox{0.43\textwidth}{!}{\input{tikz/trinity.tex}}
    \caption{Error-state relationship between two arbitrary EqFs. The arrows can be reversed to represent the inverse relationship.}
    \label{fig:error_state_relationship}
\end{figure}

To bridge distinct EqF formulations, we define the \textit{global-local map} $\varphi := \vartheta \circ \phi_{\hat{X}^{-1}}: \mathcal{U}(\hat{\xi}) \to \mathbb{R}^N$, which maps a state in a neighborhood of the current estimate on the manifold to its local error coordinate in the vector space.
This mapping not only characterizes the relationship between the true state and the error-state within a single EqF, but also serves as a bridge across different EqFs, as illustrated in Fig. \ref{fig:error_state_relationship}. Specifically, the interaction between these global-local maps induces a linear transformation between the corresponding error-states, as formalized in the following theorem.
\begin{theorem}
        \label{lemma:0}%
        Given two arbitrary EqFs of the same system associated with $\varphi$ and $\varphi^*$, there exists a \revised{first-order linear transformation relationship} between their error-states:%
        \begin{equation}
                \bfvarepsilon^* = \bfT \bfvarepsilon,
                \label{equ:relationship}
        \end{equation}
        where $\bfT$ is a nonsingular matrix given by%
        \begin{equation}
            \bfT =  \rfD_{\xi|\hat{\xi}} \varphi^* (\xi) \cdot  \rfD_{\bfvarepsilon|\bf0} \varphi^{-1}(\bfvarepsilon).%
        \label{equ:T_general}%
        \end{equation}%
\end{theorem}%
\begin{proof}
Since both EqFs satisfy 
\begin{equation}
         \varphi^{-1}(\bfvarepsilon) = \xi  = {\varphi^{*}}^{-1}(\bfvarepsilon^*),
\end{equation}
we have the following relationship between the two error-states:
\begin{equation}
        \bfvarepsilon^* = \varphi^*\left( \varphi^{-1}(\bfvarepsilon) \right). 
        \label{equ:relationship_proof}
\end{equation}
Note that when $\bfvarepsilon = \bf0$, it follows that $\xi = \hat{\xi}$, and consequently $\bfvarepsilon^* ={\varphi^{*}}(\hat{\xi})= \bf0$. 
By linearizing \eqref{equ:relationship_proof} at $\bfvarepsilon = \bf0$ and neglecting higher-order terms, we obtain the transformation relationship \eqref{equ:relationship}.

To show that $\bfT$ is nonsingular, note that the global-local map $\varphi$ (and similarly $\varphi^*$) is a diffeomorphism, as it is the composition of a coordinate chart $\vartheta$ and a smooth group action $\phi_{\hat{X}^{-1}}$, both of which are diffeomorphisms. By the inverse function theorem and the chain rule, the Jacobians $\rfD_{\xi|\hat{\xi}} \varphi^*(\xi)$ and $\rfD_{\bfvarepsilon|\bf0} \varphi^{-1} (\bfvarepsilon)$ are linear isomorphisms at their respective evaluation points. Consequently, $\bfT$ is the composition of two linear isomorphisms, which is itself a linear isomorphism, thereby ensuring that $\bfT$ is nonsingular.
\end{proof}
\begin{remark}
The transformation matrix inherently depends on the estimated system state $\hat{\xi}$, given that the maps $\phi_{\xi^{\circ}}: \bfG \to \mathcal{M}$ and $\phi_{\xi^{\circ}}^*: \bfG^* \to \mathcal{M}$ are bijective. By definition, the global-local maps $\varphi$ and $\varphi^*$ are parameterized by the estimated group elements $\hat{X} \in \bfG$ and $\hat{X}^* \in \bfG^*$, respectively. As a result, the transformation matrix $\mathbf{T}$ is influenced by the estimated elements. In the context of VINS, $\phi_{\xi^{\circ}}$ and $\phi_{\xi^{\circ}}^*$ are usually bijective. Consequently, $\mathbf{T}$ can be expressed as a function of the estimated system state, i.e., $\mathbf{T} = \mathbf{T}(\hat{\xi})$, which guarantees a well-defined transformation.
\label{remark:1}
\end{remark}

\begin{remark}
    The transformations among various EqF formulations exhibit a consistent algebraic structure characterized by the following properties:
    \begin{itemize}
        \item \textbf{Transitivity}: 
        The transformation matrices satisfy a chain rule (Fig. \ref{fig:transitivity}). For any three EqFs $\varphi$, $\varphi'$, and $\varphi^*$, the maps obey $\mathbf{T}^{\varphi^*}_{\varphi'} = \mathbf{T}^{\varphi^*}_{\varphi} \mathbf{T}^{\varphi}_{\varphi'}$.
        This allows complex transformations to be decomposed into simpler intermediate steps, facilitating the analysis of complex global-local maps.

        \item \textbf{Path independence}: 
        A direct consequence of transitivity is that the final transformation between any two EqFs is independent of the intermediate path.
        In practice, one can therefore use the simplest, most familiar EqF (e.g., ESKF) as a universal baseline for deriving other variants, without compromising the correctness of the results.
    \end{itemize}
\end{remark}

Based on Theorem \ref{lemma:0}, we can establish the transformation relationship of the linearized error-state system  between EqFs, which is summarized in the following corollary.
\begin{corollary}
    Consider two arbitrary EqFs of the same system associated with $\varphi$ and $\varphi^*$. If the linearized error-state system of the EqF for $\varphi$ is given by
    \begin{subequations}
    \begin{align}
    \dot{\bfvarepsilon} &= \bfF \bfvarepsilon + \bfG \bfn, \\
    \tilde{\bfz} &= \bfH \bfvarepsilon + \bfepsilon, 
    \end{align}
    \label{equ:eqf_general}%
    \end{subequations}
    then the linearized error-state system for $\varphi^*$ satisfies
    \begin{subequations}
    \begin{align}
        \dot{\bfvarepsilon}^* &= \bfF^* \bfvarepsilon^* + \bfG^* \bfn, \\
        \tilde{\bfz} &= \bfH^* \bfvarepsilon^* + \bfepsilon,
    \end{align}
    \label{equ:eqf_transformed}%
    \end{subequations}
    where
    \begin{equation}
        \bfF^* = \dot{\bfT} \bfT^{-1} + \bfT \bfF \bfT^{-1},  \bfG^* = \bfT \bfG,  \bfH^* = \bfH \bfT^{-1},
    \end{equation}
    and $\bfT$ is the transformation matrix given by \eqref{equ:T_general}.
    \label{corollary:1}
\end{corollary}
\begin{proof}
   See the supplementary material.
\end{proof}

\begin{figure}[t]
    \centering
    {\resizebox{0.4\textwidth}{!}{\begin{tikzpicture}[
    node distance=1.cm, 
    eqf_box/.style={
        rectangle, 
        draw=myblue, 
        fill=white, 
        very thick, 
        rounded corners=5pt,
        inner sep=8pt,          
        align=left, 
        minimum width=2.7cm,
        minimum height=1.6cm, 
        font=\footnotesize,     
        drop shadow={opacity=0.1},
    },
    node_title/.style={
        font=\small\sffamily\bfseries,
        color=darkgrey,
        inner sep=5pt
    },
    trans_arrow/.style={
        <-, >={Stealth[inset=0pt, length=5pt]}, 
        line width=1.1pt, 
        color=darkgrey
    },
    label_style/.style={
        font=\scriptsize\sffamily, 
        black!90
    }
]
\usetikzlibrary{arrows.meta, positioning, calc, backgrounds, fit, shadows}
\definecolor{myblue}{RGB}{41, 128, 185}
\definecolor{mylightblue}{RGB}{235, 245, 251}
\definecolor{myred}{RGB}{192, 57, 43}
\definecolor{darkgrey}{RGB}{52, 73, 94}
    \node[eqf_box] (phi_star) {
        Error-State: $\boldsymbol{\varepsilon}^*$ \\
        Jacobians: $\mathbf{F}^*, \mathbf{G}^*, \mathbf{H}^*$ \\
        Unobs. Sub: $\mathbf{N}^*$
    };
    \node[node_title, below=2pt of phi_star] (title_star) {EqF with $\varphi^*$};

    \node[eqf_box, right=of phi_star] (phi_mid) {
        Error-State: $\boldsymbol{\varepsilon}$ \\
        Jacobians: $\mathbf{F}, \mathbf{G}, \mathbf{H}$ \\
        Unobs. Sub: $\mathbf{N}$
    };
    \node[node_title, below=2pt of phi_mid] (title_mid) {EqF with $\varphi$};

    \node[eqf_box, right=of phi_mid] (phi_prime) {
        Error-State: $\boldsymbol{\varepsilon}'$ \\
        Jacobians: $\mathbf{F}', \mathbf{G}', \mathbf{H}'$ \\
        Unobs. Sub: $\mathbf{N}'$
    };
    \node[node_title, below=2pt of phi_prime] (title_prime) {EqF with $\varphi'$};

    \draw[trans_arrow] (phi_star.east) -- (phi_mid.west) 
        node[midway, above=-1pt, label_style] {$\mathbf{T}^{\varphi^*}_{\varphi} $};

    \draw[trans_arrow] (phi_mid.east) -- (phi_prime.west) 
        node[midway, above=-1pt, label_style] {$\mathbf{T}^{\varphi}_{\varphi'} $};

    \path (phi_star.north) edge[
        trans_arrow, 
        dashed, 
        myred, 
        bend right=-15, 
        draw opacity=0.8
    ] node[midway, above, yshift=0pt, align=center, font=\scriptsize] {
        \textcolor{myred}{\textbf{Transitivity of Transformation}} \\
        $\mathbf{T}^{\varphi^*}_{\varphi'} = \mathbf{T}^{\varphi^*}_{\varphi} \mathbf{T}^{\varphi}_{\varphi'} $
    } (phi_prime.north);
\end{tikzpicture}}%
    }
    \caption{The transitivity property of transformation among different EqFs.}
    \label{fig:transitivity}
\end{figure}

\subsection{Example: Transformation from ESKF to SD-EqF}
As a special but important case, ESKF can be viewed as an EqF based on the special orthogonal group $\mathbf{SO}(3)\times \mathbb{R}^{12+3m}$ \cite{fornasier2025equivariant}. In ESKF, the global-local map is simple and defined as the general minus $\ominus$ on the group, i.e.,
\begin{equation}
        \begin{split}
                \varphi(\xi) &= \xi \ominus \hat{\xi} \ 
               \label{equ:oplus}\\
               &= \left(
                        \begin{array}{c}
                                \text{Log} (\hat{\bfR}^{-1} \bfR),
                                \bfv - \hat{\bfv},
                                \bfp - \hat{\bfp}, \\
                                \bfb_\omega - \hat{\bfb}_\omega ,
                                \bfb_a - \hat{\bfb}_a ,\\
                                \bff_1 - \hat{\bff}_1 ,
                                \cdots, 
                                \bff_m - \hat{\bff}_m 
                        \end{array}
                \right).
        \end{split}
\end{equation}
Substituting the global-local maps of ESKF and SD-EqF into Theorem \ref{lemma:0}, we can get the explicit expression of the transformation matrix from ESKF to SD-EqF as follows:
\begin{equation}
    \bfT =
     \setlength{\arraycolsep}{4pt}
    \left[
    \begin{NiceArray}{ccccc|cccc}[
        margin, 
        nullify-dots, 
    ]
        \hat{\bfR} & \bfZo & \bfZo & \bfZo & \bfZo & \Block{5-4}{\hspace{3pt}\scaleF{\bfZo}_{15\times3m}} \\
        {[\hat{\bfv}]_\times} \hat{\bfR} & \bfI_3 & \bfZo & \bfZo & \bfZo & &&& \\
        {[\hat{\bfp}]_\times} \hat{\bfR} & \bfZo & \bfI_3 & \bfZo & \bfZo & && &\\
        \bfZo & \bfZo & \bfZo & \hat{\bfR} & \bfZo & &&&\\
        \bfZo & \bfZo & \bfZo & [\hat{\bfv}]_\times \hat{\bfR} & \hat{\bfR} & &&&\\
        \hline
        \Block{2-5} {\scaleF{\bfZo}_{3m\times15}} &&& & & \Block{2-4}{\scaleF{\bfI}_{3m}} \\
        & &&& & &&&\\ 
    \end{NiceArray}\right].
    \label{equ:Transformation_eskf_etg}
\end{equation}

The linearized error-state system of the ESKF is simpler to derive and well-studied. Substituting this system and the transformation into \eqref{equ:eqf_transformed} yields the linearized error-state system of the SD-EqF (see supplementary material), which matches the result from the conventional EqF approach \cite{vangoorEquivariantFilterEqF2023}.
\revised{
In the supplementary material, we further derive the transformation relationships among other EqFs, including I-EKF and the recently proposed tangent-group \cite{fornasier2025equivariant} and scaled-orthogonal-transform-group \cite{vangoor2023eqvio} formulations.
 }


\section{Consistent EqF Design}
\label{sec:consistent_eqf_design}
This section leverages the identified transformation relationship to address inconsistency in EqF-based estimators. We first characterize how the unobservable subspaces of different EqFs evolve under transformation. We then use this relationship to develop a general design methodology that enforces a state-independent unobservable subspace through an appropriate coordinate transformation, thereby achieving consistency without requiring ad hoc symmetry analysis.

\subsection{Inconsistency Issue}
Observability refers to the ability to recover the initial states of a system from all available measurements. The set of states that cannot be uniquely determined constitutes the unobservable subspace. For the time-varying linear system \eqref{equ:eqf_general}, the local observability matrix $\mathcal{O}$ can be used for analysis and is defined as
\begin{equation}
        \mathcal{O}= \begin{bmatrix}
                \mathcal{O}_0^\top & \mathcal{O}^\top_1 & \cdots & \mathcal{O}_{N-1}^\top\\
                \end{bmatrix}^\top,
                \label{equ:observability_matrix}  
\end{equation}
with $\mathcal{O}_0 = \bfH$ and 
$
        \mathcal{O}_{k+1} = \mathcal{O}_{k} \bfF +  \dot {\mathcal{O}}_{k}.
$
The unobservable subspace is spanned by the columns of a basis matrix $\bfN$ with 
\begin{equation}
        \mathcal{O} \bfN = \bf0.
        \label{equ:unobservable_definition}
\end{equation}

By substituting the Jacobians of the SD-EqF into \eqref{equ:observability_matrix}-\eqref{equ:unobservable_definition}, we obtain its basis matrix $\bfN$ as follows:{
\small
\begin{equation}
\setlength{\arraycolsep}{2.9pt}
\begin{array}{rl}
   \mathbf{N} &= 
   \left[
        \begin{array}{c|c}
        \bfN_{\text{pos}} & \bfN_{\text{yaw}} 
        \end{array}
   \right] \\
   &=
   \left[
\begin{array}{ccccccccc}
\bfZo_{3\times 3} & \bfZo_{3\times 3} & \bfI_3 & \bfZo_{3\times 6} & \bfI_3 & \cdots & \bfI_3 \\
\hline
-\bfg^\top & \bfZo_{1\times 3} & \bfZo_{1\times 3} & \bfZo_{1\times 6} & ([{\hat{\bff}_{1}}]_\times \bfg)^\top & \cdots & ([{\hat{\bff}_{m}}]_\times \bfg)^\top
\end{array}
\right]^\top,
\end{array}
\label{equ:unobservable_subspace_eqf}%
\end{equation}}%
where the columns of
$\bfN_{\text{pos}} \in \mathbb{R}^{(15+3m) \times 3}$ correspond to the
unobservable directions of global position, and
$\bfN_{\text{yaw}} \in \mathbb{R}^{(15+3m) \times 1}$, which is state-dependent,
corresponds to the unobservable direction of global yaw.
As discussed in \cite{song2024affine}, if the unobservable directions are state-dependent, the filter tends to acquire spurious information along these unobservable directions, leading to inconsistency.
In the following subsections, we demonstrate that this inconsistency can be resolved by ensuring a state-independent unobservable subspace by transforming the local coordinate charts.

\subsection{Consistent EqF Design}

The unobservable subspaces of different EqFs are related by the transformation matrix, as summarized in the following corollary:
\begin{corollary}
        \label{corollary:NTN}
    Consider two arbitrary EqFs of the same system associated with $\varphi$ and $\varphi^*$. Their local observability matrices $\mathcal{O}$ and $\mathcal{O}^*$ satisfy:
    \begin{equation}
        \mathcal{O}^* = \mathcal{O} \mathbf{T}^{-1},
        \label{equ:MTM}
    \end{equation} where $\mathbf{T}$ is the transformation matrix given by \eqref{equ:T_general}.
       Moreover, if the system is unobservable, let $\mathbf{N}$ be a basis matrix for the unobservable subspace of the EqF associated with $\varphi$. Then, there exists a basis $\mathbf{N}^*$ for the unobservable subspace of the EqF associated with $\varphi^*$ such that:
    \begin{equation}
        \mathbf{N}^* = \mathbf{T} \mathbf{N}.
        \label{equ:NTN}
    \end{equation}
\end{corollary}
\begin{proof}
   See the supplementary material.
\end{proof}
This result suggests that estimation consistency can be achieved by selecting a global-local map whose induced transformation renders the unobservable subspace state-independent. By definition, a global-local map is composed of a group action and a local coordinate chart. Unlike conventional methods, which achieve consistency by selecting specific group actions, the following theorem provides a more general route based on transforming the local coordinate chart.
\begin{theorem}
   For any EqF defined by a global-local map $\varphi = \vartheta \circ \phi_{\hat{X}^{-1}}$ with unobservable subspace $\mathbf{N}$, 
        there exists a nonsingular matrix $\mathbf{T}$ such that 
    the EqF constructed via the transformed global-local map $\varphi^* = (\mathbf{T} \vartheta) \circ \phi_{\hat{X}^{-1}}$ 
    possesses a {state-independent} unobservable subspace.
    \label{theorem:consistent_eqf}  
\end{theorem}
\begin{proof}
        First, according to Theorem \ref{lemma:0}, $\bfT$ is the transformation matrix between these two EqFs.
         Subsequently, by Corollary \ref{corollary:NTN}, the unobservable subspace of the new EqF is spanned by the columns of $\bfT\bfN$. Finally, the existence of a nonsingular matrix $\mathbf{T}$ that renders $\mathbf{T}\mathbf{N}$ state-independent is guaranteed by the properties of basis transformation in linear algebra \cite{tian2025unobservable}, which completes the proof.
\end{proof}
Theorem \ref{theorem:consistent_eqf} shifts the focus of consistent EqF design from identifying complex group symmetries to the more direct task of coordinate transformation. It therefore provides a general methodology for achieving consistency in arbitrary unobservable systems, especially when intuitive symmetry groups are not readily apparent.

\subsection{Example: Design Consistent EqF from SD-EqF}
\label{sec:design_consistent_eqf}
Here, we present an example of designing a consistent EqF for VINS based on Theorem \ref{theorem:consistent_eqf}.
In this example, SD-EqF serves as the auxiliary EqF with global-local map $\varphi$, and the newly designed EqF with global-local map $\varphi^*$  is named T-EqF\footnote{Outside theorems and corollaries, we abuse the superscript $*$ to distinguish T-EqF variables from SD-EqF ones.}.
The nonsingular matrix is chosen as 
    
\begin{equation}
    \bfT =
     \setlength{\arraycolsep}{4pt}
    \left[
    \begin{NiceArray}{cc|c}[
        margin, 
        nullify-dots, 
    ]
        \Block{1-2}{\scaleF{\bfI}_{15}}& &\scaleF\bfZo_{15\times 3m} \\
        \hline
        {[\hat{\bff}_{1}]_\times} & \bfZo_{3\times 12} &\Block{3-1} {\scaleF\bfI_{3m}}\\
                \vdots & \vdots& \\
                {[\hat{\bff}_{m}]_\times} & \bfZo_{3\times 12} & \\
    \end{NiceArray}\right].
    \label{equ:transformation_consistent_eqf}
\end{equation}
Substituting \eqref{equ:transformation_consistent_eqf} and \eqref{equ:unobservable_subspace_eqf} into \eqref{equ:NTN}, one can verify that the unobservable subspace of T-EqF is state-independent, thereby ensuring its consistency.
It is worth noting that we can utilize arbitrary EqF as the auxiliary EqF. If the auxiliary EqF is fixed as ESKF, T-EqF is reduced to T-ESKF in \cite{tian2025teskf}.

\subsection{\revised{Discussion}}
\revised{
The proposed consistency design is carried out at the level of linearized error-state systems.
Consequently, the resulting T-EqF may coincide, at the linearized-system level, with a native EqF derived from symmetry groups, such as the ISD-EqF discussed in the supplementary material. 
In such cases, the two EqFs have identical Jacobians and therefore share the same observability-related properties under the same linearization.}

\revised{
However, sharing the same first-order linearized system does not imply that the two EqFs are fully equivalent as nonlinear filters. 
Their global-local maps may still differ beyond first order, even when the transformation matrix between their local error coordinates reduces to the identity at the current estimate. 
Consequently, finite corrections may be injected into the state manifold through different inverse maps, and the higher-order terms of the error coordinates may differ. 
Thus, higher-order nonlinear properties of a native EqF, such as its full nonlinear error geometry or reduced estimate-dependence beyond the linearized model, are not guaranteed to be inherited by the T-EqF.}
\revised{A detailed discussion is provided in the supplementary material.}

\section{Efficient EqF Implementation}
\label{sec:efficient_eqf_implementation}
This section uses the identified transformation relationship to resolve the fundamental trade-off between theoretical consistency and computational efficiency in EqFs. While advanced EqFs guarantee consistency or other desirable properties, they often introduce non-block-diagonal Jacobians, which significantly increase the computational cost of covariance propagation relative to conventional filters \cite{yangDecoupledRightInvariant2022}. To mitigate this issue, we propose two implementation strategies that leverage the transformation relationship between EqFs: {Transforming Propagation (TP)} and {Transforming Correction (TC)}. These strategies exploit the block-diagonal structure of an auxiliary EqF (e.g., ESKF) for the heavy covariance computations, thereby preserving the consistency benefits of the target EqF with minimal computational overhead.

\subsection{Computational Bottleneck in Covariance Propagation}

\begin{figure}[t]
    \centering
    \includegraphics[width=0.99\linewidth]{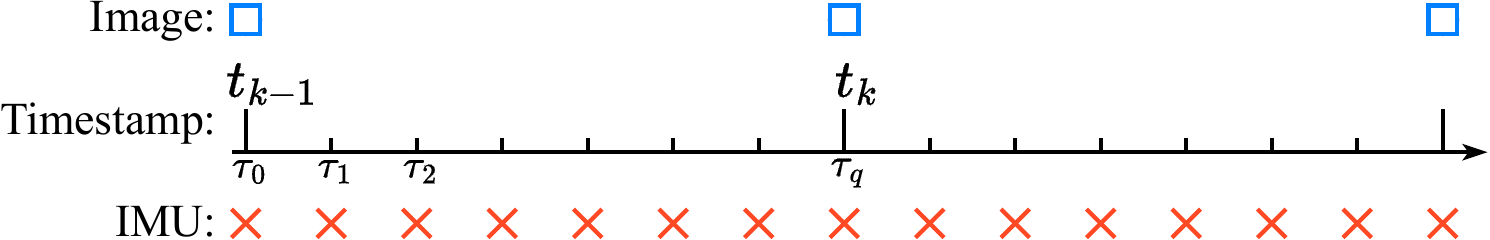}
    \caption{Different sampling frequencies of the IMU and camera.}
    \label{fig:freq}
\end{figure}

\label{sec:computational_bottleneck}
IMUs are typically sampled at a much higher frequency than cameras. Without loss of generality, let the IMU sampling frequency be $q$ times the camera sampling frequency, as shown in Fig.~\ref{fig:freq}.
Let $t_{k-1}$ and $t_{k}$ denote two consecutive camera measurement instants. Covariance propagation between these instants in the EqF framework follows the standard EKF structure: \begin{equation}
\bfP_{k|k-1} = \bfPhi_{k} \bfP_{k-1|k-1} {\bfPhi_{k}}^\top + \bfQ_{k}.
\label{equ:cov_propagation}
\end{equation}
where the state transition matrix $\bfPhi_k := \bfPhi(\tau_{q},\tau_0)$ and the accumulated noise matrix $\bfQ_k := \bfQ(\tau_{q},\tau_0)$ are computed iteratively over $q$ sub-intervals:
\begin{subequations}
\begin{align}
& \bfPhi(\tau_{i+1},\tau_{0}) = \bfPhi(\tau_{i+1},\tau_{i}) \bfPhi(\tau_{i},\tau_{0}), \\
& \begin{aligned}
\bfQ(\tau_{i+1},\tau_0) = & \ \bfPhi(\tau_{i+1},\tau_{i}) \bfQ(\tau_{i},\tau_0) \bfPhi(\tau_{i+1},\tau_{i})^\top \\
& +\bfQ(\tau_{i+1},\tau_i).
\end{aligned}
\end{align}
\label{equ:Q_accumulation_star}%
\end{subequations}
The analytical and explicit expressions for $\bfPhi(\tau_{i+1},\tau_{i})$ and $\bfQ(\tau_{i+1},\tau_i)$ are provided in the supplementary material.
In SD-EqF, $\bfPhi(\tau_{i+1},\tau_{i})$ and $\bfQ(\tau_{i+1},\tau_i)$ are block-diagonal: \begin{equation}
\bfPhi = \begin{bmatrix} \bfPhi_I & \bfZo \\ \bfZo & \bfI \end{bmatrix}, \quad
\bfQ = \begin{bmatrix} \bfQ_I & \bfZo \\ \bfZo & \bfZo \end{bmatrix},
\label{equ:PhiQ_blockdiagonal}
\end{equation} where $\bfPhi_I, \bfQ_I \in \mathbb{R}^{15\times 15}$ relate only to the IMU error-states. As a result, iterative computation of \eqref{equ:Q_accumulation_star} is highly efficient.

However, EqFs with better properties may have more complex Jacobians. For instance, in T-EqF, the transition matrix $\bfPhi^*(\tau_{i+1},\tau_i)$ is no longer block diagonal and the noise matrix $\bfQ^*(\tau_{i+1},\tau_i)$ contains a dense landmark--landmark block:
\begin{equation}
\bfPhi^* = \begin{bmatrix} \bfPhi^*_I & \bfZo \\ \bfPhi^*_{FI} & \bfI \end{bmatrix}, \quad
\bfQ^* = \begin{bmatrix} \bfQ^*_{I} & \bfQ^*_{IF}\\ \bfQ^*_{FI} & \bfQ^*_{FF} \end{bmatrix}.
\label{equ:PhiQ_blockdiagonal2}
\end{equation}
When the state vector contains a large number of landmarks, the computation of \eqref{equ:Q_accumulation_star} becomes a major bottleneck.
\revised{We next present two efficient implementations using T-EqF as an example, both applicable to any target EqF whose dense Jacobians can be related through the proposed transformation framework to a block-diagonal auxiliary EqF.}





\subsection{Implementation 1: Transforming Propagation}
\label{sec:TP}

To bypass the aforementioned bottleneck, our first strategy, TP, exploits the transformation relationship to streamline matrix operations. The theoretical foundation for this acceleration is established by the following corollary:
\begin{corollary}
        Consider two arbitrary EqFs of the same system, associated with the global-local maps $\varphi$ and $\varphi^*$, respectively. The accumulated transition matrices and noise matrices of these two EqFs are related by the transformation matrix:
        \begin{subequations}
                \begin{align}
                        & \bfPhi_k^* = \bfT(\hat{\xi}_{k|k-1}) \bfPhi_{k} \bfT(\hat{\xi}_{k-1|k-1})^{-1}, \\
                        & \bfQ_{k}^* = \bfT(\hat{\xi}_{k|k-1}) \bfQ_{k} \bfT(\hat{\xi}_{k|k-1})^\top,
                \end{align}
        \end{subequations}
        where $\bfT(\hat{\xi}_{k-1|k-1})$ and $\bfT(\hat{\xi}_{k|k-1})$ are evaluated at $\hat{\xi}_{k-1|k-1}$ and $\hat{\xi}_{k|k-1}$, respectively. 
        \label{corollary:PhiQ}
\end{corollary}
\begin{proof}
   See the supplementary material.
\end{proof}

\begin{algorithm}[t]
        \caption{Transforming Propagation (TP)}
        \label{alg:TP}
        \setstretch{1.1} 
        \KwIn {Covariance: \(\bfP_{k-1|k-1}^{*}\) and IMU measurements.}
        \KwOut {Covariance: \(\bfP_{k|k-1}^*\).}

    \tcp{\small $\mathbf{T}$: from auxiliary EqF to target EqF}
    \BlankLine
        \tcp{\small Small-sized matrix computation}
        Initialize \(\bfPhi_I(\tau_0,\tau_0)=\bfI_{15},\bfQ_{I}(\tau_0,\tau_0)=\bfZo_{15}\);\\
        \For{$i \in \{0,1,\dots,q-1\}$}{

        \( \bfPhi_I(\tau_{i+1},\tau_0) = \bfPhi_I(\tau_{i+1},\tau_{i}) \bfPhi_I(\tau_{i},\tau_0)\);

        \( \begin{aligned}
\bfQ_I(\tau_{i+1},\tau_0) = & \ \bfPhi_I(\tau_{i+1},\tau_{i}) \bfQ_I(\tau_{i},\tau_0) \bfPhi_I(\tau_{i+1},\tau_{i})^\top \\
& +\bfQ_I(\tau_{i+1},\tau_i) ;
\end{aligned} \)
        }
        \(
        \bfPhi_k = 
        \begin{bmatrix}
                        \bfPhi_I (\tau_q, \tau_0) & \bf0      \\
                        \bf0           & \bfI_{3m} \\
                \end{bmatrix}
        \) ;

               \(
        \bfQ_k = 
        \begin{bmatrix}
                        \bfQ_I (\tau_q, \tau_0) & \bf0      \\
                        \bf0           & \bfZo_{3m\times 3m} \\
                \end{bmatrix}
        \) ;
        \BlankLine
        \tcp{\small Full-state-size matrix computation}

        \(  \bfPhi_{k}^* =
                \setlength{\arraycolsep}{1pt}
                \bfT(\hat{\xi}_{k|k-1}) \bfPhi_k \bfT(\hat{\xi}_{k-1|k-1})^{-1}\);
                
        \(\bfQ_k^*=\bfT(\hat{\xi}_{k|k-1}) 
     \setlength{\arraycolsep}{3pt}
        \bfQ_k \bfT(\hat{\xi}_{k|k-1})^\top\); \\
        \(\bfP_{k|k-1}^* = \bfPhi_{k}^*  \bfP_{k-1|k-1}^* {\bfPhi_{k}^*}^\top + \bfQ_{k}^*\).
\end{algorithm}

By selecting a block-diagonal auxiliary EqF, such as SD-EqF or ESKF, TP restricts the expensive accumulation of transition and noise matrices to the fixed-size IMU sub-blocks. In this example, SD-EqF is used as the auxiliary EqF. As shown in Algorithm~\ref{alg:TP}, Lines 1--6 compute the accumulated SD-EqF matrices using the block-diagonal structure in \eqref{equ:PhiQ_blockdiagonal}. Lines 7--8 then recover the corresponding full-state T-EqF matrices through the transformation relationship, and Line 9 propagates the target covariance.
As a result, the full-state operations are performed only once per camera interval.

\subsection{Implementation 2: Transforming Correction } 
\label{sec:TC}

Instead of directly maintaining the T-EqF covariance $\bfP^*$ as in TP, TC tracks the equivalent covariance $\bfP := \bfT^{-1}\bfP^*\bfT^{-\top}$ in the auxiliary EqF coordinates; a detailed derivation is provided in the supplementary material. By Corollary~\ref{corollary:PhiQ}, the propagation of $\bfP$ follows the same equation as the auxiliary EqF covariance, avoiding dense propagation in the target state.
This additional correction-stage transformation is the defining feature of TC.
The only additional covariance operation appears after correction: since the state update changes the evaluation point of $\bfT$ from prior to posterior, a relative transformation is required to preserve the relation between $\bfP$ and $\bfP^*$. Algorithm~\ref{alg:TC} summarizes TC under the bijective setting of Remark~\ref{remark:1}, where the estimated group elements need not be explicitly maintained; the non-bijective case is discussed in the supplementary material. Lines 1--2 initialize the equivalent covariance, Lines 3--6 perform propagation and correction using the efficient SD-EqF Jacobians, Line 7 injects the Kalman correction through the T-EqF error geometry,  Line 8 applies the relative transformation
to prepare the covariance for the next recursion, and Line 9 optionally converts $\bfP$ back to $\bfP^*$ when the target covariance is required for output.

\begin{algorithm}[t]
    \setstretch{1.1} 
    \caption{Transforming Correction (TC)}
    \label{alg:TC}
    \KwIn{Initial guess $(\hat{\xi}_{0}$, $\mathbf{P}^*_{0})$ and all measurements.}
    \KwOut{Posterior $\hat{\xi}_{k|k}$, $\mathbf{P}_{k|k}^* $.}

    \tcp{\small $\mathbf{T}$: from auxiliary EqF to target EqF}

    \BlankLine

    \tcp{\small Initialization}
    $\hat{\xi}_{0|0} = \hat{\xi}_{0}$\;
    $\mathbf{P}_{0|0} = \mathbf{T}(\hat{\xi}_{0})^{-1} \mathbf{P}_{0}^* \mathbf{T}(\hat{\xi}_{0})^{-\top}$\;

    \BlankLine
    \tcp{\small Propagation from $t_{k-1}$ to $t_k$}
    Propagate $\hat{\xi}_{k-1|k-1}$ to $\hat{\xi}_{k|k-1}$\;
    $\mathbf{P}_{k|k-1} = {\bfPhi}_k \mathbf{P}_{k-1|k-1} {\bfPhi}_k^\top + \mathbf{Q}_k$\;

    \BlankLine
    \tcp{\small Correction at $t_k$}
    $\mathbf{K} = \mathbf{P}_{k|k-1} \mathbf{H}_{k}^\top (\mathbf{H}_{k} \mathbf{P}_{k|k-1} \mathbf{H}_{k}^\top + \mathbf{R}_{k})^{-1}$\;
    $\mathbf{P}_{k|k} = (\mathbf{I} - \mathbf{K} \mathbf{H}_{k}) \mathbf{P}_{k|k-1}$\;

    \BlankLine
    \tcp{Transforming correction at $t_k$}
    $\hat{\xi}_{k|k} = \left. {\varphi^*}^{-1}( \mathbf{T} (\hat{\xi}_{k|k-1}) \mathbf{K} \tilde{\mathbf{z}}_{k}\big) \right|_{\hat{\xi} = \hat{\xi}_{k|k-1}}$\;
    $\mathbf{P}_{k|k} \leftarrow \mathcal{T} \mathbf{P}_{k|k} \mathcal{T}^\top $ with $\mathcal{T} = \mathbf{T}(\hat{\xi}_{k|k})^{-1} \mathbf{T}(\hat{\xi}_{k|k-1})$\;
    $\mathbf{P}_{k|k}^* = \mathbf{T}(\hat{\xi}_{k|k}) \mathbf{P}_{k|k} \mathbf{T}(\hat{\xi}_{k|k})^\top$.  {\small \color{gray}\quad This line is optional.}
\end{algorithm}

\subsection{\revised{Computational Complexity Analysis}}
\label{sec:complexity_analysis}

\revised{
Let \(m\) denote the number of landmarks, \(q\) the number of IMU propagation sub-intervals between two consecutive camera frames, \(p\) the number of correction steps performed within one camera frame, and \(N=15+3m\) the dimension of the full error-state. 
In practical VINS pipelines, $p$ is larger than 1 due to the use of batched visual measurements and delayed feature initialization \cite{genevaOpenVINSResearchPlatform2020}. 
In our subsequent simulations and experiments, its average value typically falls between 3 and 5.
The costs reported in Table \ref{tab:complexity} are the additional costs relative to the auxiliary EqF, whose propagation already exploits the block-diagonal structure in \eqref{equ:PhiQ_blockdiagonal}. 
Common costs shared by all implementations, such as feature tracking, data association, projection Jacobian evaluation, and the standard measurement update, are not included.}

\revised{
For the naive implementation of T-EqF, the target transition and noise matrices in \eqref{equ:PhiQ_blockdiagonal2} are accumulated at every IMU sub-interval. 
After exploiting their zero and identity blocks, the remaining dominant operations update dense $3m\times3m$ landmark--landmark blocks, giving $O(qm^2)$ cost. 
In TP, the \(q\)-step propagation is first performed in the auxiliary coordinates using only fixed-size \(15\times15\) IMU blocks,  whose cost is negligible.
 The target covariance is then recovered once per camera interval through block-sparse transformations, 
which introduces an additional \(O(m^2)\) propagation cost. 
In TC, propagation is identical to that of the auxiliary EqF, so no additional propagation cost is introduced; however, the relative transformation in Algorithm \ref{alg:TC}, Line 8, is required after each correction. 
Exploiting its block sparsity gives \(O(m^2)\) per correction and \(O(pm^2)\) per camera frame. 
Thus, TP is typically more efficient than TC when multiple corrections are performed per camera frame. 
In terms of implementation effort, however, TC is generally easier to integrate into an existing EqF pipeline, since it can be implemented as a correction-stage refinement built on top of the auxiliary EqF. 
By contrast, TP requires a more complete implementation of the target EqF covariance propagation.
A detailed FLOP-level derivation is provided in the supplementary material.}

\begin{table}[!thp]
        \caption{Additional computational costs of different implementations compared to the auxiliary EqF.}
        \centering
    \setlength\tabcolsep{10pt}
        \begin{tabular}{lccc}
                \toprule
                 & \textbf{Naive} & \textbf{TP} & \textbf{TC} \\
                \midrule
                Propagation & $O(qm^2)$ & $O(m^2)$ & 0 \\
                Correction & 0 & 0 & $O(pm^2)$\\
                \bottomrule
        \end{tabular}
        \label{tab:complexity}
\end{table}

\section{Monte Carlo Simulations}
\label{sec:monte_carlo_simulations}

In this section, we validate the proposed consistent EqF design and its efficient implementation through Monte Carlo simulations. \revised{In addition to ESKF \cite{eskf}, SD-EqF \cite{fornasier2025equivariant}, and the proposed T-EqF, we consider two consistent EqF formulations extracted from the existing literature. The first is Invariant SD-EqF (ISD-EqF) \cite{fornasier2024hybrid}, which combines the advantages of SD-EqF and the invariant EKF. The second is SOT-EqF \cite{vangoor2023eqvio}, 
which retains the invariant IMU state and the {\it scaled orthogonal transform} group for landmarks.
}
 For a fair comparison, all estimators are implemented within the same OpenVINS-based VINS pipeline \cite{genevaOpenVINSResearchPlatform2020}, with identical data association, feature management, and initialization modules. The only differences among the EqF-based estimators lie in the design of the global-local maps and the corresponding Jacobians.

During testing, the IMU and camera measurements are generated using a high-fidelity simulator \cite{genevaOpenVINSResearchPlatform2020}, with detailed parameters listed in Table \ref{tab:sim_params}.
\revised{
This table summarizes the sensor noise parameters, sampling rates, camera configuration, and maximum number of landmarks used in the OpenVINS simulator. All parameters follow the default OpenVINS simulator configuration.}
The simulator takes a time-stamped pose trajectory as input, fits it by an $\mathbf{SE}(3)$ B-spline, and then generates inertial and visual measurements. Hence, the trajectories used in Table \ref{tab:rmse} are reference motion trajectories for simulation rather than raw image/IMU datasets.
\revised{These trajectories cover representative VINS scenarios, including outdoor/vehicle-like motions (Udel-Gara and Udel-Neig), micro-aerial-vehicle motion (EuRoC-V11), and indoor handheld motions. Detailed trajectory descriptions and a 3D visualization are provided in the supplementary material.}


\begin{table}[!thp]
    \caption{\revised{Simulator basic parameters}}
    \centering
    \footnotesize
    \setlength\tabcolsep{4.5pt}
    \begin{tabular}[]{cccc}
            \toprule
            { \textbf{Parameter}} & \textbf{Value} & \textbf{Parameter} & \textbf{Value}  \\
            \midrule
            Accel. White Noise    & 2.0e-03  $\text{m} / \text{s}^2 / \sqrt{\text{Hz}}$    &IMU Freq.           & 200 Hz  \\
            Accel. Random Walk    & 3.0e-03  $\text{m} / \text{s}^3 / \sqrt{\text{Hz}}$     & Cam. Freq.          & 10  Hz  \\
            Gyro. White Noise  & 1.7e-04  $\text{rad} / \text{s} / \sqrt{\text{Hz}}$      & Cam. Noise          & 1 pixel \\
             Gyro. Random Walk  & 2.0e-05  $\text{rad} / \text{s}^2 / \sqrt{\text{Hz}}$  &  Cam. Number         & Mono  \\
              Cam. Resolution     & 752 $\times$ 480 pixels & Max Landmarks       & 40 \\
            \bottomrule
    \end{tabular}
    \label{tab:sim_params}
\end{table}


\subsection{Estimation Consistency}
We first validate the proposed consistent EqF design by comparing estimator accuracy and consistency. The five estimators are tested on seven trajectories from OpenVINS. For each trajectory and each estimator, 100 Monte Carlo simulations are conducted.

The position and orientation root mean square errors (RMSEs) over the 100 runs are reported in Table \ref{tab:rmse}.
Since the unobservable subspaces of ESKF and SD-EqF are state-dependent, both filters suffer from inconsistency and degraded accuracy. By contrast, T-EqF shares the same state-independent unobservable subspace as ISD-EqF, resulting in consistent estimation and improved performance. \revised{SOT-EqF achieves the best accuracy in some cases because its local landmark parameterization reduces measurement linearization error. T-EqF nevertheless remains competitive. More importantly, T-EqF provides a systematic route for designing consistent EqFs for arbitrary systems and is compatible with designs that reduce linearization errors. In the supplementary material, we transform the ESKF-style IMU error into a right-invariant error while retaining a local landmark parameterization, thereby improving consistency and reducing linearization error simultaneously.} 

  \begin{table}[h]
    \caption{\revised{Orientation (deg) / Position (m) RMSE comparison in simulation.}}
    \centering
    \scriptsize
    \setlength\tabcolsep{2.2pt}
    \begin{tabular}[]{cccccc}
            \toprule
            {\textbf{Trajectory}} & \textbf{ESKF} & \textbf{SD-EqF} & \textbf{ISD-EqF} & \revised{\textbf{SOT-EqF}} & \textbf{T-EqF} \\
            \midrule
  Udel-Gore  & 0.360 / 0.102 & 0.360 / 0.102 & \textbf{0.254} / 0.092 & \revised{0.255 / \textbf{0.091}} & \textbf{0.254} / 0.092 \\
            Udel-Gara  & 0.245 / 1.686 & 0.249 / 1.704 & 0.222 / 1.655 & \revised{\textbf{0.219} / \textbf{1.568}} & 0.222 / 1.653 \\
            Udel-Arl-s & 1.291 / 0.141 & 1.291 / 0.141 & \textbf{0.756} / \textbf{0.124} & \revised{0.757 / \textbf{0.124}} & \textbf{0.756} / \textbf{0.124} \\ 
            Udel-Neig  & 1.625 / 10.94 & 1.569 / 10.65 & 1.137 / 7.847 & \revised{\textbf{1.123} / \textbf{7.823}} & 1.139 / 7.855 \\
            TUM-Corr   & 0.169 / 0.082 & 0.170 / 0.082 & \textbf{0.152} / \textbf{0.079} & \revised{0.153 / \textbf{0.079}} & \textbf{0.152} / \textbf{0.079} \\
            TUM-Magi   & 0.456 / 0.420 & 0.499 / 0.432 & \textbf{0.376} / 0.392 & \revised{0.380 / \textbf{0.372}} & \textbf{0.376} / 0.391\\
            EuRoC-V11 & 0.232 / 0.027 & 0.231 / 0.027 & \textbf{0.166} / \textbf{0.025} & \revised{\textbf{0.166} / \textbf{0.025}} & \textbf{0.166} / \textbf{0.025} \\
            \bottomrule
    \end{tabular}
    \label{tab:rmse}
\end{table}

To illustrate the consistency differences among the estimators, we run 1000 Monte Carlo simulations on the Udel-Gore trajectory. The orientation and position normalized estimation error squared (NEES) results are shown in Fig. \ref{fig:nees}. Compared with SD-EqF, the mean and 95\% bounds of the T-EqF NEES closely match the theoretical values, indicating good consistency. Fig. \ref{fig:error_bound} further shows the estimation errors. The T-EqF errors remain well bounded, whereas those of SD-EqF exceed the 3-sigma bound, especially in yaw.


\begin{figure}[!tp]
        \centering
        \includegraphics[width=0.95\linewidth]{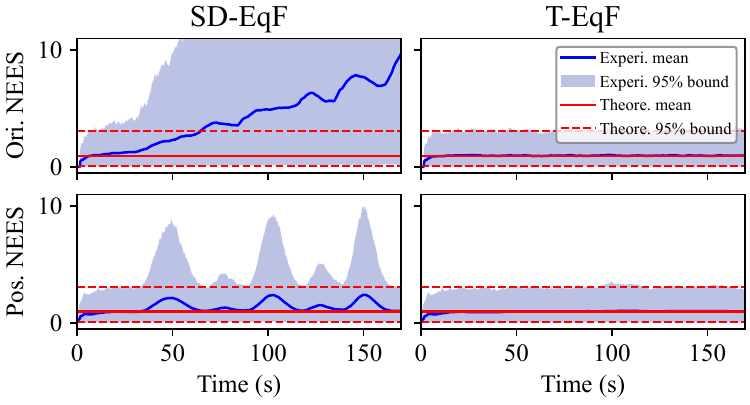}
        \caption{Orientation and position NEES over time of 1000 Monte Carlo simulations on Udel-Gore.
        The ESKF, ISD-EqF, and \revised{SOT-EqF} plots are omitted because they are nearly identical to those of SD-EqF and T-EqF, respectively.}
        \label{fig:nees}
\end{figure}

\begin{figure}[h]
        \centering
        \includegraphics[width=0.99\linewidth]{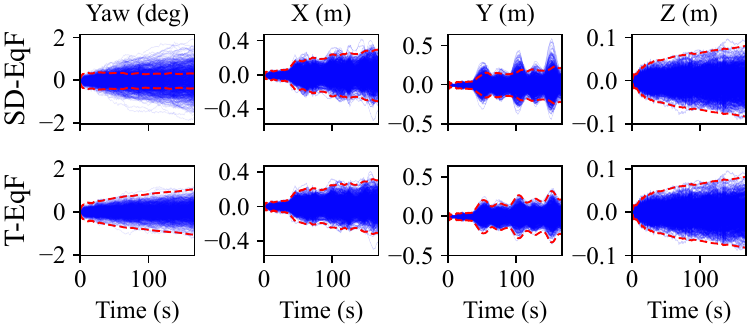}
        \caption{Estimation error (blue lines) and 3-sigma bounds (red dashed line) of IMU yaw and position of 1000 Monte Carlo simulations on Udel-Gore. Other estimators are omitted because their results are similar to those of SD-EqF and T-EqF.}
        \label{fig:error_bound}
\end{figure}
\begin{figure}[!h]
        \centering
        \includegraphics[width=0.95\linewidth]{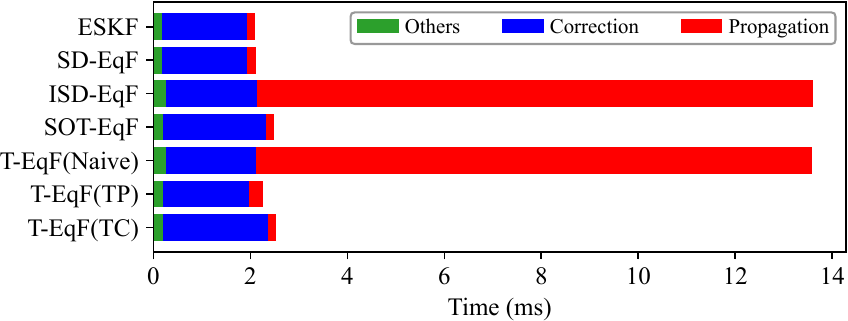}
        \caption{\revised{Average time to process one camera frame on Udel-Gore. The test platform is an x86 desktop (R9 7950X at 4.5~GHz).}}
        \label{fig:time_bar}
\end{figure}

\begin{figure*}[!ht]
        \centering
        \includegraphics[width=0.95\linewidth]{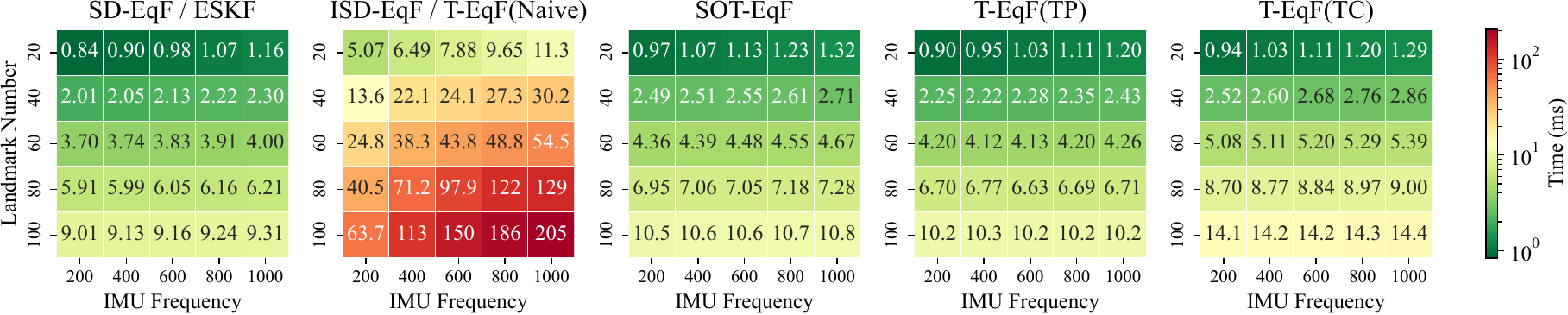}
        \caption{\revised{Average time to process one camera frame on Udel-Gore with varying numbers of landmarks and different IMU frequencies.
        The results for SD-EqF and ESKF are similar and are therefore merged into one plot, as are the results for ISD-EqF and T-EqF(Naive).
         The test platform is an x86 desktop (R9 7950X at 4.5~GHz).}}
        \label{fig:time_condition}
\end{figure*} 

\subsection{Computational Efficiency}

We next compare the runtime of ESKF, SD-EqF, ISD-EqF, \revised{SOT-EqF}, and three T-EqF implementations: naive propagation, TP, and TC. \revised{ESKF, SD-EqF, and SOT-EqF already have efficient covariance propagation, whereas ISD-EqF and naive T-EqF repeatedly process dense landmark--landmark covariance blocks during propagation.} \revised{All compared implementations exploit the available block sparsity to skip products involving zero and identity blocks.}


 The average processing time per camera frame is shown in Fig. \ref{fig:time_bar}. The naive T-EqF implementation clearly suffers from the covariance-propagation bottleneck, whereas the proposed TP and TC implementations substantially reduce the runtime.  
Furthermore, TP is more efficient than TC, especially during correction. As discussed in Section \ref{sec:complexity_analysis}, the additional transformation in TP is executed only once per camera frame, whereas the relative transformation in TC must be executed after every correction (including the feature initialization), leading to higher computational cost compared to TP. 
\revised{
Compared with SD-EqF and ESKF, SOT-EqF has a higher correction cost because its local landmark representation activates additional anchored IMU states during measurement updates \cite{genevaOpenVINSResearchPlatform2020}, increasing the dimension of the required matrix inversion.}




To evaluate the proposed TP and TC implementations under many landmarks and high-frequency IMU measurements, we run additional simulations with varying landmark numbers and IMU frequencies. The results are shown in Fig. \ref{fig:time_condition}. As both factors increase, the runtime of the naive T-EqF implementation grows substantially, whereas TP and TC maintain low computational costs. Compared with the simplest estimator, ESKF, TP incurs almost no additional cost, while TC introduces only minor overhead.


\section{Experiments on Real-world Datasets}
\label{sec:real_world_datasets}
\subsection{Dataset Experiment on EuRoC MAV}
The proposed method is evaluated on the public EuRoC MAV dataset \cite{burri2016EuRoCMicroAerial}, \revised{which provides stereo images, IMU measurements, and ground-truth trajectories collected by a micro aerial vehicle in indoor environments. To facilitate
reproducibility and ensure a fair comparison, all compared estimators use the default camera and landmark-state settings of the
OpenVINS \textbf{euroc\_mav} configuration, i.e., stereo input and at most 50
in-state landmarks. } Due to the complexities of real-world environments, such as motion blur, imperfect data association, and non-Gaussian noise, performance differences among the estimators are less pronounced than in simulation. Nevertheless, T-EqF maintains superior accuracy over ESKF and SD-EqF across most sequences, as reported in Table \ref{tab:euroc_rmse}. 

\begin{table}[htp]
    \caption{ \revised{Orientation (deg) and Position (m) RMSE on EuRoC MAV dataset.}}
    \centering
    \scriptsize
    \setlength\tabcolsep{2.9pt}
    \begin{tabular}{cccccc}
            \toprule
            { \textbf{Dataset}} & \textbf{ESKF} & \textbf{SD-EqF} & \textbf{ISD-EqF} & \textbf{\revised{SOT-EqF}} & \textbf{T-EqF} \\
            \midrule
                V1-01               & 1.184 / 0.129    & 1.537 / 0.170    & 1.625 / 0.122    & \revised{\textbf{1.084} / \textbf{0.102}}    & 1.620 / 0.123 \\
                V1-02             & 1.859 / \textbf{0.103}    & 0.863 / 0.138    & \textbf{0.512} / 0.120    & \revised{1.088 / 0.130}    & 0.631 / 0.116 \\
                V1-03          & 2.494 / 0.225    & \textbf{1.959} / 0.198    & 3.094 / 0.147    & \revised{3.227 / 0.138}    & 2.781 / \textbf{0.129} \\
                V2-01               & 1.315 / 0.113    & 1.303 / 0.160    & 0.870 / 0.133    & \revised{\textbf{0.747} / \textbf{0.088}}    & 0.927 / 0.132 \\
                V2-02             & 1.384 / 0.084    & 1.697 / \textbf{0.067}    & 1.400 / 0.072    & \revised{1.757 / 0.070}    & \textbf{1.343} / 0.090 \\
                V2-03          & 5.547 / 0.207    & 6.178 / 0.230    & \textbf{1.158} / 0.228    & \revised{1.390 / 0.205}    & 1.275 / \textbf{0.169} \\
                MH-01               & 5.598 / 0.453    & 3.746 / \textbf{0.262}    & 1.663 / 0.329    & \revised{1.781 / 0.327}    & \textbf{1.446} / 0.321 \\
                MH-02               & 3.426 / 0.382    & 4.029 / 0.295    & 1.512 / 0.406    & \revised{0.998 / \textbf{0.221}}    & \textbf{0.837} / 0.261 \\
                MH-03             & 3.009 / \textbf{0.236}    & 2.844 / 0.309    & 0.854 / 0.284    & \revised{\textbf{0.572} / 0.265}    & 0.660 / 0.332 \\
                MH-04          & 1.912 / 0.558    & 1.393 / 0.650    & 1.074 / 0.482    & \revised{1.123 / \textbf{0.438}}    & \textbf{0.951} / 0.632 \\
                MH-05          & 3.433 / 0.591    & 2.731 / 0.842    & 2.627 / 0.566    & \revised{\textbf{0.792} / \textbf{0.318}}    & 1.310 / 0.392 \\
                \midrule
                average                  & 2.597 / 0.257    & 2.357 / 0.277    & 1.366 / 0.241    & \revised{1.213 / \textbf{0.192}}    & \textbf{1.149} / 0.225 \\
            \bottomrule
    \end{tabular}
    \label{tab:euroc_rmse}
\end{table}
\begin{figure}
        \centering
        \includegraphics[width=0.95\linewidth]{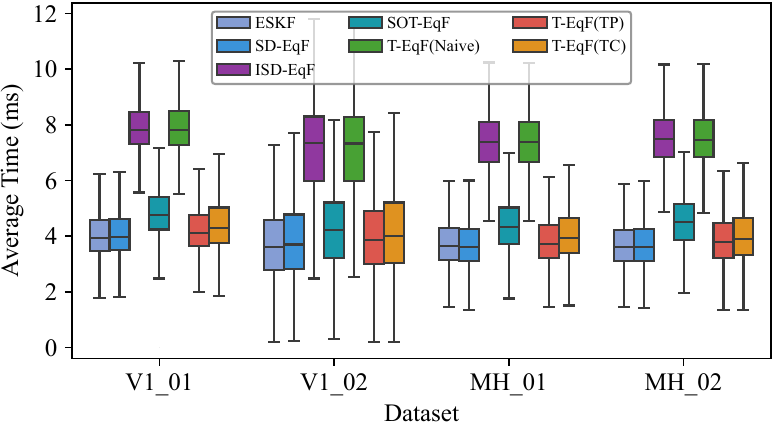}
        \caption{\revised{Average time of processing one camera frame (including prediction and correction) on EuRoC dataset.}}
        \label{fig:real_time_bar}
\end{figure}

 Furthermore, Fig. \ref{fig:real_time_bar} shows the average computational time per camera frame. Compared with the simulation results, the processing time on real-world datasets is higher because of the overhead of feature extraction and matching. However, as shown in Fig. \ref{fig:real_time_bar}, the proposed TP and TC implementations significantly reduce the computational latency of T-EqF, achieving efficiency comparable to ESKF and SD-EqF.

\subsection{Experiment on an Aerial Robot}
We further validate the proposed approach through experiments on a custom aerial robot platform (Fig. \ref{fig:platform} (a)).
The platform provides stereo imagery at 30 Hz (848 $\times$ 480 pixels) and IMU data at 200 Hz, while ground-truth data are obtained from a NOKOV motion capture system. We collect a total of six experimental sequences: three for the Eight motion pattern (Fig.~\ref{fig:platform}(b)) and three for the Zero motion pattern.
\revised{
        These reproducible closed-loop patterns were selected because their repeated
turns, heading changes, and continuous yaw rotation provide challenging
visual--inertial excitation while retaining easily inspected trajectory
shapes. The proposed method is independent of the trajectory shape and applies
to arbitrary motion trajectories.
}
Within each sequence, the corresponding pattern is repeated six times, resulting in a total flight distance of approximately 80--85 meters. 
\revised{
Except for the sensor noise parameters and the intrinsic and extrinsic parameters of the IMU and camera, all other settings are the same as those used in the EuRoC MAV experiments.}
 A representative frame displaying the tracked features is shown in Fig.  \ref{fig:platform} (c).
\begin{figure}[!t]
        \centering
        \includegraphics[width=0.95\linewidth]{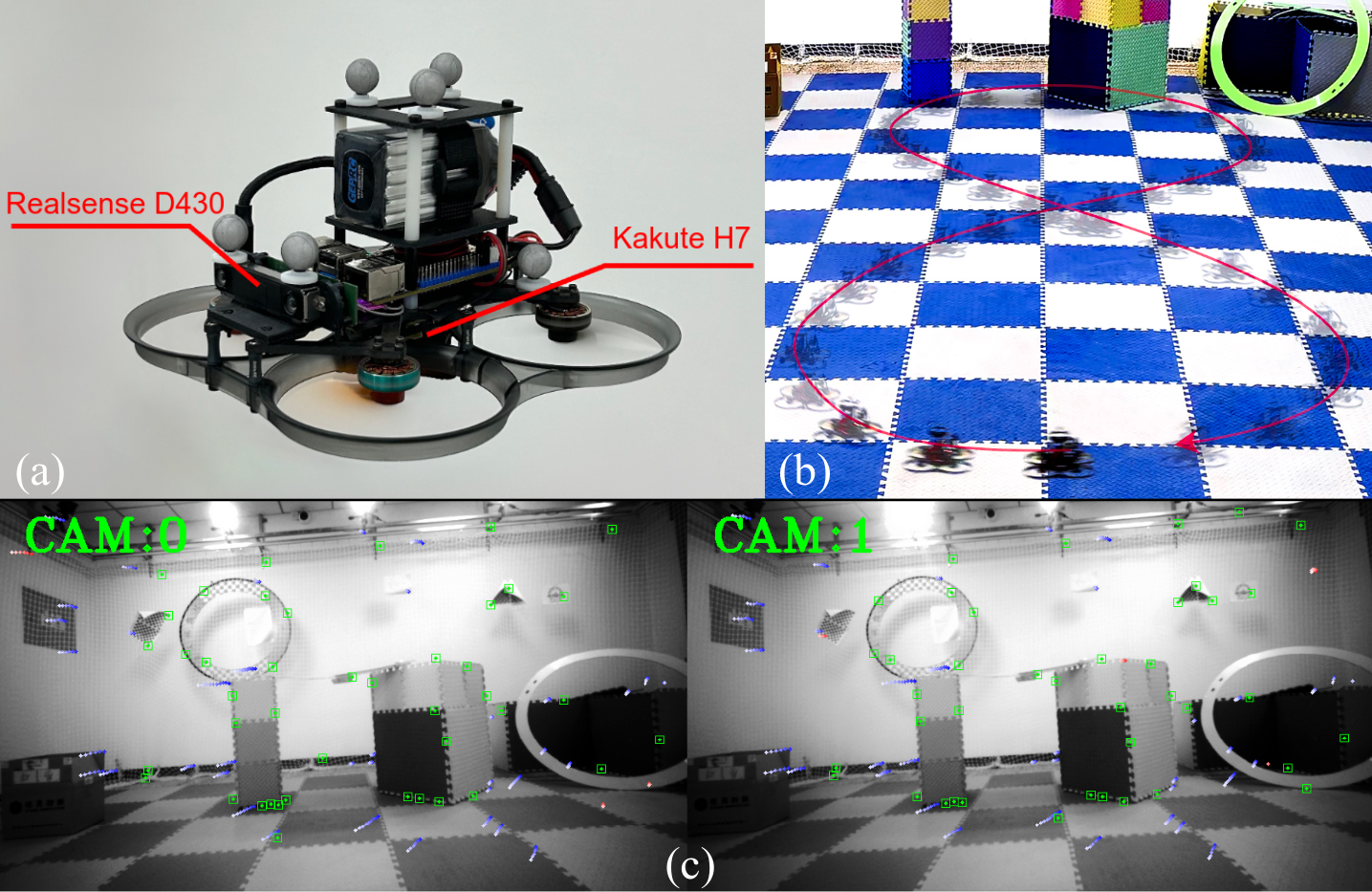}
        \caption{(a) Aerial robot with a RealSense D430 stereo camera and a Kakute H7 flight controller (MPU6000, 200 Hz) (b) Aerial robot flight trajectory shape. (c) A sample frame with tracked features in the experiment.}
        \label{fig:platform}
\end{figure}

\begin{figure}[h]
        \centering
        \includegraphics[width=0.95\linewidth]{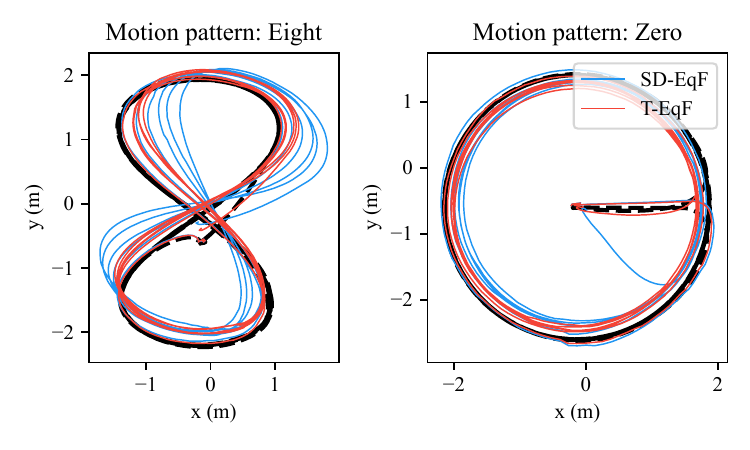}
        \caption{Estimated trajectories of SD-EqF (blue solid) and T-EqF (red solid) on the Eight and Zero motion patterns. The ground-truth trajectories are shown in black dashed lines.}
        \label{fig:exp_traj}
\end{figure}
Fig.  \ref{fig:exp_traj} compares the estimated trajectories of SD-EqF and T-EqF against ground truth for both motion patterns in one representative sequence, where the estimated trajectories are aligned to the ground truth using the first frame as the reference. T-EqF yields more accurate trajectory estimates than SD-EqF, especially in the yaw orientation. As discussed earlier, the state-dependent unobservable subspace of SD-EqF leads to overconfident estimates in this direction, resulting in significant drift over time. 

\begin{figure}
        \centering
        \includegraphics[width=0.95\linewidth]{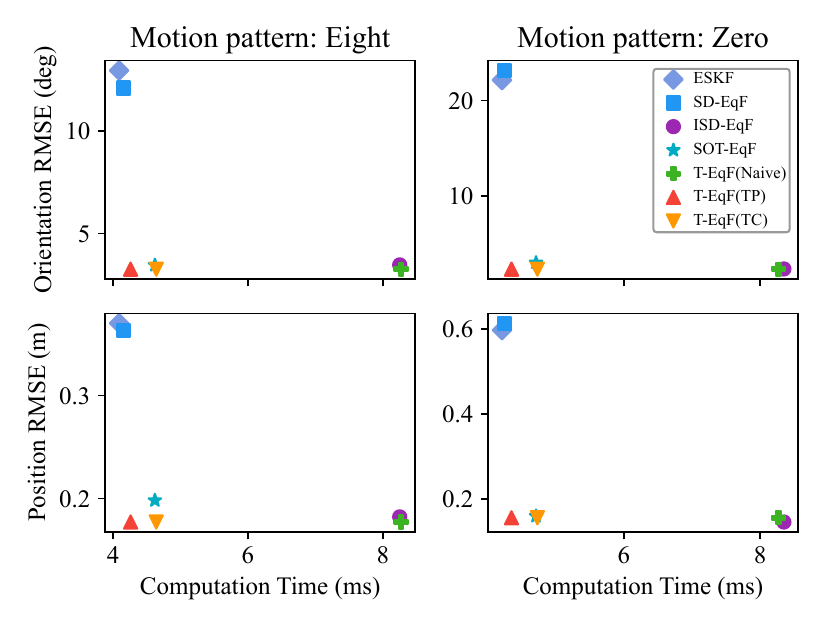}
        \caption{\revised{RMSE-Time joint comparison of these estimators on the Eight and Zero motion patterns. The closer to the bottom-left corner, the better the performance.}}
        \label{fig:exp_time2}
\end{figure}

Finally, the joint comparison of estimation accuracy and computational efficiency for all seven estimators is summarized in Fig. \ref{fig:exp_time2}.
The results form three distinct clusters. First, ESKF and SD-EqF are positioned in the top-left corner: they are computationally efficient but suffer from significant drift due to inconsistency. Second, the naive T-EqF and ISD-EqF appear in the bottom-right corner, demonstrating high estimation accuracy at the expense of a heavy computational penalty. Third, T-EqF with the proposed TP and TC implementations, \revised{together with SOT-EqF}, lies in the bottom-left corner, achieving the best overall performance by maintaining high accuracy while significantly reducing processing time, with TP delivering the best trade-off overall.

\section{Conclusion and Future Work}
\label{sec:conclusion}
This paper proposed an EqF transformation approach linking formulations with different symmetries for consistent and efficient visual--inertial navigation. By introducing the global-local map, we established a \revised{first-order transformation relationship} between the error-states of arbitrary EqFs and derived the corresponding laws for linearized error dynamics and unobservable subspaces. For unobservable systems, these results provide a systematic consistency design principle: a coordinate transformation can produce an EqF with a state-independent unobservable subspace without ad hoc symmetry analysis. To reduce the computational cost caused by non-block-diagonal Jacobians, we developed Transforming Propagation and Transforming Correction. Both exploit the block-diagonal structure of a simpler auxiliary EqF, achieving efficiency comparable to ESKF while preserving consistency. Monte Carlo simulations, EuRoC evaluations, and real-robot experiments validated the proposed approach in both consistency and efficiency.

The current framework focuses on transformations of state-error representations. Incorporating output equivariance may further reduce measurement linearization errors and improve estimation accuracy.



\bibliographystyle{IEEEtran}
\bibliography{main.bib}

\def\SUPPLEMENTINMAIN{}
\ifdefined\SUPPLEMENTINMAIN
\else
\documentclass{article}

\usepackage{geometry}
\geometry{
    textheight=8.5in,
    textwidth=6.5in,
    top=1.5in,
    headheight=12pt,
    headsep=25pt,
    footskip=30pt
}
\setstretch{1.2}

\usepackage{tabularx}
\usepackage{amsmath}
\fi
\def\GSD{\mathbf{G}}
\def\bfPsi{\mathbf{\Psi}}
\def\SO3{{\mathbf{SO}(3)}}
\def\textnote#1{\textcolor[HTML]{5C6BC0}{#1}}

\def\textqcr#1{{\fontfamily{pcr}\selectfont #1}}
\def\rfD{\mathrm{D}}
\def\spancol#1{\underset{\text{col}}{\text{Span}} \left(#1\right)}

\ifdefined\SUPPLEMENTINMAIN
\clearpage
\onecolumn
\setcounter{section}{0}
\setcounter{subsection}{0}
\setcounter{equation}{0}
\setcounter{figure}{0}
\setcounter{table}{0}
\setcounter{theorem}{0}
\setcounter{corollary}{0}
\setcounter{proposition}{0}
\setcounter{lemma}{0}
\setcounter{definition}{0}
\setcounter{remark}{0}
\renewcommand{\thesection}{S\arabic{section}}
\renewcommand{\theequation}{S\arabic{equation}}
\renewcommand{\thefigure}{S\arabic{figure}}
\renewcommand{\thetable}{S\arabic{table}}
\renewcommand{\theHsection}{sup.\arabic{section}}
\renewcommand{\theHsubsection}{sup.\arabic{section}.\arabic{subsection}}
\renewcommand{\theHequation}{sup.\arabic{equation}}
\renewcommand{\theHfigure}{sup.\arabic{figure}}
\renewcommand{\theHtable}{sup.\arabic{table}}
\begin{center}
    {\LARGE\bfseries Supplementary Material for\\[0.4em]
    ``Equivariant Filter Transformations for Consistent and Efficient Visual--Inertial Navigation''}
\end{center}
\phantomsection
\addcontentsline{toc}{part}{Supplementary Material}
\etocstandardlines
\etocsettocstyle{\section*{Contents}}{}
\localtableofcontents
\clearpage
\else
\begin{document}
\title{Supplementary Material for ``Equivariant Filter Transformations for Consistent and Efficient Visual--Inertial Navigation''}
\author{}
\date{}
\maketitle
\fi

\ifdefined\SUPPLEMENTINMAIN
\else
\tableofcontents 
\newpage
\fi
\section{Proofs and Derivations}
\subsection{Proof of Corollary 1}
\begin{corollary}
    Consider two arbitrary EqFs of the same system, associated with $\varphi$ and $\varphi^*$. If the linearized error-state system associated with $\varphi$ is given by
    \begin{subequations}
    \begin{align}
    \dot{\bfvarepsilon} &= \bfF \bfvarepsilon + \bfG \bfn, \label{sup:equ:eqf_general_a}\\
    \tilde{\bfz} &= \bfH \bfvarepsilon + \bfepsilon,  \label{sup:equ:eqf_general_b}
    \end{align}
    \label{sup:equ:eqf_general}%
    \end{subequations}
    then the linearized error-state system associated with $\varphi^*$ satisfies
    \begin{subequations}
    \begin{align}
        \dot{\bfvarepsilon}^* &= \bfF^* \bfvarepsilon^* + \bfG^* \bfn, \\
        \tilde{\bfz} &= \bfH^* \bfvarepsilon^* + \bfepsilon,
    \end{align}
    \label{sup:equ:eqf_transformed}%
    \end{subequations}
    where
    \begin{align}
        \bfF^* &= \dot{\bfT} \bfT^{-1} + \bfT \bfF \bfT^{-1},  \\ 
            \bfG^* &= \bfT \bfG,\\
            \bfH^* &= \bfH \bfT^{-1},
    \end{align}
    and $\bfT$ is the transformation matrix given by Theorem 1 (in the main manuscript).
    \label{sup:corollary:1}
\end{corollary}
\begin{proof}
    By Theorem 1 (in the main manuscript), the error states of the two EqFs are related by
    \begin{equation}
            \bfvarepsilon^* = \bfT \bfvarepsilon, \label{sup:equ:error_state_transformation_sup}%
    \end{equation}
    where $\bfT$ is nonsingular. Taking the time derivative of both sides gives
    \begin{equation}
            \dot{\bfvarepsilon}^* = \dot{\bfT} \bfvarepsilon + \bfT \dot{\bfvarepsilon}.  \label{sup:equ:error_state_derivative_sup}%
    \end{equation}
    Substituting \eqref{sup:equ:eqf_general_a} and \eqref{sup:equ:error_state_transformation_sup} into \eqref{sup:equ:error_state_derivative_sup}  yields
    \begin{subequations}
        \begin{align}
             \dot{\bfvarepsilon}^* & = \left(\dot{\bfT} \bfT^{-1}  + \bfT \bfF \bfT^{-1}\right) \bfvarepsilon^* + \bfT \bfG \bfn \\
             & =: \bfF^* \bfvarepsilon^* + \bfG^* \bfn,
        \end{align}
    \end{subequations}
    Similarly, substituting \eqref{sup:equ:error_state_transformation_sup} into \eqref{sup:equ:eqf_general_b} gives
    \begin{subequations}
        \begin{align}
            	\tilde{\bfz} &= \bfH \bfT^{-1} \bfvarepsilon^* + \bfepsilon \\
            	&=: \bfH^* \bfvarepsilon^* + \bfepsilon,
        \end{align}
    \end{subequations}
    which completes the proof.
\end{proof}

\subsection{Proof of Corollary 2}

\begin{corollary}
        \label{sup:corollary:NTN}
    Consider two arbitrary EqFs of the same system, associated with $\varphi$ and $\varphi^*$. Their local observability matrices $\mathcal{O}$ and $\mathcal{O}^*$ satisfy
    \begin{equation}
        \mathcal{O}^* = \mathcal{O} \mathbf{T}^{-1},
        \label{sup:equ:MTM}
     \end{equation}
     where $\mathbf{T}$ is the transformation matrix given by Theorem 1 (in the main manuscript).
         Moreover, if the system is unobservable, let $\mathbf{N}$ be a basis matrix for the unobservable subspace associated with $\varphi$. Then there exists a basis matrix $\mathbf{N}^*$ for the unobservable subspace associated with $\varphi^*$ such that
    \begin{equation}
        \mathbf{N}^* = \mathbf{T} \mathbf{N}.
        \label{sup:equ:NTN}
    \end{equation}
\end{corollary}
\begin{proof}
    Equation \eqref{sup:equ:MTM} can be proved by induction. First, for $k=0$, we have $\mathcal{O}_0^* = \bfH^* = \bfH \bfT^{-1} = \mathcal{O}_0 \bfT^{-1}$. Assume that $\mathcal{O}_{k-1}^* = \mathcal{O}_{k-1} \bfT^{-1}$ holds. Then,
    \begin{subequations}
        \begin{align}
        \mathcal{O}_k^* & = \mathcal{O}_{k-1}^* \bfF^* + \dot{\mathcal{O}}_{k-1}^* \\        
        & = \mathcal{O}_{k-1} \bfT^{-1} (\bfT \bfF \bfT^{-1} + \dot{\bfT} \bfT^{-1}) + (\dot{\mathcal{O}}_{k-1} \bfT^{-1} + \mathcal{O}_{k-1} \dot{\bfT}^{-1}) \\
        & = (\mathcal{O}_{k-1} \bfF + \dot{\mathcal{O}}_{k-1}) \bfT^{-1} + \mathcal{O}_{k-1} \underbrace{(\bfT^{-1} \dot{\bfT} \bfT^{-1} + \dot{\bfT}^{-1})}_{\mathbf{0}}   \\
        & = \mathcal{O}_k \bfT^{-1}.
\end{align}
    \end{subequations}
Therefore, by induction, we obtain
\begin{equation} 
    \mathcal{O}^* = \begin{bmatrix}
        \mathcal{O}_0^* \\
        \mathcal{O}_1^* \\
        \vdots \\
        \mathcal{O}_n^*
    \end{bmatrix} = \begin{bmatrix}
        \mathcal{O}_0 \\
        \mathcal{O}_1 \\
        \vdots \\
        \mathcal{O}_n
    \end{bmatrix} \bfT^{-1} = \mathcal{O} \bfT^{-1},
\end{equation}
which proves \eqref{sup:equ:MTM}.

    If the system is unobservable, let the columns of $\mathbf{N}$ form a basis for the unobservable subspace associated with $\varphi$, so that $\mathcal{O}\mathbf{N} = \mathbf{0}$. For the formulation associated with $\varphi^*$, we seek a basis matrix $\mathbf{N}^*$ satisfying $\mathcal{O}^*\mathbf{N}^* = \mathbf{0}$. Substituting \eqref{sup:equ:MTM} into this condition yields
    \begin{equation}
        (\mathcal{O} \mathbf{T}^{-1}) \mathbf{N}^* = \mathbf{0}.
        \label{sup:equ:proof_sub}
    \end{equation}
    Define $\mathbf{N}^* = \mathbf{T}\mathbf{N}$. Then,
    \begin{equation}
        \mathcal{O} \mathbf{T}^{-1} (\mathbf{T} \mathbf{N}) = \mathcal{O} (\mathbf{T}^{-1} \mathbf{T}) \mathbf{N} = \mathcal{O} \mathbf{N} = \mathbf{0}.
    \end{equation}
    Since $\mathbf{T}$ is nonsingular and the columns of $\mathbf{N}$ are linearly independent, the columns of $\mathbf{N}^* = \mathbf{T}\mathbf{N}$ are also linearly independent. Moreover, because $\text{rank}(\mathcal{O}^*) = \text{rank}(\mathcal{O})$, the null-space dimension is preserved. Therefore, $\mathbf{N}^* = \mathbf{T}\mathbf{N}$ is a valid basis for the unobservable subspace associated with $\varphi^*$. This completes the proof.
\end{proof}

\subsection{Proof of Corollary 3}
\begin{corollary}
    Consider two arbitrary EqFs of the same system, associated with the global-local maps $\varphi$ and $\varphi^*$, respectively. Their accumulated transition matrices and noise matrices are related by the transformation matrix as follows:
                \begin{align}
                        & \bfPhi_k^* = \bfT(\hat{\xi}_{k|k-1}) \bfPhi_{k} \bfT(\hat{\xi}_{k-1|k-1})^{-1}, \label{sup:equ:Phi_star_sup}\\
                        & \bfQ_{k}^* = \bfT(\hat{\xi}_{k|k-1}) \bfQ_{k} \bfT(\hat{\xi}_{k|k-1})^\top, \label{sup:equ:Q_star_sup}
                \end{align}
        where $\bfT(\hat{\xi}_{k-1|k-1})$ and $\bfT(\hat{\xi}_{k|k-1})$ are the transformation matrices from $\varphi$ to $\varphi^*$ evaluated at $\hat{\xi}_{k-1|k-1}$ and $\hat{\xi}_{k|k-1}$, respectively. 
        \label{sup:corollary:PhiQ}
\end{corollary}
\begin{proof}
\textbf{Proof of \eqref{sup:equ:Phi_star_sup}}:
    The accumulated transition matrix $\bfPhi_k = \bfPhi(\tau_q,\tau_0)$ is obtained by solving
\begin{equation}
        \frac{\text{d}}{\text{d}\tau}{\bfPhi} (\tau,\tau_0)  = \bfF_\tau \bfPhi (\tau,\tau_0), \quad \bfPhi (\tau_0,\tau_0) = \bfI,
\end{equation}
where $\bfF_\tau$ is the state-propagation Jacobian evaluated at $\hat{\xi}_\tau$. The definition of $\bfPhi^*(\tau,\tau_0)$ is analogous, with $\bfF_\tau$ replaced by $\bfF_\tau^*$.

Define $\bfPsi(\tau) = \bfT_\tau \bfPhi(\tau,\tau_0) \bfT_{\tau_0}^{-1}$, where $\bfT_\tau$ denotes the transformation matrix evaluated at $\hat{\xi}_\tau$. Differentiating $\bfPsi(\tau)$ with respect to $\tau$ yields
\begin{equation}
    \begin{split}
        \frac{\text{d}}{\text{d}\tau} \bfPsi(\tau) & = \dot{\bfT}_\tau \bfPhi (\tau,\tau_0) \bfT_{\tau_0}^{-1} + \bfT_\tau \frac{\text{d}}{\text{d}\tau} \bfPhi (\tau,\tau_0) \bfT_{\tau_0}^{-1} \\
        & = \dot{\bfT}_\tau \bfPhi (\tau,\tau_0) \bfT_{\tau_0}^{-1} + \bfT_\tau \bfF_\tau \bfPhi (\tau,\tau_0) \bfT_{\tau_0}^{-1} \\
        & = (\dot{\bfT}_\tau \bfT_\tau^{-1} + \bfT_\tau \bfF_\tau \bfT_\tau^{-1}) \bfPsi(\tau) \\
        & = \bfF_\tau^* \bfPsi(\tau).
    \end{split}
\end{equation}
Since $\bfPsi(\tau)$ satisfies the same differential equation as $\bfPhi^*(\tau,\tau_0)$ and shares the same initial condition $\bfPsi(\tau_0) =  \bfPhi^*(\tau_0,\tau_0) = \bfI$, uniqueness implies that $\bfPsi(\tau) = \bfPhi^*(\tau,\tau_0)$. Consequently,
\begin{equation}
        \bfPhi^*(\tau_q,\tau_0) = \bfT_{\tau_q} \bfPhi (\tau_q,\tau_0) \bfT_{\tau_0}^{-1}.
\end{equation}
During propagation, $\hat{\xi}_{\tau_0}$ and $\hat{\xi}_{\tau_q}$ correspond to $\hat{\xi}_{k-1|k-1}$ and $\hat{\xi}_{k|k-1}$, respectively. Therefore,
\begin{equation}
        \bfPhi_k^* = \bfT(\hat{\xi}_{k|k-1}) \bfPhi_{k} \bfT(\hat{\xi}_{k-1|k-1})^{-1}.
\end{equation}

\textbf{Proof of \eqref{sup:equ:Q_star_sup}}:
The accumulated noise matrix $\bfQ_k = \bfQ(\tau_q,\tau_0)$ is computed by integrating the noise covariance over time:
\begin{equation}    
    \bfQ(\tau_q,\tau_0) = \int_{\tau_0}^{\tau_q} \bfPhi (\tau_q,\tau) \bfG_\tau \bfQ_c {\bfG_\tau}^\top {\bfPhi} (\tau_q,\tau)^\top \text{d}\tau,
\end{equation}
where $\bfG_\tau$ is the noise propagation Jacobian evaluated at $\hat{\xi}_\tau$, and $\bfQ_c$ is the covariance of the continuous-time noise. Then, the accumulated noise matrix $\bfQ^*(\tau_q,\tau_0)$ is given by
\begin{equation}
    \begin{split}
    \bfQ^*(\tau_q,\tau_0) & = \int_{\tau_0}^{\tau_q} \bfPhi^* (\tau_q,\tau) \bfG_\tau^* \bfQ_c {\bfG_\tau^*}^\top {\bfPhi^*} (\tau_q,\tau)^\top \text{d}\tau \\
    & = \int_{\tau_0}^{\tau_q} \bfT_{\tau_q} \bfPhi (\tau_q,\tau) \bfT_\tau^{-1} \bfT_\tau \bfG_\tau \bfQ_c {\bfG_\tau}^\top \bfT_\tau^\top \bfPhi (\tau_q,\tau)^\top \bfT_{\tau_q}^\top \text{d}\tau \\
    & = \bfT_{\tau_q} \left( \int_{\tau_0}^{\tau_q} \bfPhi (\tau_q,\tau) \bfG_\tau \bfQ_c {\bfG_\tau}^\top {\bfPhi} (\tau_q,\tau)^\top \text{d}\tau \right) \bfT_{\tau_q}^\top \\
    & = \bfT_{\tau_q} \bfQ(\tau_q,\tau_0) \bfT_{\tau_q}^\top.
    \end{split}
\end{equation}
Replacing $\tau_q$ by the corresponding discrete-time instant gives
\begin{equation}
        \bfQ_{k}^* = \bfT(\hat{\xi}_{k|k-1}) \bfQ_{k} \bfT(\hat{\xi}_{k|k-1})^\top,
\end{equation}
which completes the proof.
\end{proof}

\newpage
\subsection{Derivation of Algorithm 2 (Transforming Correction)}
Algorithm 2 is derived by tracking an equivalent covariance $\bfP := \bfT^{-1}\bfP^*\bfT^{-\top}$ instead of tracking T-EqF covariance $\bfP^*$ directly.
Table~\ref{sup:tab:alg_compare} presents the derivation side by side. The left column shows the naive T-EqF implementation (i.e., T-EqF(Naive) in the main manuscript), which propagates $\bfP^*$ using T-EqF Jacobians. The right column derives the corresponding steps of Algorithm 2 by substituting the transformation relationship
\begin{equation}
    \bfP = \bfT^{-1} \bfP^* \bfT^{-\top}
    \label{sup:equ:P_relationship}
\end{equation}
into the corresponding initialization, propagation, and correction steps in the left column. For each step, $\bfT$ is evaluated at the resulting state estimate. The blue annotations in the right column show the intermediate algebra behind this substitution. As a result, Algorithm 2 operates entirely on $\bfP$ using the computationally efficient Jacobians of the auxiliary EqF (e.g., SD-EqF), while maintaining the transformation relationship \eqref{sup:equ:P_relationship} at the end of each step (Lines 2, 4, and 8).

\begin{table}[h]
\centering
\small
\caption{Comparison of the naive implementation and Algorithm 2.}
\label{sup:tab:alg_compare}
\renewcommand{\arraystretch}{1.} 
\begin{tabularx}{\textwidth}{l p{0.34\textwidth} p{0.54\textwidth}}
\toprule
& \textbf{Naive Implementation ($\mathbf{P}^*$)} & \textbf{Transforming Correction ($\mathbf{P}$)} \\ \midrule
&\multicolumn{1}{l}{\textqcr{//Initialization:}}&\multicolumn{1}{l}{\textqcr{//Initialization:}}\\

1& $\hat{\xi}_{0|0} = \hat{\xi}_{0}$;& $\hat{\xi}_{0|0} = \hat{\xi}_{0}$; \\ 
2& \par $\mathbf{P}^*_{0|0} =\mathbf{P}^*_0$;& $\mathbf{P}_{0|0} = \mathbf{T}(\hat{\xi}_{0})^{-1} \mathbf{P}_{0}^* \mathbf{T}(\hat{\xi}_{0})^{-\top}$;\\ \addlinespace\addlinespace

&\multicolumn{1}{l}{\textqcr{//Propagation from $t_{k-1}$ to $t_{k}$:}} &\multicolumn{1}{l}{\textqcr{//Propagation from $t_{k-1}$ to $t_{k}$:}} \\
3& Propagate $\hat{\xi}_{k-1|k-1}$ to $\hat{\xi}_{k|k-1}$; & Propagate $\hat{\xi}_{k-1|k-1}$ to $\hat{\xi}_{k|k-1}$; \\

4& $\bfP_{k|k-1}^* = \bfPhi_{k}^* \bfP_{k-1|k-1}^* {\bfPhi_{k}^*}^\top + \bfQ_{k}^*$; & 
$\mathbf{P}_{k|k-1} = \mathbf{\Phi}_k \mathbf{P}_{k-1|k-1} \mathbf{\Phi}_k^\top + {\mathbf{Q}_k}$; \\
&&
\par \textnote{Given}  \textnote{$\bfP_{k-1|k-1} = \bfT(\hat{\xi}_{k-1|k-1})^{-1} \bfP_{k-1|k-1}^* \bfT(\hat{\xi}_{k-1|k-1})^{-\top}$,} \par
\textnote{$\bfPhi_k = \bfT(\hat{\xi}_{k|k-1})^{-1}\bfPhi^*_k \bfT(\hat{\xi}_{k-1|k-1})$,} \par \textnote{and ${\bfQ}_k = \bfT(\hat{\xi}_{k|k-1})^{-1} \bfQ^*_k \bfT(\hat{\xi}_{k|k-1})^{-\top}$,} \par
\textnote{we have:} \par \textnote{$\bfP_{k|k-1} = \bfT(\hat{\xi}_{k|k-1})^{-1} \bfP_{k|k-1}^* \bfT(\hat{\xi}_{k|k-1})^{-\top}$.} \\ \addlinespace \addlinespace

&\multicolumn{1}{l}{\textqcr{//Correction at $t_k$:}} &\multicolumn{1}{l}{\textqcr{//Correction at $t_k$:}} \\
5& $\bfK^* = \bfP_{k|k-1}^* \bfH_k^{*\top} (\bfH_k^* \bfP_{k|k-1}^* \bfH_k^{*\top} + \bfR)^{-1}$; &  $\bfK = \bfP_{k|k-1} \bfH_k^\top (\bfH_k \bfP_{k|k-1} \bfH_k^\top + \bfR)^{-1}$; \\
6& $\bfP_{k|k}^* = (\bfI - \bfK^* \bfH_k^*) \bfP_{k|k-1}^*$; & $\bfP_{k|k} = (\bfI - \bfK \bfH_k) \bfP_{k|k-1}$; \\
&& \textnote{Given $\bfH_k = \bfH_k^* \bfT(\hat{\xi}_{k|k-1})$, we have:} \par \textnote{$\bfK = \bfT(\hat{\xi}_{k|k-1})^{-1} \bfK^*$,} \par \textnote{and $\bfP_{k|k} = \bfT(\hat{\xi}_{k|k-1})^{-1} \bfP_{k|k}^* \bfT(\hat{\xi}_{k|k-1})^{-\top}$.} \par \textnote{Note that $\bfT$ is still evaluated at $\hat{\xi}_{k|k-1}$.} \\ 

&&\multicolumn{1}{l}{\textqcr{//Transforming Correction at $t_k$:}} \\
7& $\hat{\xi}_{k|k} = \varphi{^*}^{-1}(\bfK^*\tilde{\bfz}_k)|_{\hat{\xi}=\hat{\xi}_{k|k-1}}$. &  $\hat{\xi}_{k|k} = \left. {\varphi^*}^{-1}\big( \mathbf{T} (\hat{\xi}_{k|k-1}) \mathbf{K} \tilde{\mathbf{z}}_{k}\big) \right|_{\hat{\xi} = \hat{\xi}_{k|k-1}}$;\\
8& -- & $\mathbf{P}_{k|k} \leftarrow \mathcal{T} \mathbf{P}_{k|k} \mathcal{T}^\top $ with $\mathcal{T} = \mathbf{T}(\hat{\xi}_{k|k})^{-1} \mathbf{T}(\hat{\xi}_{k|k-1})$. \\

&& \textnote{ After the transformation, we have:} \par \textnote{$\mathbf{T} (\hat{\xi}_{k|k-1}) \mathbf{K} \tilde{\mathbf{z}}_{k} = \mathbf{K}^* \tilde{\mathbf{z}}_{k}$,} \par \textnote{and $\bfP_{k|k} = \bfT(\hat{\xi}_{k|k})^{-1} \bfP_{k|k}^* \bfT(\hat{\xi}_{k|k})^{-\top}$.} \par \textnote{Note that $\bfT$ is evaluated at $\hat{\xi}_{k|k}$ as expected.} \\ \addlinespace\addlinespace
\bottomrule
\end{tabularx}
\end{table}

\subsection{Algorithm 2 When \texorpdfstring{$\phi_{\xi^{\circ}}$}{phi\_xi0} or \texorpdfstring{$\phi_{\xi^{\circ}}^*$}{phi\_xi0*} Is Non-Bijective}
The implementation presented in Table~\ref{sup:tab:alg_compare} follows a standard EKF structure under the assumption that the maps $\phi_{\xi^{\circ}}$ and $\phi_{\xi^{\circ}}^*$ are bijective. In that case, the estimated group elements $\hat{X}$ and $\hat{X}^*$ are uniquely determined by the state estimate $\hat{\xi}$ and therefore do not need to be maintained explicitly. Although this bijectivity holds for the VINS estimators analyzed in this work, it may not hold for other manifolds or applications. When $\phi_{\xi^{\circ}}$ or $\phi_{\xi^{\circ}}^*$ is non-bijective, the group-element estimates must be maintained explicitly.
For completeness, Table~\ref{sup:tab:alg_compare2} compares the naive implementation with Algorithm 2 for the non-bijective case. (Algorithm 1 would also need to be modified in this setting, but the required changes are straightforward and are therefore omitted.)

\begin{table}[h]
\centering
\small
\caption{Comparison of the naive implementation and Algorithm 2 when $\phi_{\xi^{\circ}}$ or $\phi_{\xi^{\circ}}^*$ is non-bijective.}
\label{sup:tab:alg_compare2}
\renewcommand{\arraystretch}{1.} 
\begin{tabularx}{\textwidth}{r p{0.34\textwidth} p{0.54\textwidth}}
\toprule
& \textbf{Naive Implementation ($\mathbf{P}^*$)} & \textbf{Transforming Correction ($\mathbf{P}$)} \\ \midrule
&\multicolumn{1}{l}{\textqcr{//Initialization:}}&\multicolumn{1}{l}{\textqcr{//Initialization:}}\\

1& $\xi^{\circ} = \hat\xi_0$;& $\xi^{\circ} = \hat\xi_0$; \\ 
2&  -- & $\hat{X}_{0|0} = I$; \\ 
3& $\hat{X}^*_{0|0} = I$;& $\hat{X}_{0|0}^* = I$; \\ 
4& \par $\mathbf{P}^*_{0|0} =\mathbf{P}^*_0$;& $\mathbf{P}_{0|0} = \mathbf{T}(\hat{X}_{0},\hat{X}_{0}^*)^{-1} \mathbf{P}_{0}^* \mathbf{T}(\hat{X}_{0},\hat{X}_{0}^*)^{-\top}$;\\ \addlinespace\addlinespace

&\multicolumn{1}{l}{\textqcr{//Propagation from $t_{k-1}$ to $t_{k}$:}} &\multicolumn{1}{l}{\textqcr{//Propagation from $t_{k-1}$ to $t_{k}$:}} \\
5& -- & Propagate $\hat{X}_{k-1|k-1}$ to $\hat{X}_{k|k-1}$; \\
6& Propagate $\hat{X}_{k-1|k-1}^*$ to $\hat{X}_{k|k-1}^*$; & Propagate $\hat{X}_{k-1|k-1}^*$ to $\hat{X}_{k|k-1}^*$; \\

7& $\bfP_{k|k-1}^* = \bfPhi_{k}^* \bfP_{k-1|k-1}^* {\bfPhi_{k}^*}^\top + \bfQ_{k}^*$; & 
$\mathbf{P}_{k|k-1} = \mathbf{\Phi}_k \mathbf{P}_{k-1|k-1} \mathbf{\Phi}_k^\top + \mathbf{Q}_k$;
\\
&& \textnote{As in Table 1, we have:} \par \textnote{$\bfP_{k|k-1} = \bfT(\hat{X}_{k|k-1},\hat{X}_{k|k-1}^*)^{-1} \bfP_{k|k-1}^* \bfT(\hat{X}_{k|k-1},\hat{X}_{k|k-1}^*)^{-\top}$.} \\
\addlinespace\addlinespace

&\multicolumn{1}{l}{\textqcr{//Correction at $t_k$:}} &\multicolumn{1}{l}{\textqcr{//Correction at $t_k$:}} \\
8& $\bfK^* = \bfP_{k|k-1}^* \bfH_k^{*\top} (\bfH_k^* \bfP_{k|k-1}^* \bfH_k^{*\top} + \bfR)^{-1}$; &  $\bfK = \bfP_{k|k-1} \bfH_k^\top (\bfH_k \bfP_{k|k-1} \bfH_k^\top + \bfR)^{-1}$; \\
9& $\bfP_{k|k}^* = (\bfI - \bfK^* \bfH_k^*) \bfP_{k|k-1}^*$; & $\bfP_{k|k} = (\bfI - \bfK \bfH_k) \bfP_{k|k-1}$; \\
&& 
\textnote{As in Table 1, we have:} \par
\textnote{$\bfK = \bfT(\hat{X}_{k|k-1},\hat{X}_{k|k-1}^*)^{-1} \bfK^*$,} \par \textnote{and $\bfP_{k|k} = \bfT(\hat{X}_{k|k-1},\hat{X}_{k|k-1}^*)^{-1} \bfP_{k|k}^* \bfT(\hat{X}_{k|k-1},\hat{X}_{k|k-1}^*)^{-\top}$.} \par \textnote{Note that $\bfT$ is still evaluated at $\hat{X}_{k|k-1}$ and $\hat{X}_{k|k-1}^*$.} \\

&&\multicolumn{1}{l}{\textqcr{//Transforming Correction at $t_k$:}} \\
10& --
&$\Delta = \textrm{D}_{E|I}\phi_{\xi^{\circ }}(E)^{\dagger} \textrm{D}_{\varepsilon|0}\vartheta^{-1}(\varepsilon) \bfK\tilde{\bfz}_k$;
\\
11& $\Delta^* = \textrm{D}_{E|I}\phi_{\xi^{\circ }}^{*} (E)^{\dagger} \textrm{D}_{\varepsilon|0}\vartheta{^*}^{-1}(\varepsilon) \bfK^*\tilde{\bfz}_k$;
&$\Delta^* = \textrm{D}_{E|I}\phi_{\xi^{\circ }}^{*} (E)^{\dagger} \textrm{D}_{\varepsilon|0}\vartheta{^*}^{-1}(\varepsilon) \mathbf{T}(\hat{X}_{k|k-1},\hat{X}_{k|k-1}^* ) \bfK\tilde{\bfz}_k$;
\\
12& -- & $\hat{X}_{k|k} =\exp(\Delta) \hat{X}_{k|k-1}$;  \\
13& $\hat{X}_{k|k}^* =\exp(\Delta^*) \hat{X}_{k|k-1}^*$. & $\hat{X}_{k|k}^* =\exp(\Delta^*) \hat{X}_{k|k-1}^*$;  \\
14& --  &  $\mathcal{T} = \mathbf{T}(\hat{X}_{k|k}, \hat{X}_{k|k}^*)^{-1} \mathbf{T}(\hat{X}_{k|k-1}, \hat{X}_{k|k-1}^*)$; \\
15& -- & $\mathbf{P}_{k|k} \leftarrow \mathcal{T} \mathbf{P}_{k|k} \mathcal{T}^\top $.\\
&& \textnote{After the transformation, we have:} \par
\textnote{$\mathbf{T} (\hat{X}_{k|k-1},\hat{X}_{k|k-1}^*) \mathbf{K} \tilde{\mathbf{z}}_{k} = \mathbf{K}^* \tilde{\mathbf{z}}_{k}$,} \par \textnote{ and $\bfP_{k|k} = \bfT(\hat{X}_{k|k},\hat{X}_{k|k}^*)^{-1} \bfP_{k|k}^* \bfT(\hat{X}_{k|k},\hat{X}_{k|k}^*)^{-\top}$.} \par \textnote{Note that $\bfT$ is evaluated at $\hat{X}_{k|k}$ and $\hat{X}_{k|k}^*$ as expected.} \\ \addlinespace\addlinespace
\bottomrule
\end{tabularx}
\end{table}

As noted in Remark 1 (in the main manuscript), when bijectivity is lost, the transformation matrix $\mathbf{T}$ becomes a function of the group-element estimates $\hat{X}$ and $\hat{X}^*$, i.e., $\mathbf{T} = \mathbf{T}(\hat{X}, \hat{X}^*)$. Consequently, both estimates must be tracked throughout the filtering process:
\begin{itemize}
    \item During the propagation step (Lines 5--6), both $\hat{X}$ and $\hat{X}^*$ are updated via their respective kinematic lifts \cite{vangoorEquivariantFilterEqF2023}.
    \item During the correction step, the group-element corrections $\Delta$ and $\Delta^*$ are computed (Lines 10--11) and then used to update $\hat{X}$ and $\hat{X}^*$ (Lines 12--13) \cite{ge2022equivariant}.
\end{itemize}
Although this formulation requires maintaining two group-element estimates, covariance propagation and correction are still performed only once. Hence, the additional computational overhead remains limited relative to the bijective case.

\section{Transformations Between Representative EqFs}
Due to space limitations, the main manuscript presents only the transformation from ESKF to SD-EqF and omits the detailed derivation. In this section, we provide that derivation in full. We also present transformations involving other representative EqFs, including right-invariant EKF (RI-EKF), left-invariant EKF (LI-EKF), and invariant SD-EqF (ISD-EqF).

\subsection{Transformation from ESKF to SD-EqF}
We first recall Theorem 1 in the main manuscript. Given two EqFs associated with $\varphi$ and $\varphi^*$, their error states are related by the linear transformation
        \begin{equation}
                \bfvarepsilon = \bfT \bfvarepsilon^*,
                \label{sup:equ:relationship}
        \end{equation}
where $\bfT$ is the nonsingular matrix
 \begin{equation}
            \bfT =  \rfD_{\xi|\hat{\xi}} \varphi (\xi) \cdot  \rfD_{\bfvarepsilon^*|\bf0} \varphi{^*}^{-1}(\bfvarepsilon^*).%
        \label{sup:equ:T_general26}%
\end{equation}
The transformation matrix from ESKF to SD-EqF is computed as follows.\footnote{For convenience, we reverse the notation used in the main manuscript: the starred variables $(\bfvarepsilon^*, \varphi^*)$ now refer to ESKF, whereas the unstarred variables $(\bfvarepsilon, \varphi)$ refer to SD-EqF, RI-EKF, LI-EKF, or ISD-EqF.}

	\textbf{Step 1. Compute the global-local maps.}
Let $\varphi$ denote the global-local map of SD-EqF. To derive $\varphi$, we first compute $\phi_{\hat{X}^{-1}}$:
    \begin{align}
        \phi_{\hat{X}^{-1}}(\xi) = \left( 
            \begin{array}{c}
                 \bfA \hat{C}^{-1},\\ \mathbf{Ad}_{\Gamma(\hat{C})}^{\lor}\bfb + \hat{\gamma}^{\lor},\\ \bff-\hat{p}
            \end{array}
        \right).
    \end{align}
Since $\phi_{\xi}^{\circ }$ is bijective, $\hat{X} = (\hat{C},\hat{\gamma},\hat{p})$ is uniquely determined by $\hat{\xi} = (\hat{\bfA}, \hat{\bfb}, \hat{\bff})$. Therefore, the above expression can be rewritten as
\begin{equation}
    \phi_{\hat{X}^{-1}}(\xi) = \left(
        \begin{array}{c}
             \bfA\hat\bfA^{-1},\\ \mathbf{Ad}_{\Gamma(\hat{\bfA})}^{\lor}(\bfb -\hat{\bfb}),\\ \bff-\hat{\bff}
        \end{array}
    \right).
\end{equation}
Given that $\vartheta(e) = (\log(e_{\mathbf{A}})^\lor, e_{\mathbf{b}}, e_{\bff} )$, we have 
\begin{equation}
    \begin{split}
        \varphi(\xi) & = \vartheta \circ \phi_{\hat{X}^{-1}}(\xi) \\
        & = \left(
            \begin{array}{c}
                \log (\bfA\hat\bfA^{-1})^{\lor}, \\
                \mathbf{Ad}_{\Gamma(\hat{\bfA})}^{\lor}(\bfb -\hat{\bfb}), \\
                \bff-\hat{\bff}
            \end{array}
        \right).
    \end{split}
\end{equation}

Let $\varphi^*$ denote the global-local map of ESKF, which coincides with its error-state and is given by
\begin{equation}
        \begin{split}
                \varphi^*(\xi) &=  \left(
                          \begin{array}{c}
                                \log (\hat{\bfR}^{-1} \bfR)^{\lor},
                                \bfv - \hat{\bfv},
                                \bfp - \hat{\bfp}, \\
                                \bfb - \hat{\bfb},\\
                                \bff - \hat{\bff} 
                        \end{array}
                \right),
        \end{split}
\end{equation} 
with 
\begin{equation}
    {\varphi^*}^{-1}(\bfvarepsilon^*) = \left(
        \begin{array}{c}
            \hat{\bfR} \exp([\bfvarepsilon^*_{\mathbf{R}}]_\times), \hat{\bfv} + \bfvarepsilon^*_{\mathbf{v}}, \hat{\bfp} + \bfvarepsilon^*_{\mathbf{p}}, \\
            \hat{\bfb} + \bfvarepsilon^*_{\mathbf{b}},\\
            \hat{\bff} + \bfvarepsilon^*_{\mathbf{f}}
        \end{array}
    \right).
\end{equation}

	\textbf{Step 2. Compute the Jacobians.}
\begin{equation}
\rfD_{\bfvarepsilon^*|\bf0} \varphi{^*}^{-1}(\bfvarepsilon^*)[\bfvarepsilon^*] =  \left(
    \begin{array}{c}
        \hat{\bfR} [\bfvarepsilon^*_{\mathbf{R}}]_\times, 
        \bfvarepsilon^*_{\mathbf{v}}, 
        \bfvarepsilon^*_{\mathbf{p}},\\
        \bfvarepsilon^*_{\mathbf{b}},\\
        \bfvarepsilon^*_{\mathbf{f}}
    \end{array}
\right),
\label{sup:equ:Jacobian_varphi_star_inv}
\end{equation}
\begin{equation}
    \begin{split}
         \rfD_{\xi|\hat{\xi}} \varphi (\xi) [\mu] &= \left(
        \begin{array}{c}
            \big(\mu_\bfA\hat{\bfA}^{-1}\big)^{\lor}, \\
            \mathbf{Ad}_{\Gamma(\hat{\bfA})}^{\lor} \mu_{\bfb}, \\
            \mu_{\bff}
        \end{array}
    \right) \\
    & = \left(
        \begin{array}{c}
            (\mu_\bfR \hat{\bfR}^{-1})^{\lor}, \mu_\bfv - \mu_\bfR \hat{\bfR}^{-1} \hat{\bfv}, \mu_\bfp - \mu_\bfR \hat{\bfR}^{-1} \hat{\bfp},\\
            \mathbf{Ad}_{\Gamma(\hat{\bfA})}^{\lor} \mu_{\bfb}, \\
            \mu_{\bff}
        \end{array}
    \right),
    \end{split}
    \label{sup:equ:Jacobian_varphi}
\end{equation}
where $\mu = (\mu_\bfA, \mu_{\bfb}, \mu_{\bff}) = \big((\mu_\bfR, \mu_\bfv, \mu_\bfp), \mu_{\bfb}, \mu_{\bff}\big)$ belongs to the tangent space of $\mathcal{M}$ at $\hat{\xi}$. Combining \eqref{sup:equ:Jacobian_varphi_star_inv} and \eqref{sup:equ:Jacobian_varphi} yields
\begin{equation}
    \begin{split}
         \bfvarepsilon &=   \rfD_{\xi|\hat{\xi}} \varphi (\xi) \cdot  \rfD_{\bfvarepsilon^*|\bf0} \varphi{^*}^{-1}(\bfvarepsilon^*)[\bfvarepsilon^*]\\ & = \left(
        \begin{array}{c}
            (\hat{\bfR} [\bfvarepsilon^*_{\mathbf{R}}]_\times \hat{\bfR}^{-1})^{\lor}, 
            \bfvarepsilon^*_{\mathbf{v}} - \hat{\bfR} [\bfvarepsilon^*_{\mathbf{R}}]_\times \hat{\bfR}^{-1} \hat{\bfv}, 
            \bfvarepsilon^*_{\mathbf{p}} - \hat{\bfR} [\bfvarepsilon^*_{\mathbf{R}}]_\times \hat{\bfR}^{-1} \hat{\bfp},\\
            \mathbf{Ad}_{\Gamma(\hat{\bfA})}^{\lor} \bfvarepsilon^*_{\mathbf{b}},\\
            \bfvarepsilon^*_{\mathbf{f}}
        \end{array}
    \right)\\
        & = \left(
            \begin{array}{c}
                \hat{\bfR}\bfvarepsilon^*_\bfR, [\hat{\bfv}]_\times\hat{\bfR}\bfvarepsilon^*_\bfR + \bfvarepsilon^*_\bfv, [\hat{\bfp}]_\times\hat{\bfR}\bfvarepsilon^*_\bfR + \bfvarepsilon^*_\bfp,\\
                \mathbf{Ad}_{\Gamma(\hat{\bfA})}^{\lor} \bfvarepsilon^*_{\mathbf{b}},\\
                \bfvarepsilon^*_{\mathbf{f}}
            \end{array}
        \right).
    \end{split}
    \label{sup:equ:Jacobian_product}
\end{equation}

	\textbf{Step 3. Compute the transformation matrix.}
Rewriting \eqref{sup:equ:Jacobian_product} in matrix form yields the transformation matrix from ESKF to SD-EqF:
\begin{equation}
    \underbrace{
    \begin{bmatrix}
        \bfvarepsilon_\bfR \\
        \bfvarepsilon_\bfv \\
        \bfvarepsilon_\bfp \\
        \bfvarepsilon_\bfb \\
        \bfvarepsilon_\bff
    \end{bmatrix}}_{\begin{smallmatrix}
        \text{SD-EqF} \\ \text{error-state} \\ 
    \end{smallmatrix}} = 
    \underbrace{
    \begin{bmatrix}
        \hat{\bfR} & \mathbf{0} & \mathbf{0} & \mathbf{0} & \mathbf{0} \\
        {[\hat{\bfv}]}_\times\hat{\bfR} & \mathbf{I}_3 & \mathbf{0} & \mathbf{0} & \mathbf{0} \\
        {[\hat{\bfp}]}_\times\hat{\bfR} & \mathbf{0} & \mathbf{I}_3 & \mathbf{0} & \mathbf{0} \\
        \mathbf{0} & \mathbf{0} & \mathbf{0} & \mathbf{Ad}_{\Gamma(\hat{\bfA})}^{\lor} & \mathbf{0} \\
        \mathbf{0} & \mathbf{0} & \mathbf{0} & \mathbf{0} & \mathbf{I}_{3m}
    \end{bmatrix}}_{\bfT_{\text{ESKF}}^\text{SD-EqF}} \underbrace{\begin{bmatrix}
        \bfvarepsilon^*_\bfR \\
        \bfvarepsilon^*_\bfv \\
        \bfvarepsilon^*_\bfp \\
        \bfvarepsilon^*_\bfb \\
        \bfvarepsilon^*_\bff
    \end{bmatrix}}_{\begin{smallmatrix}
        \text{ESKF} \\ \text{error-state} \\ 
    \end{smallmatrix}}.
\end{equation}

\subsection{Transformation from ESKF to RI-EKF}
In RI-EKF, the IMU state and the landmark states are defined on $\mathbf{SE}_{2+m}(3)$. The global-local map (i.e., the error-state definition) of RI-EKF is given by
\begin{equation}
    \varphi(\xi) =\begin{pmatrix}
        \log(\bfR \hat\bfR^{-1})^\lor, \bfJ \big( \bfv - \bfR \hat{\bfR}^{-1}\hat{\bfv} \big), \bfJ \big( \bfp - \bfR \hat{\bfR}^{-1}\hat{\bfp} \big),\\
        \bfb - \hat{\bfb},\\
         \bfJ \big( \bff_1 - \bfR \hat{\bfR}^{-1}\hat{\bff}_1 \big), \dots, \bfJ \big( \bff_m - \bfR \hat{\bfR}^{-1}\hat{\bff}_m \big)
    \end{pmatrix},
\end{equation}
with $\bfJ = \bfJ_{l}^{-1}\left( \log(\bfR \hat\bfR^{-1})^\lor \right)$. Its Jacobian is
\begin{equation}
    \begin{split}
         \rfD_{\xi|\hat{\xi}} \varphi (\xi) [\mu] =\left(
        \begin{array}{c}
            (\mu_\bfR \hat{\bfR}^{-1})^{\lor}, \mu_\bfv - \mu_\bfR \hat{\bfR}^{-1} \hat{\bfv}, \mu_\bfp - \mu_\bfR \hat{\bfR}^{-1} \hat{\bfp},\\
             \mu_{\bfb}, \\
            \mu_{\bff_1} - \mu_\bfR \hat{\bfR}^{-1} \hat{\bff}_1, \dots, \mu_{\bff_m} - \mu_\bfR \hat{\bfR}^{-1} \hat{\bff}_m
        \end{array}
    \right).
    \end{split}
    \label{sup:equ:Jacobian_varphi_ri}
\end{equation}

Combining \eqref{sup:equ:Jacobian_varphi_star_inv} and \eqref{sup:equ:Jacobian_varphi_ri} yields
\begin{equation}
    \begin{split}
           \bfvarepsilon &=   \rfD_{\xi|\hat{\xi}} \varphi (\xi) \cdot  \rfD_{\bfvarepsilon^*|\bf0} \varphi{^*}^{-1}(\bfvarepsilon^*)[\bfvarepsilon^*]\\ &= \left(
            \begin{array}{c}
                \hat{\bfR}\bfvarepsilon^*_\bfR, [\hat{\bfv}]_\times\hat{\bfR}\bfvarepsilon^*_\bfR + \bfvarepsilon^*_\bfv, [\hat{\bfp}]_\times\hat{\bfR}\bfvarepsilon^*_\bfR + \bfvarepsilon^*_\bfp,\\
                \bfvarepsilon^*_{\mathbf{b}},\\
                {[\hat{\bff}_1]}_\times\hat{\bfR}\bfvarepsilon^*_\bfR + \bfvarepsilon^*_{\mathbf{f}_1}, \dots, [\hat{\bff}_m]_\times\hat{\bfR}\bfvarepsilon^*_\bfR + \bfvarepsilon^*_{\mathbf{f}_m}
            \end{array}
        \right).
    \end{split}
\end{equation}
Rewriting the above equation in matrix form yields the transformation matrix from ESKF to RI-EKF:
\begin{equation}\underbrace{
    \begin{bmatrix}
        \bfvarepsilon_\bfR \\
        \bfvarepsilon_\bfv \\
        \bfvarepsilon_\bfp \\
        \bfvarepsilon_\bfb \\
        \bfvarepsilon_{\bff_1} \\
        \vdots \\
        \bfvarepsilon_{\bff_m}
    \end{bmatrix}}_{{\begin{smallmatrix}
        \text{RI-EKF} \\ \text{error-state} \\ 
    \end{smallmatrix}}} = 
    \underbrace{\begin{bmatrix}
        \hat{\bfR} & \mathbf{0} & \mathbf{0} & \mathbf{0} & \mathbf{0} & \dots & \mathbf{0} \\
        {[\hat{\bfv}]}_\times\hat{\bfR} & \mathbf{I}_3 & \mathbf{0} & \mathbf{0} & \mathbf{0} & \dots & \mathbf{0} \\
        {[\hat{\bfp}]}_\times\hat{\bfR} & \mathbf{0} & \mathbf{I}_3 & \mathbf{0} & \mathbf{0} & \dots & \mathbf{0} \\
        \mathbf{0} & \mathbf{0} & \mathbf{0} & \mathbf{I}_6 & \mathbf{0} & \dots & \mathbf{0} \\
        {[\hat{\bff}_1]}_\times\hat{\bfR} & \mathbf{0} & \mathbf{0} & \mathbf{0} & \mathbf{I}_3 & \dots & \mathbf{0} \\
        \vdots & \vdots & \vdots & \vdots & \vdots & \ddots & \vdots \\
        {[\hat{\bff}_m]}_\times\hat{\bfR} & \mathbf{0} & \mathbf{0} & \mathbf{0} & \mathbf{0} & \dots & \mathbf{I}_3
    \end{bmatrix}}_{{\bfT_{\text{ESKF}}^\text{RI-EKF}}} 
    \underbrace{
    \begin{bmatrix}
        \bfvarepsilon^*_\bfR \\
        \bfvarepsilon^*_\bfv \\
        \bfvarepsilon^*_\bfp \\
        \bfvarepsilon^*_\bfb \\
        \bfvarepsilon^*_{\bff_1} \\
        \vdots \\
        \bfvarepsilon^*_{\bff_m}
    \end{bmatrix}}_{\begin{smallmatrix}
        \text{ESKF} \\ \text{error-state} \\ 
    \end{smallmatrix}}.
\end{equation}

\subsection{Transformation from ESKF to LI-EKF}
LI-EKF is also defined on $\mathbf{SE}_{2+m}(3)$ but uses a left-invariant error definition:
\begin{equation}
    \varphi(\xi) =\begin{pmatrix}
        \log( \hat\bfR^{-1} \bfR)^\lor, \bfJ \big( \hat\bfR^{-1} (\bfv - \hat{\bfv}) \big), \bfJ \big( \hat\bfR^{-1} (\bfp - \hat{\bfp}) \big),\\
        \bfb - \hat{\bfb},\\
         \bfJ \big( \hat\bfR^{-1} (\bff_1 - \hat{\bff}_1) \big), \dots, \bfJ \big( \hat\bfR^{-1} (\bff_m - \hat{\bff}_m) \big)
    \end{pmatrix},
\end{equation}
with $\bfJ = \bfJ_{l}^{-1}\left( \log(\hat\bfR^{-1} \bfR)^\lor \right)$. Its Jacobian is
\begin{equation}
    \begin{split}
         \rfD_{\xi|\hat{\xi}} \varphi (\xi) [\mu] =\left(
        \begin{array}{c}
            (\hat{\bfR}^{-1}\mu_\bfR )^{\lor}, \hat{\bfR}^{-1} \mu_\bfv, \hat{\bfR}^{-1} \mu_\bfp ,\\
             \mu_{\bfb}, \\
            \hat{\bfR}^{-1} \mu_{\bff_1}, \dots, \hat{\bfR}^{-1} \mu_{\bff_m},
        \end{array}
    \right).
    \end{split}
    \label{sup:equ:Jacobian_varphi_li}
\end{equation}
Combining \eqref{sup:equ:Jacobian_varphi_star_inv} and \eqref{sup:equ:Jacobian_varphi_li} yields
\begin{equation}
    \begin{split}
           \bfvarepsilon &=   \rfD_{\xi|\hat{\xi}} \varphi (\xi) \cdot  \rfD_{\bfvarepsilon^*|\bf0} \varphi{^*}^{-1}(\bfvarepsilon^*)[\bfvarepsilon^*]\\ &= \left(
            \begin{array}{c}
                \bfvarepsilon^*_\bfR, \hat{\bfR}^{-1}\bfvarepsilon^*_\bfv, \hat{\bfR}^{-1}\bfvarepsilon^*_\bfp,\\
                \bfvarepsilon^*_{\mathbf{b}},\\
                \hat{\bfR}^{-1} \bfvarepsilon^*_{\mathbf{f}_1}, \dots, \hat{\bfR}^{-1} \bfvarepsilon^*_{\mathbf{f}_m}
            \end{array}
        \right).
    \end{split} 
\end{equation}
Rewriting the above equation in matrix form yields the transformation matrix from ESKF to LI-EKF:
\begin{equation}\underbrace{
    \begin{bmatrix}
        \bfvarepsilon_\bfR \\
        \bfvarepsilon_\bfv \\
        \bfvarepsilon_\bfp \\
        \bfvarepsilon_\bfb \\
        \bfvarepsilon_{\bff_1} \\
        \vdots \\
        \bfvarepsilon_{\bff_m}  
    \end{bmatrix}}_{{\begin{smallmatrix}
        \text{LI-EKF} \\ \text{error-state} \\
    \end{smallmatrix}}} = 
    \underbrace{\begin{bmatrix}
        \mathbf{I}_3 & \mathbf{0} & \mathbf{0} & \mathbf{0} & \mathbf{0} & \dots & \mathbf{0} \\
        \mathbf{0} & \hat{\bfR}^{-1} & \mathbf{0} & \mathbf{0} & \mathbf{0} & \dots & \mathbf{0} \\
        \mathbf{0} & \mathbf{0} & \hat{\bfR}^{-1} & \mathbf{0} & \mathbf{0} & \dots & \mathbf{0} \\
        \mathbf{0} & \mathbf{0} & \mathbf{0} & \mathbf{I}_6 & \mathbf{0} & \dots & \mathbf{0} \\
        \mathbf{0} & \mathbf{0} & \mathbf{0} & \mathbf{0} & \hat{\bfR}^{-1} & \dots & \mathbf{0} \\
        \vdots & \vdots & \vdots & \vdots & \vdots & \ddots & \vdots \\
        \mathbf{0} & \mathbf{0} & \mathbf{0} & \mathbf{0} & \mathbf{0} & \dots & \hat{\bfR}^{-1}
    \end{bmatrix}}_{{\bfT_{\text{ESKF}}^\text{LI-EKF}}}
      \underbrace{\begin{bmatrix}
        \bfvarepsilon^*_\bfR \\
        \bfvarepsilon^*_\bfv \\
        \bfvarepsilon^*_\bfp \\
        \bfvarepsilon^*_\bfb \\
        \bfvarepsilon^*_{\bff_1} \\
        \vdots \\
        \bfvarepsilon^*_{\bff_m}
    \end{bmatrix}}_{\begin{smallmatrix}
        \text{ESKF} \\ \text{error-state} \\
    \end{smallmatrix}}.   
\end{equation}  

\subsection{Transformation from ESKF to ISD-EqF}
ISD-EqF is based on the invariant semi-direct bias group $\mathbf{G} =  \mathbf{SE}_{2+m}(3)\ltimes \mathfrak{se}(3)$, which combines the right-invariant error definition for the landmark states with SD-EqF error definition for the IMU states (including the biases).
Denote an element of $\mathbf{G}$ by $X = (B,\gamma)$, where $B = (R,a,b,p_1,\dots,p_m) \in \mathbf{SE}_{2+m}(3)$ and $\gamma \in \mathfrak{se}(3)$. The group multiplication is given by 
\begin{equation}
    X_1 X_2 = \left( B_1 B_2, \gamma_1 + \mathbf{Ad}_{\Gamma ({B_1})}(\gamma_2) \right), 
\end{equation}
where $\Gamma(B) = (R,a) \in \mathbf{SE}(3)$.
To define the group action, we first regroup the system state as follows: $\xi = (\bfD, \bfb)$, where $\bfD = (\bfR,\bfv,\bfp, \bff_1,\dots,\bff_m) \in \mathbf{SE}_{2+m}(3)$ and $\bfb = (\bfb_{\omega},\bfb_a) \in \mathbb{R}^6$.
The group action on the state space $\mathcal{M}$ is given by
\begin{equation}
    \phi_X(\xi) = \left(
        \begin{array}{c}
             \bfD B, \\ \mathbf{Ad}_{\Gamma(\bfD)^{-1}}^{\lor}(\bfb - \gamma^{\lor})\\
        \end{array}
    \right).
\end{equation}
The local coordinate is given by
\begin{equation}
\vartheta(e) = \underbrace{ \begin{bmatrix}
    \bfI_9 & \mathbf{0} &\mathbf{0} \\ 
    \mathbf{0} & \mathbf{0} & \mathbf{I}_6 \\
    \mathbf{0} & \mathbf{I}_{3m} & \mathbf{0}
\end{bmatrix}}_{\mathsf{P}} \begin{pmatrix} \log(e_{\mathbf{B}})^\lor \\ e_{\mathbf{b}} \end{pmatrix},
\end{equation}
where $\mathsf{P}$ is a permutation matrix used to rearrange the components of the error state. The global-local map of ISD-EqF is then
\begin{equation}
    \begin{split}
        \varphi(\xi) &= \vartheta \circ \phi_{\hat{X}^{-1}}(\xi) \\
        &= \mathsf{P} \left(
            \begin{array}{c}
                \log (\bfD\hat\bfD^{-1})^{\lor}, \\
                \mathbf{Ad}_{\Gamma(\hat{\bfD})}^{\lor}(\bfb -\hat{\bfb}) \\
            \end{array}
        \right).
    \end{split}
\end{equation}
The Jacobian $\rfD_{\xi|\hat{\xi}} \varphi (\xi)$ is given by
\begin{equation}
    \begin{split}
         \rfD_{\xi|\hat{\xi}} \varphi (\xi) [\mu] & = \mathsf{P} \left(
        \begin{array}{c}
            (\mu_\bfD \hat{\bfD}^{-1})^{\lor}, \\
            \mathbf{Ad}_{\Gamma(\hat{\bfD})}^{\lor} \mu_{\bfb}
        \end{array}
    \right) \\
    & =  \left(
        \begin{array}{c}
            (\mu_\bfR \hat{\bfR}^{-1})^{\lor}, \mu_\bfv - \mu_\bfR \hat{\bfR}^{-1} \hat{\bfv}, \mu_\bfp - \mu_\bfR \hat{\bfR}^{-1} \hat{\bfp},\\
            \mathbf{Ad}_{\Gamma(\hat{\bfD})}
                ^{\lor} \mu_{\bfb},\\
            \mu_{\bff_1} - \mu_\bfR \hat{\bfR}^{-1} \hat{\bff}_1, \dots, \mu_{\bff_m} - \mu_\bfR \hat{\bfR}^{-1} \hat{\bff}_m\\
        \end{array}
    \right).
    \end{split}
    \label{sup:equ:Jacobian_varphi_isd}
\end{equation}
Combining \eqref{sup:equ:Jacobian_varphi_star_inv} and \eqref{sup:equ:Jacobian_varphi_isd} yields
\begin{equation}
    \begin{split}
           \bfvarepsilon &=   \rfD_{\xi|\hat{\xi}} \varphi (\xi) \cdot  \rfD_{\bfvarepsilon^*|\bf0} \varphi{^*}^{-1}(\bfvarepsilon^*)[\bfvarepsilon^*]\\ &= \left(
            \begin{array}{c}
                \hat{\bfR}\bfvarepsilon^*_\bfR, [\hat{\bfv}]_\times\hat{\bfR}\bfvarepsilon^*_\bfR + \bfvarepsilon^*_\bfv, [\hat{\bfp}]_\times\hat{\bfR}\bfvarepsilon^*_\bfR + \bfvarepsilon^*_\bfp,\\
                \mathbf{Ad}_{\Gamma(\hat{\bfD})}^{\lor} \bfvarepsilon^*_{\mathbf{b}},\\
                {[\hat{\bff}_1]}_\times\hat{\bfR}\bfvarepsilon^*_\bfR + \bfvarepsilon^*_{\mathbf{f}_1}, \dots, [\hat{\bff}_m]_\times\hat{\bfR}\bfvarepsilon^*_\bfR + \bfvarepsilon^*_{\mathbf{f}_m}
            \end{array}
        \right).
    \end{split}
\end{equation}
Rewriting the above equation in matrix form yields the transformation matrix from ESKF to ISD-EqF:
\begin{equation}\underbrace{
    \begin{bmatrix}
        \bfvarepsilon_\bfR \\
        \bfvarepsilon_\bfv \\
        \bfvarepsilon_\bfp \\
        \bfvarepsilon_\bfb \\
        \bfvarepsilon_{\bff_1} \\
        \vdots \\
        \bfvarepsilon_{\bff_m}
    \end{bmatrix}}_{{\begin{smallmatrix}
        \text{ISD-EqF} \\ \text{error-state} \\
    \end{smallmatrix}}} = 
    \underbrace{\begin{bmatrix}
        \hat{\bfR} & \mathbf{0} & \mathbf{0} & \mathbf{0} & \mathbf{0} & \dots & \mathbf{0} \\
        {[\hat{\bfv}]}_\times\hat{\bfR} & \mathbf{I}_3 & \mathbf{0} & \mathbf{0} & \mathbf{0} & \dots & \mathbf{0} \\
        {[\hat{\bfp}]}_\times\hat{\bfR} & \mathbf{0} & \mathbf{I}_3 & \mathbf{0} & \mathbf{0} & \dots & \mathbf{0} \\
        \mathbf{0} & \mathbf{0} & \mathbf{0} & \mathbf{Ad}_{\Gamma(\hat{\bfD})}^{\lor} & \mathbf{0} & \dots & \mathbf{0} \\
        {[\hat{\bff}_1]}_\times\hat{\bfR} & \mathbf{0} & \mathbf{0} & \mathbf{0} & \mathbf{I}_3 & \dots & \mathbf{0} \\
        \vdots & \vdots & \vdots & \vdots & \vdots & \ddots & \vdots \\
        {[\hat{\bff}_m]}_\times\hat{\bfR} & \mathbf{0} & \mathbf{0} & \mathbf{0} & \mathbf{0} & \dots & \mathbf{I}_3
    \end{bmatrix}}_{{\bfT_{\text{ESKF}}^\text{ISD-EqF}}}
      \underbrace{\begin{bmatrix}
        \bfvarepsilon^*_\bfR \\
        \bfvarepsilon^*_\bfv \\
        \bfvarepsilon^*_\bfp \\
        \bfvarepsilon^*_\bfb \\
        \bfvarepsilon^*_{\bff_1} \\
        \vdots \\
        \bfvarepsilon^*_{\bff_m}
    \end{bmatrix}}_{\begin{smallmatrix}
        \text{ESKF} \\ \text{error-state} \\
    \end{smallmatrix}}.   
\end{equation}

\subsection{Transformation from SD-EqF to ISD-EqF}
By transitivity, the transformation from SD-EqF to ISD-EqF is obtained by combining their respective transformations from ESKF:
\begin{equation}
    \begin{split}
    \bfT_{\text{SD-EqF}}^\text{ISD-EqF} &= \bfT_{\text{ESKF}}^\text{ISD-EqF}  \left( \bfT_{\text{ESKF}}^\text{SD-EqF} \right)^{-1} \\
        & = \begin{bmatrix}
        \mathbf{I}_3  & \mathbf{0} & \mathbf{0} & \dots & \mathbf{0} \\
        \mathbf{0}  & \mathbf{I}_{12} & \mathbf{0} & \dots & \mathbf{0} \\
        {[\hat{\bff}_1]_\times} & \mathbf{0} & \mathbf{I}_3 & \dots & \mathbf{0} \\
        \vdots & \vdots  & \vdots & \ddots & \vdots \\
        {[\hat{\bff}_m]_\times}  & \mathbf{0} & \mathbf{0} & \dots & \mathbf{I}_3   
        \end{bmatrix}.
    \end{split}
\end{equation}

\subsection{Transformation from ISD-EqF to T-EqF}
The transformation from SD-EqF to ISD-EqF is identical to that from SD-EqF to T-EqF (Eq. (29) in the main manuscript), i.e.,
\begin{equation}
    \bfT_{\text{SD-EqF}}^\text{ISD-EqF} = \bfT_{\text{SD-EqF}}^\text{T-EqF}.
\end{equation}
Therefore, the transformation from ISD-EqF to T-EqF is
\begin{equation}
    \bfT_\text{ISD-EqF}^\text{T-EqF} =  \bfT_{\text{SD-EqF}}^\text{T-EqF} \left( \bfT_{\text{SD-EqF}}^\text{ISD-EqF} \right)^{-1} = \mathbf{I},
\end{equation}
which indicates that T-EqF and ISD-EqF share the same continuous- and discrete-time Jacobians and the same observability properties.

\section{Unobservable Subspaces of Representative EqFs}
Corollary 2 allows the unobservable subspace of one EqF to be derived from that of another through the transformation, without explicitly constructing the observability matrix. In this section, we derive the unobservable subspaces of SD-EqF, RI-EKF, LI-EKF, ISD-EqF, and T-EqF from that of ESKF, and then analyze the consistency of these filters through the resulting subspaces.

\subsection{Unobservable Subspace of ESKF}
The unobservable subspace of ESKF has been widely studied in the literature \cite{genevaOpenVINSResearchPlatform2020} and is given by
 \begin{equation}
            \spancol{
            \setlength{\arraycolsep}{3pt}
            \mathbf{N}_\text{ESKF}}  =\spancol{ \left[\begin{array}{cc}
            \bfZo_{3\times 3}& -{\hat{\bfR}}^\top \bfg \\
            \bfZo_{3\times 3} & [{\hat{\bfv}}]_\times \bfg\\
            \bfI_3 & [{\hat{\bfp}}]_\times \bfg\\
            \bfZo_{6\times 3} & \bfZo_{6\times 1}\\
            \bfI_3 & [{\hat{\bff}_{1}}]_\times \bfg\\
            \vdots & \vdots \\
            \bfI_3 & [{\hat\bff_{m}}]_\times \bfg\\
            \end{array}\right]},
            \label{sup:equ:unobservable_subspace_eskf}
\end{equation}
where the first three columns correspond to the unobservable directions of global position, and the last column, which depends on the state estimates, corresponds to the unobservable direction associated with global yaw.
Because the last column of the unobservable subspace depends on the state estimates, it may spuriously become observable during estimation, leading to overconfidence.

\subsection{Unobservable Subspace of SD-EqF}
According to Corollary 2, the unobservable subspace of SD-EqF can be obtained by transforming the unobservable subspace of ESKF using the transformation matrix from ESKF to SD-EqF:
\begin{equation}
    \begin{split}
        \spancol{\mathbf{N}_{\text{SD-EqF}}} & = \spancol{\bfT_{\text{ESKF}}^\text{SD-EqF} \mathbf{N}_\text{ESKF} }\\
        & = \spancol{\left[\begin{array}{cc}
        \bfZo_{3\times 3} & -\bfg\\
        \bfZo_{3\times 3}& \bfZo_{3\times1} \\    
        \bfI_3 &  \bfZo_{3\times 1}\\
        \bfZo_{6\times 3} & \bfZo_{6\times 1}\\
         \bfI_3 & [{\hat{\bff}_{1}}]_\times \bfg\\
            \vdots & \vdots \\
        \bfI_3 & [{\hat\bff_{m}}]_\times \bfg\\
        \end{array}\right]}.
    \end{split}
\end{equation}
As in ESKF, the last column of the unobservable subspace of SD-EqF also depends on the state estimates. It may therefore become spuriously observable during filtering, which leads to inconsistency.

\subsection{Unobservable Subspace of RI-EKF}
The unobservable subspace of RI-EKF can be obtained by transforming the unobservable subspace of ESKF using the transformation matrix from ESKF to RI-EKF:
\begin{equation}
    \begin{split}
        \spancol{\mathbf{N}_{\text{RI-EKF}}} & = \spancol{\bfT_{\text{ESKF}}^\text{RI-EKF} \mathbf{N}_\text{ESKF} }\\
        & = \spancol{\left[\begin{array}{cc}
        \bfZo_{3\times 3} & - \bfg\\
        \bfZo_{3\times 3}& \bfZo_{3\times1} \\    
        \bfI_3 &  \bfZo_{3\times 1}\\
        \bfZo_{6\times 3} & \bfZo_{6\times 1}\\
        \bfI_3 &  \bfZo_{3\times 1}\\
        \vdots & \vdots \\
        \bfI_3 &  \bfZo_{3\times 1}\\
        \end{array}\right]}.
    \end{split}
\end{equation}
Notably, the unobservable subspace of RI-EKF does not depend on the state estimates and therefore does not exhibit this inconsistency.

\subsection{Unobservable Subspace of LI-EKF}
The unobservable subspace of LI-EKF can be obtained by transforming the unobservable subspace of ESKF using the transformation matrix from ESKF to LI-EKF:
\begin{subequations}
    \begin{align}
           \spancol{\mathbf{N}_{\text{LI-EKF}}} & = \spancol{\bfT_{\text{ESKF}}^\text{LI-EKF} \mathbf{N}_\text{ESKF} }\\
        & = \spancol{ \left[\begin{array}{cc}
            \bfZo_{3\times 3}& -{\hat{\bfR}}^\top \bfg \\
            \bfZo_{3\times 3} & {\hat{\bfR}}^\top[{\hat{\bfv}}]_\times \bfg\\
            {\hat{\bfR}}^\top& {\hat{\bfR}}^\top[{\hat{\bfp}}]_\times \bfg\\
            \bfZo_{6\times 3} & \bfZo_{6\times 1}\\
            {\hat{\bfR}}^\top & {\hat{\bfR}}^\top[{\hat{\bff}_{1}}]_\times \bfg\\
            \vdots & \vdots \\
            {\hat{\bfR}}^\top & {\hat{\bfR}}^\top[{\hat\bff_{m}}]_\times \bfg\\
            \end{array}\right]} \\
                & = \spancol{ \left[\begin{array}{cc}
            \bfZo_{3\times 3}& -{\hat{\bfR}}^\top\bfg \\
            \bfZo_{3\times 3} & {\hat{\bfR}}^\top[\hat{\bfv}]_\times \bfg\\
            \bfI_3& {\hat{\bfR}}^\top[\hat{\bfp}]_\times \bfg\\
            \bfZo_{6\times 3} & \bfZo_{6\times 1}\\
            \bfI_3 & {\hat{\bfR}}^\top[{\hat{\bff}_{1}}]_\times \bfg\\
            \vdots & \vdots \\
            \bfI_3 & {\hat{\bfR}}^\top [{\hat\bff_{m}}]_\times \bfg\\
            \end{array}\right]}. \label{sup:equ:unobservable_subspace_li2}
    \end{align}
\end{subequations}
Applying column operations to a basis matrix does not change the spanned subspace. In \eqref{sup:equ:unobservable_subspace_li2}, such operations are applied to the first three columns to make them independent of the state estimates. However, no analogous column operation can eliminate the state dependence of the last column. Therefore, the unobservable subspace of LI-EKF also depends on the state estimates, which may again become spuriously observable during filtering and lead to inconsistency.
\subsection{Unobservable Subspaces of ISD-EqF and T-EqF}
Since the transformation from ISD-EqF to T-EqF is an identity transformation, ISD-EqF and T-EqF share the same unobservable subspace. Their unobservable subspace can be obtained by transforming the unobservable subspace of ESKF using the transformation matrix from ESKF to ISD-EqF:
\begin{equation}
    \begin{split}
    \spancol{\mathbf{N}_{\text{T-EqF}}} = \spancol{\mathbf{N}_{\text{ISD-EqF}}} &= \spancol{\bfT_{\text{ESKF}}^\text{ISD-EqF} \mathbf{N}_\text{ESKF} } \\
    & = \spancol{\left[\begin{array}{cc}
        \bfZo_{3\times 3} & -\bfg\\
        \bfZo_{3\times 3}& \bfZo_{3\times1} \\    
        \bfI_3 &  \bfZo_{3\times 1}\\
        \bfZo_{6\times 3} & \bfZo_{6\times 1}\\
         \bfI_3 & \bfZo_{3\times 1}\\
            \vdots & \vdots \\
        \bfI_3 & \bfZo_{3\times 1}\\
        \end{array}\right]}.
    \end{split}
\end{equation}
Because the unobservable subspaces of ISD-EqF and T-EqF are state-independent, these filters do not exhibit this inconsistency.

\section{Propagation Jacobians}
In the main manuscript, we compared five types of EqFs: ESKF, SD-EqF, ISD-EqF, {SOT-EqF}, and T-EqF.
Due to space constraints, the explicit forms of the continuous-time Jacobians $\mathbf{F}$ and $\mathbf{G}$, the state transition matrix $\mathbf{\Phi}(\tau_{i+1}, \tau_i)$, and the discrete-time noise covariance $\mathbf{Q}(\tau_{i+1}, \tau_i)$ were omitted from the main manuscript. However, these expressions are essential both for practical implementation and for identifying the computational bottlenecks in covariance propagation. We therefore provide the detailed derivations here. The Jacobians of RI-EKF and LI-EKF can be derived analogously and are omitted because these filters are not discussed in the main manuscript.

Let $\bfvarepsilon$ denote the error state of an EqF, with dynamics given by
\begin{equation}
    \dot \bfvarepsilon = \bfF \bfvarepsilon + \bfG \bfn,
\end{equation}
where $\bfF \in \mathbb{R}^{N\times N}$ and $\bfG \in \mathbb{R}^{N\times 12}$ ($N=15+3m$) are the state-propagation and noise-propagation Jacobians, respectively. The Gaussian white noise is
\begin{equation}
\bfn = [\bfn_{\omega}^\top, \bfn_a^\top, \bfn_{\text{w}\omega}^\top, \bfn_{\text{w}a}^\top]^\top,
\qquad
\bfQ_c = \operatorname{diag}(\sigma_\omega^2 \bfI_3, \sigma_a^2 \bfI_3, \sigma_{\text{w}\omega}^2 \bfI_3, \sigma_{\text{w}a}^2 \bfI_3),
\end{equation}
where $\bfQ_c$ is its covariance.
In the estimator, this continuous-time error-state model is discretized to propagate uncertainty, yielding
\begin{equation}
    \bfP_{k|k-1} = \bfPhi_{k} \bfP_{k-1|k-1} \bfPhi^\top_{k} + \bfQ_k,
\end{equation}
where the state transition matrix $\bfPhi_k \triangleq \bfPhi(\tau_{q},\tau_0)$ and the accumulated noise matrix $\bfQ_k \triangleq \bfQ(\tau_{q},\tau_0)$ are computed iteratively over $q$ sub-intervals:
\begin{align}
& \bfPhi(\tau_{i+1},\tau_{0}) = \bfPhi(\tau_{i+1},\tau_{i}) \bfPhi(\tau_{i},\tau_{0}), \\
& \begin{aligned}
\bfQ(\tau_{i+1},\tau_0) = & \ \bfPhi(\tau_{i+1},\tau_{i}) \bfQ(\tau_{i},\tau_0) \bfPhi(\tau_{i+1},\tau_{i})^\top \\
& +\bfQ(\tau_{i+1},\tau_i).
\end{aligned}
\end{align}
$\bfPhi(\tau_{i+1},\tau_{i})$ is obtained by solving the linear differential equation
\begin{equation}
       \frac{\text{d}}{\text{d}\tau}{\bfPhi} (\tau,\tau_i)  = \bfF_\tau \bfPhi (\tau,\tau_i), \quad \bfPhi (\tau_i,\tau_i) = \bfI,
        \label{sup:equ:Phi_differential_star}
\end{equation}
and the noise covariance over each sub-interval is given by
\begin{equation}
        \bfQ(\tau_{i+1},\tau_i) = \int_{\tau_i}^{\tau_{i+1}} \bfPhi (\tau_{i+1},\tau) \bfG_\tau \bfQ_c {\bfG_\tau}^\top {\bfPhi} (\tau_{i+1},\tau)^\top \text{d}\tau. 
        \label{sup:equ:Q_differential_star}
\end{equation}

\subsection{Jacobians of ESKF}
The continuous-time Jacobians of ESKF have been widely studied \cite{mourikisMultiStateConstraintKalman2007,genevaOpenVINSResearchPlatform2020} and are given by
\begin{align}
            \bfF &= \left[
    \begin{NiceArray}{ccccc|cccc}[
        margin, 
        nullify-dots, 
    ]
             -[\bfomega_m - \hat{\bfb}_{\omega}]_\times & \bfZo & \bfZo & -\bfI & \bfZo   &\Block{5-4}{\scaleF{\bfZo}_{15\times3m}}\\
                -\hat{\bfR}[\bfa_m - \hat{\bfb}_{a}]_\times & \bfZo & \bfZo & \bfZo & -\hat{\bfR}&&&&\\
                 \bfZo  & \bfI & \bfZo & \bfZo & \bfZo  &&&& \\
                \bfZo & \bfZo & \bfZo & \bfZo & \bfZo  &&&&\\
                \bfZo & \bfZo & \bfZo & \bfZo & \bfZo &&&&\\
           \hline
        \Block{3-5}{\scaleF{\bfZo}_{3m\times 15}} &&&& & \Block{3-4}{\scaleF{\bfZo}_{3m\times 3m}} \\
         &  &  &  &  & &&\\ 
         &  &  &  &  & &&&\\ 
    \end{NiceArray}\right], \\
    \bfG &= 
    \left[
                \begin{NiceArray}{cccc}[
        margin, 
        nullify-dots, 
    ]
                        -\bfI & \bfZo & \bfZo & \bfZo  \\
                       \bfZo & -\hat{\bfR}  & \bfZo & \bfZo  \\
                        \bfZo   & \bfZo & \bfZo  & \bfZo \\
                        \bfZo & \bfZo  & \bfI & \bfZo \\
                        \bfZo & \bfZo &\bfZo & \bfI \\
                        \hline
                        \Block{3-4}{\scaleF{\bfZo}_{3m\times 12}}   \\
                        \\
                        \\
         \end{NiceArray}\right].
\end{align}

Using the IMU integration theory \cite{yangAnalyticCombinedIMU2020}, the discrete-time Jacobians of ESKF can be computed as follows:
\begin{align}
    \bfPhi(\tau_{i+1},\tau_i) & = \left[
        \begin{NiceArray}{ccccc|cccc}[margin, nullify-dots]
          \hat{\bfR}_{\tau_{i+1}}^\top \hat{\bfR}_{\tau_{i}} &\bfZo &\bfZo & - \hat{\bfR}_{\tau_{i+1}}^\top \hat{\bfR}_{\tau_{i}}  \bfJ_l(\Delta \bftheta) \Delta \tau &\bfZo &\Block{5-4}{{\bfZo}_{15\times3m}} &&&\\
          -\hat{\bfR}_{\tau_{i}} [\Delta \bfv]_\times &\bfI &\bfZo & \hat{\bfR}_{\tau_{i}} \boldsymbol{\Xi}_3  & - \hat{\bfR}_{\tau_{i}} \boldsymbol{\Xi}_1 &&&&\\
          -\hat{\bfR}_{\tau_{i}} [\Delta \bfp]_\times &\bfI \Delta \tau  &\bfI & \hat{\bfR}_{\tau_{i}} \boldsymbol{\Xi}_4  & - \hat{\bfR}_{\tau_{i}} \boldsymbol{\Xi}_2 &&&&\\
            \bfZo &\bfZo &\bfZo & \bfI & \bfZo &&&&\\
            \bfZo &\bfZo &\bfZo & \bfZo & \bfI &&&&\\
          \hline 
    \Block{3-5}{\bfZo_{3m\times 15}} &&&& &\Block{3-4}{\bfI_{3m}} \\
       &  &  &  &  & &&\\ 
         &  &  &  &  & &&&\\ 
        \end{NiceArray}
    \right], \\
    \bfQ (\tau_{i+1},\tau_i) & = \bfG_{d}(\tau_{i+1},\tau_i) \bfQ_d{\bfG_{d}(\tau_{i+1},\tau_i)}^\top, \\
    \bfG_d(\tau_{i+1},\tau_i) & = \left[
        \begin{NiceArray}{cccc}[margin, nullify-dots]
             - \hat{\bfR}_{\tau_{i+1}}^\top \hat{\bfR}_{\tau_{i}}  \bfJ_l(\Delta \bftheta) \Delta \tau &\bfZo &\bfZo  &\bfZo \\
             \hat{\bfR}_{\tau_{i}} \boldsymbol{\Xi}_3  &- \hat{\bfR}_{\tau_{i}} \boldsymbol{\Xi}_1 &\bfZo  &\bfZo \\
             \hat{\bfR}_{\tau_{i}} \boldsymbol{\Xi}_4  &- \hat{\bfR}_{\tau_{i}} \boldsymbol{\Xi}_2 &\bfZo  &\bfZo \\
             \bfZo &\bfZo & \bfI \Delta \tau & \bfZo \\
             \bfZo &\bfZo & \bfZo & \bfI \Delta \tau \\
            \hline
    \Block{3-4}{\bfZo_{3m\times 12}} \\
       &  &  &  \\ 
         &  &  &  \\ 
        \end{NiceArray}
    \right], \\
\end{align}
where 
\begin{align}
      \Delta \bfv & = \hat\bfR ^\top_{\tau_i} (\hat{\bfv}_{\tau_{i+1}} - \hat{\bfv}_{\tau_i} + \bfg \Delta \tau), \\
    \Delta \bfp & = \hat\bfR ^\top_{\tau_i} (\hat{\bfp}_{\tau_{i+1}} - \hat{\bfp}_{\tau_i} - \hat{\bfv}_{\tau_i} \Delta \tau + \frac{1}{2} \bfg \Delta \tau^2), \\
    \bfQ_d & = \frac{1}{\Delta \tau} \bfQ_c.
\end{align}
$\boldsymbol{\Xi}_1$, $\boldsymbol{\Xi}_2$, $\boldsymbol{\Xi}_3$, and $\boldsymbol{\Xi}_4$ are given by \cite{yangAnalyticCombinedIMU2020}.
As the above expressions show, the transition matrix $\bfPhi(\tau_{i+1},\tau_i)$ and the noise matrix $\bfQ(\tau_{i+1},\tau_i)$ of ESKF are both block diagonal, with nontrivial entries confined to their top-left $15\times 15$ submatrices.

\subsection{Jacobians of SD-EqF}
According to Corollary 1, the continuous-time Jacobians of SD-EqF can be obtained by applying the transformation matrix $\bfT_{\text{ESKF}}^\text{SD-EqF}$ to ESKF Jacobians. The explicit forms of $\bfF$ and $\bfG$ for SD-EqF are
\begin{align}
            \bfF &= \left[
    \begin{NiceArray}{ccccc|cccc}[
        margin, 
        nullify-dots, 
    ]
             \bfZo & \bfZo & \bfZo  &-\bfI &\bfZo&\Block{5-4}{\scaleF{\bfZo}_{15\times3m}}\\
                {[\bfg]_\times} & \bfZo & \bfZo & \bfZo & \bfZo &&&&\\
                \bfZo & \bfI_3 & \bfZo & -[\hat{\bfp}]_\times &\bfZo&&&& \\
                \bfZo &\bfZo &\bfZo & [\hat{\bfR}(\bfomega_m- \hat{\bfb}_{\omega})]_\times & \bfZo &&&&\\
                \bfZo & \bfZo & \bfZo & [\hat{\bfR}(\bfa_m- \hat{\bfb}_{a})+ [\hat{\bfv}]_\times \hat{\bfR}(\bfomega_m- \hat{\bfb}_{\omega}) +\bfg]_\times & [\hat{\bfR}(\bfomega_m- \hat{\bfb}_{\omega})]_\times&&&&\\
        \hline
     \Block{3-5}{\bfZo_{3m\times 15}} &&&& &\Block{3-4}{\bfZo_{3m\times 3m}} \\
       &  &  &  &  & &&\\ 
         &  &  &  &  & &&&\\ 
    \end{NiceArray}\right], \\
    \bfG &= 
    \left[
                \begin{NiceArray}{cccc}[margin, nullify-dots]
                        -\hat{\bfR} & \bfZo & \bfZo & \bfZo  \\
                        -{[\hat{\bfv}]_\times} \hat{\bfR}  & -\hat{\bfR}  & \bfZo & \bfZo  \\
                        -{[\hat{\bfp}]_\times} \hat{\bfR}   & \bfZo & \bfZo  & \bfZo \\
                        \bfZo & \bfZo  & \hat{\bfR} & \bfZo \\
                        \bfZo & \bfZo & {[\hat{\bfv}]_\times} \hat{\bfR} & \hat{\bfR}  \\
                        \hline 
         \Block{3-4}{\scaleF{\bfZo}_{3m\times 12}}   \\
                        \\
                        \\
                \end{NiceArray}
        \right].
\end{align}

According to Corollary 3, the discrete-time Jacobians of SD-EqF can be obtained by applying the transformation matrix $\bfT_{\text{ESKF}}^\text{SD-EqF}$ to the discrete-time ESKF Jacobians. The explicit forms of $\bfPhi(\tau_{i+1},\tau_i)$ and $\bfQ(\tau_{i+1},\tau_i)$ for SD-EqF are
\begin{align}
    \bfPhi(\tau_{i+1},\tau_i) & = \left[
        \begin{NiceArray}{ccccc|cccc}[margin, nullify-dots]
          \bfI &\bfZo &\bfZo & \bfPhi_{14} &\bfZo &\Block{5-4}{{\bfZo}_{15\times3m}} &&&\\
          {[\bfg]_\times} \Delta \tau &\bfI &\bfZo & \bfPhi_{24} & \bfPhi_{25} &&&&\\
           \frac{1}{2}{[\bfg]_\times} \Delta \tau^2 &\bfI \Delta \tau  &\bfI & \bfPhi_{34}  & \bfPhi_{35} &&&&\\
            \bfZo &\bfZo &\bfZo &  \bfPhi_{44} & \bfZo &&&&\\
            \bfZo &\bfZo &\bfZo & \bfPhi_{54}  & \bfPhi_{55} &&&&\\
            \hline 
    \Block{3-5}{\bfZo_{3m\times 15}} &&&& &\Block{3-4}{\bfI_{3m}} \\
       &  &  &  &  & &&\\ 
         &  &  &  &  & &&&\\ 
        \end{NiceArray}
    \right], \\
    \bfQ (\tau_{i+1},\tau_i) & = \bfG_{d}(\tau_{i+1},\tau_i) \bfQ_d{\bfG_{d}(\tau_{i+1},\tau_i)}^\top, \\
    \bfG_d(\tau_{i+1},\tau_i) & = \left[
        \begin{NiceArray}{cccc}[margin, nullify-dots]
             -  \hat{\bfR}_{\tau_{i}}  \bfJ_l(\Delta \bftheta) \Delta \tau &\bfZo &\bfZo  &\bfZo \\
             - [\hat{\bfv}_{\tau_{i+1}}]_\times  \hat{\bfR}_{\tau_{i}}  \bfJ_l(\Delta \bftheta) \Delta \tau+  \hat{\bfR}_{\tau_{i}} \boldsymbol{\Xi}_3  &- \hat{\bfR}_{\tau_{i}} \boldsymbol{\Xi}_1 &\bfZo  &\bfZo \\
             - [\hat{\bfp}_{\tau_{i+1}}]_\times  \hat{\bfR}_{\tau_{i}}  \bfJ_l(\Delta \bftheta) \Delta \tau+\hat{\bfR}_{\tau_{i}} \boldsymbol{\Xi}_4  &- \hat{\bfR}_{\tau_{i}} \boldsymbol{\Xi}_2 &\bfZo  &\bfZo \\
             \bfZo &\bfZo & \hat{\bfR}_{\tau_{i+1}}  \Delta \tau & \bfZo \\
             \bfZo &\bfZo & [\hat{\bfv}_{\tau_{i+1}}]_\times \hat{\bfR}_{\tau_{i+1}} \Delta \tau & \hat{\bfR}_{\tau_{i+1}}\Delta \tau \\
            \hline
    \Block{3-4}{\bfZo_{3m\times 12}} \\
       &  &  &  \\ 
         &  &  &  \\ 
        \end{NiceArray}
    \right], \\
\end{align}
with 
\begin{align}
    \bfPhi_{14} & = -  \hat{\bfR}_{\tau_{i}}  \bfJ_l(\Delta \bftheta) \hat{\bfR}_{\tau_{i}}^\top \Delta \tau, \\
    \bfPhi_{24} & = -  [\hat\bfv_{\tau_{i+1}}]_\times \hat{\bfR}_{\tau_{i}}  \bfJ_l(\Delta \bftheta) \hat{\bfR}_{\tau_{i}}^\top \Delta \tau + \hat{\bfR}_{\tau_{i}} \boldsymbol{\Xi}_3 \hat{\bfR}_{\tau_{i}}^\top + \hat{\bfR}_{\tau_{i}} \boldsymbol{\Xi}_1 \hat{\bfR}_{\tau_{i}}^\top [\hat{\bfv}_{{\tau}_{i}}]_\times, \\
    \bfPhi_{34} & = -  [\hat\bfp_{\tau_{i+1}}]_\times \hat{\bfR}_{\tau_{i}}  \bfJ_l(\Delta \bftheta) \hat{\bfR}_{\tau_{i}}^\top \Delta \tau + \hat{\bfR}_{\tau_{i}} \boldsymbol{\Xi}_4 \hat{\bfR}_{\tau_{i}}^\top + \hat{\bfR}_{\tau_{i}} \boldsymbol{\Xi}_2 \hat{\bfR}_{\tau_{i}}^\top [\hat{\bfv}_{{\tau}_{i}}]_\times, \\
    \bfPhi_{25} & = - \hat{\bfR}_{\tau_{i}} \boldsymbol{\Xi}_1 \hat{\bfR}_{\tau_{i}}^\top,  \\
    \bfPhi_{35} & = - \hat{\bfR}_{\tau_{i}} \boldsymbol{\Xi}_2 \hat{\bfR}_{\tau_{i}}^\top, \\
    \bfPhi_{44} & = \bfPhi_{55} = \hat{\bfR}_{\tau_{i+1}} \hat{\bfR}_{\tau_{i}}^\top, \\
    \bfPhi_{54} & = [\hat\bfv_{\tau_{i+1}}]_\times \hat{\bfR}_{\tau_{i+1}} \hat{\bfR}_{\tau_{i}}^\top - \hat{\bfR}_{\tau_{i+1}} \hat{\bfR}_{\tau_{i}}^\top [\hat\bfv_{\tau_{i}}]_\times.
\end{align}
Like those of ESKF, the discrete-time Jacobians of SD-EqF are also block diagonal, which naturally enables efficient covariance propagation.

\subsection{Jacobians of ISD-EqF and T-EqF}

Since $\bfT_\text{ISD-EqF}^\text{T-EqF} = \bfI$, ISD-EqF and T-EqF share the same Jacobians. Their continuous-time Jacobians can be obtained by applying the transformation matrix $\bfT_{\text{ESKF}}^\text{ISD-EqF}$ to ESKF Jacobians. The explicit forms of $\bfF$ and $\bfG$ for ISD-EqF and T-EqF are
\begin{align}
            \bfF &= \left[
    \begin{NiceArray}{ccccc|cccc}[
        margin, 
        nullify-dots, 
    ]
             \bfZo & \bfZo & \bfZo  &-\bfI &\bfZo&\Block{5-4}{\scaleF{\bfZo}_{15\times3m}}\\
                {[\bfg]_\times} & \bfZo & \bfZo & \bfZo & \bfZo &&&&\\
                \bfZo & \bfI_3 & \bfZo & -[\hat{\bfp}]_\times &\bfZo&&&& \\
                \bfZo &\bfZo &\bfZo & [\hat{\bfR}(\bfomega_m- \hat{\bfb}_{\omega})]_\times & \bfZo &&&&\\
                \bfZo & \bfZo & \bfZo & [\hat{\bfR}(\bfa_m- \hat{\bfb}_{a})+ [\hat{\bfv}]_\times \hat{\bfR}(\bfomega_m- \hat{\bfb}_{\omega}) +\bfg]_\times& [\hat{\bfR}(\bfomega_m- \hat{\bfb}_{\omega})]_\times&&&&\\
        \hline
        \bfZo & \bfZo & \bfZo & - {[\hat{\bff}_1]_\times}  & \bfZo & \Block{3-4}{\scaleF{\bfZo}_{3m\times 3m}} \\
        \vdots & \vdots & \vdots & \vdots & \vdots & &&&\\ 
         \bfZo & \bfZo & \bfZo & - {[\hat{\bff}_m]_\times} & \bfZo & &&&\\
    \end{NiceArray}\right], \\
    \bfG &= 
    \left[
                \begin{array}{cccc}
                        -\hat{\bfR} & \bfZo & \bfZo & \bfZo \\ -{[\hat{\bfv}]_\times} \hat{\bfR}  & -\hat{\bfR}  & \bfZo & \bfZo  \\
                        -{[\hat{\bfp}]_\times} \hat{\bfR}   & \bfZo & \bfZo  & \bfZo  \\
                        \bfZo & \bfZo  & \hat{\bfR} & \bfZo \\
                        \bfZo & \bfZo & {[\hat{\bfv}]_\times} \hat{\bfR} & \hat{\bfR}  \\
                        \hline
                        -{[\hat{\bff}_1]_\times} \hat{\bfR}& \bfZo & \bfZo & \bfZo  \\
                        \vdots & \vdots & \vdots & \vdots  \\
                        -{[\hat{\bff}_m]_\times} \hat{\bfR}  & \bfZo & \bfZo & \bfZo  \\
                \end{array}
        \right].
\end{align}

Their discrete-time Jacobians are given by
\begin{align}
    \bfPhi(\tau_{i+1},\tau_i) & = \left[
        \begin{NiceArray}{ccccc|cccc}[margin, nullify-dots]
          \bfI &\bfZo &\bfZo & \bfPhi_{14} &\bfZo &\Block{5-4}{{\bfZo}_{15\times3m}} &&&\\
          {[\bfg]_\times} \Delta \tau &\bfI &\bfZo & \bfPhi_{24} & \bfPhi_{25} &&&&\\
           \frac{1}{2}{[\bfg]_\times} \Delta \tau^2 &\bfI \Delta \tau  &\bfI & \bfPhi_{34}  & \bfPhi_{35} &&&&\\
            \bfZo &\bfZo &\bfZo &  \bfPhi_{44} & \bfZo &&&&\\
            \bfZo &\bfZo &\bfZo & \bfPhi_{54}  & \bfPhi_{55} &&&&\\
            \hline 
    \bfZo &\bfZo&\bfZo& [\hat{\bff}_1]_\times \bfPhi_{14} & \bfZo &\Block{3-4}{\bfI_{3m}} \\
      \vdots &\vdots  & \vdots &\vdots  & \vdots & &&\\ 
        \bfZo &\bfZo&\bfZo& [\hat{\bff}_m]_\times \bfPhi_{14} & \bfZo&&\\ 
        \end{NiceArray}
    \right], \\
    \bfQ (\tau_{i+1},\tau_i) & = \bfG_{d}(\tau_{i+1},\tau_i) \bfQ_d{\bfG_{d}(\tau_{i+1},\tau_i)}^\top, \\
    \bfG_d(\tau_{i+1},\tau_i) & = \left[
        \begin{NiceArray}{cccc}[margin, nullify-dots]
             -  \hat{\bfR}_{\tau_{i}}  \bfJ_l(\Delta \bftheta) \Delta \tau &\bfZo &\bfZo  &\bfZo \\
             - [\hat{\bfv}_{\tau_{i+1}}]_\times  \hat{\bfR}_{\tau_{i}}  \bfJ_l(\Delta \bftheta) \Delta \tau+  \hat{\bfR}_{\tau_{i}} \boldsymbol{\Xi}_3  &- \hat{\bfR}_{\tau_{i}} \boldsymbol{\Xi}_1 &\bfZo  &\bfZo \\
             - [\hat{\bfp}_{\tau_{i+1}}]_\times  \hat{\bfR}_{\tau_{i}}  \bfJ_l(\Delta \bftheta) \Delta \tau+\hat{\bfR}_{\tau_{i}} \boldsymbol{\Xi}_4  &- \hat{\bfR}_{\tau_{i}} \boldsymbol{\Xi}_2 &\bfZo  &\bfZo \\
             \bfZo &\bfZo & \hat{\bfR}_{\tau_{i+1}}  \Delta \tau & \bfZo \\
             \bfZo &\bfZo & [\hat{\bfv}_{\tau_{i+1}}]_\times \hat{\bfR}_{\tau_{i+1}} \Delta \tau & \hat{\bfR}_{\tau_{i+1}}\Delta \tau \\
            \hline
         -[\hat{\bff}_1]_\times  \hat{\bfR}_{\tau_{i}}  \bfJ_l(\Delta \bftheta) \Delta \tau & \bfZo & \bfZo &\bfZo  \\
         \vdots & \vdots & \vdots & \vdots  \\
         -[\hat{\bff}_m]_\times  \hat{\bfR}_{\tau_{i}}  \bfJ_l(\Delta \bftheta) \Delta \tau & \bfZo & \bfZo &\bfZo  \\
        \end{NiceArray}
    \right], \\
\end{align}
where $\bfPhi_{14}$, $\bfPhi_{24}$, $\bfPhi_{34}$, $\bfPhi_{25}$, $\bfPhi_{35}$, $\bfPhi_{44}$, and $\bfPhi_{54}$ are the same as those of SD-EqF.

Note that the transition matrix $\bfPhi(\tau_{i+1},\tau_i)$ of ISD-EqF and T-EqF is not block diagonal, and the noise matrix $\bfQ(\tau_{i+1},\tau_i)$ becomes dense. These structures increase the cost of covariance propagation in ISD-EqF and T-EqF, especially when the number of landmarks $m$ is large.

\section{Computational Complexity Analysis}
\label{sup:sec:sup_complexity}

We examine the additional block-matrix operations performed by Naive, TP, and
TC\@.
SD-EqF is used as the auxiliary EqF; choosing ESKF leads to the same FLOP
scaling. Let $d=15$ be the IMU error-state dimension, $m$ the number of
landmarks, $f=3m$ the landmark-state dimension, and $N=d+f=15+3m$ the
dimension of the full error-state. Let $q$ denote the number of IMU propagation
sub-intervals between two consecutive camera frames and $p$ the number of
correction steps performed within one camera frame. In practical VINS
pipelines, $p$ is larger than 1 due to the use of batched visual measurements
and delayed feature initialization
\cite{genevaOpenVINSResearchPlatform2020}. In our subsequent simulations and
experiments, its average value typically lies between 3 and 5. Costs common to
all implementations, such as feature tracking and the standard measurement
update, are excluded.

\subsection{FLOP Convention and Covariance Recursion}

For a dense $a\times b$ matrix multiplied by a dense $b\times c$ matrix, we
count
\begin{equation}
    \mathcal{F}(a,b,c)=ac(2b-1)
    \label{sup:equ:sup_flop_convention}
\end{equation}
FLOPs, including $abc$ multiplications and $ac(b-1)$ additions. A direct
addition of two $a\times c$ matrices costs $ac$ FLOPs.

Between two camera frames, the accumulated transition, accumulated process
noise, and propagated covariance satisfy
\begin{align}
    \bfPhi(\tau_{i+1},\tau_0)
    &=
    \bfPhi(\tau_{i+1},\tau_i)\bfPhi(\tau_i,\tau_0),
    \label{sup:equ:sup_phi_accumulation}\\
    \bfQ(\tau_{i+1},\tau_0)
    &=
    \bfPhi(\tau_{i+1},\tau_i)\bfQ(\tau_i,\tau_0)
    \bfPhi(\tau_{i+1},\tau_i)^\top
    +\bfQ(\tau_{i+1},\tau_i),
    \label{sup:equ:sup_q_accumulation}\\
    \bfP_{k|k-1}
    &=
    \bfPhi_k\bfP_{k-1|k-1}\bfPhi_k^\top+\bfQ_k.
    \label{sup:equ:sup_p_propagation}
\end{align}
For the auxiliary SD-EqF,
\begin{equation}
    \bfPhi=
    \begin{bmatrix}\bfPhi_I&\bf0\\\bf0&\bfI_f\end{bmatrix},
    \qquad
    \bfQ=
    \begin{bmatrix}\bfQ_I&\bf0\\\bf0&\bf0\end{bmatrix}.
\end{equation}
Consequently, \eqref{sup:equ:sup_phi_accumulation} and
\eqref{sup:equ:sup_q_accumulation} only accumulate fixed-size $d\times d$ blocks,
and \eqref{sup:equ:sup_p_propagation} additionally computes
$\bfPhi_I\bfP_{IF}$ at $O(f)$ cost. This auxiliary propagation is the common
baseline.

The target T-EqF transition or the relative transformation used by TC has the
block-sparse form
\begin{equation}
    \bfS=
    \begin{bmatrix}
        \bfS_I&\bf0\\
        \bfD&\bfI_f
    \end{bmatrix},
    \qquad
    \bfD=\bar{\bfD}\bfE^\top,
    \label{sup:equ:sup_sparse_transform}
\end{equation}
where $\bar{\bfD}\in\mathbb{R}^{f\times3}$ and
$\bfE\in\mathbb{R}^{d\times3}$ selects the only nonzero three-column block of
$\bfD$; $\bfS_I\in\mathbb{R}^{d\times d}$ is the corresponding IMU-state
block. Thus, for $\bfX_{IF}\in\mathbb{R}^{d\times f}$, a nominal
$f\times d$ by $d\times f$ product is evaluated as
\begin{equation}
    \bfD\bfX_{IF}
    =
    \bar{\bfD}(\bfE^\top\bfX_{IF}),
    \qquad
    \mathrm{cost}=\mathcal{F}(f,3,f)=5f^2.
    \label{sup:equ:sup_effective_three_column_product}
\end{equation}
Here, $\bfE^\top\bfX_{IF}$ is obtained by selecting the three active rows and
requires no floating-point arithmetic. This effective inner dimension of three
is used throughout the following tables.

\subsection{Naive Implementation: Block-by-Block Cost}
\label{sup:sec:sup_naive_detailed}

For one IMU sub-interval, write the target transition and process noise as
\begin{equation}
    \bfPhi_i^*
    =
    \begin{bmatrix}
        \bfA_i&\bf0\\
        \bfD_i&\bfI_f
    \end{bmatrix},
    \qquad
    \bfQ_i^*
    =
    \begin{bmatrix}
        \bfQ_{II,i}^*&\bfQ_{IF,i}^*\\
        \bfQ_{FI,i}^*&\bfQ_{FF,i}^*
    \end{bmatrix}.
    \label{sup:equ:sup_naive_interval_blocks}
\end{equation}
Here, $\bfA_i\in\mathbb{R}^{d\times d}$ is the IMU-state transition,
$\bfD_i=\bar{\bfD}_i\bfE^\top\in\mathbb{R}^{f\times d}$ is the
IMU-to-landmark coupling, and the subscripts $I$ and $F$ denote the IMU and
landmark partitions, respectively. Accordingly,
$\bfQ_{II,i}^*\in\mathbb{R}^{d\times d}$,
$\bfQ_{IF,i}^*\in\mathbb{R}^{d\times f}$, and
$\bfQ_{FF,i}^*\in\mathbb{R}^{f\times f}$.
The landmark rows of the discrete noise Jacobian contain only one effective
$f\times3$ block, denoted by $\bar{\bfG}_{F,i}$. Thus, the dense
landmark--landmark process-noise block is
\begin{equation}
    \bfQ_{FF,i}^*
    =
    (\bar{\bfG}_{F,i}\bfQ_{\omega})
    \bar{\bfG}_{F,i}^{\top},
    \label{sup:equ:sup_naive_interval_qff}
\end{equation}
where $\bfQ_\omega\in\mathbb{R}^{3\times3}$ is the corresponding gyro-noise
block. Computing $\bar{\bfG}_{F,i}\bfQ_\omega$ costs
$\mathcal{F}(f,3,3)=15f$ FLOPs, and the second multiplication costs
$\mathcal{F}(f,3,f)=5f^2$ FLOPs. The cross block
$\bfQ_{IF,i}^*$ costs only $O(f)$, so constructing $\bfQ_i^*$ contributes
\begin{equation}
    5f^2+O(f)
    \label{sup:equ:sup_naive_interval_noise_cost}
\end{equation}
additional FLOPs per IMU sub-interval.

To expose the cost of accumulating this noise, partition the current
accumulated matrix as
\begin{equation}
    \bar{\bfQ}
    =
    \begin{bmatrix}
        \bar{\bfQ}_{II}&\bar{\bfQ}_{IF}\\
        \bar{\bfQ}_{FI}&\bar{\bfQ}_{FF}
    \end{bmatrix}.
\end{equation}
The bars distinguish the noise accumulated up to the current IMU
sub-interval from the interval noise $\bfQ_i^*$; the superscript $+$ below
denotes the accumulator after incorporating interval $i$.
Substitution of \eqref{sup:equ:sup_naive_interval_blocks} into
\eqref{sup:equ:sup_q_accumulation} gives
\begin{subequations}
\begin{align}
    \bar{\bfQ}_{II}^{+}
    &=
    \bfA_i\bar{\bfQ}_{II}\bfA_i^\top+\bfQ_{II,i}^*,
    \\
    \bar{\bfQ}_{IF}^{+}
    &=
    \bfA_i
    (\bar{\bfQ}_{IF}+\bar{\bfQ}_{II}\bfD_i^\top)
    +\bfQ_{IF,i}^*,
    \\
    \bar{\bfQ}_{FF}^{+}
    &=
    \bar{\bfQ}_{FF}
    +\underbrace{\bfD_i\bar{\bfQ}_{IF}}_{\bfM_1}
    +\underbrace{(\bfD_i\bar{\bfQ}_{IF})^\top}_{\bfM_1^\top}
    +\underbrace{\bfD_i(\bar{\bfQ}_{II}\bfD_i^\top)}_{\bfM_2}
    +\bfQ_{FF,i}^*.
    \label{sup:equ:sup_naive_qff_accumulation}
\end{align}
\end{subequations}
Only the last equation contains quadratic landmark-size operations.
Specifically, $\bfM_1$ and $\bfM_2$ each cost
$\mathcal{F}(f,3,f)=5f^2$ FLOPs, and summing the five $f\times f$ terms costs
$4f^2$ FLOPs. Therefore, one evaluation of
\eqref{sup:equ:sup_q_accumulation} costs
\begin{equation}
    2(5f^2)+4f^2+O(f)=14f^2+O(f).
    \label{sup:equ:sup_naive_q_accumulation_cost}
\end{equation}
The $O(f)$ term includes, for example,
$\bar{\bfQ}_{II}\bfD_i^\top$, whose cost is
$\mathcal{F}(d,3,f)=5df=75f$ FLOPs, and the propagation of the $d\times f$
cross block.

The final covariance propagation
\eqref{sup:equ:sup_p_propagation} has exactly the same block form:
\begin{equation}
\begin{aligned}
    \bfP_{FF}^{-}
    ={}&
    \bfP_{FF}^{+}
    +\underbrace{\bfD_k\bfP_{IF}^{+}}_{\bfN_1}
    +\underbrace{(\bfD_k\bfP_{IF}^{+})^\top}_{\bfN_1^\top}
    +\underbrace{\bfD_k(\bfP_{II}^{+}\bfD_k^\top)}_{\bfN_2}
    +\bfQ_{FF,k}^*.
\end{aligned}
\label{sup:equ:sup_naive_pff_propagation}
\end{equation}
Here, $\bfP^+$ is the posterior covariance at the preceding camera frame,
$\bfP^-$ is the propagated prior covariance, and $\bfD_k$ and
$\bfQ_{FF,k}^*$ are the lower-left transition block and process-noise block
accumulated over all $q$ IMU sub-intervals.
It likewise costs $14f^2+O(f)$, but is evaluated only once per camera frame.
The transition accumulation in \eqref{sup:equ:sup_phi_accumulation} computes
$\bfD_i\bar{\bfA}$ and is only $O(f)$ because $\bfD_i$ has three active
columns.

Table~\ref{sup:tab:sup_naive_block_flops} summarizes where the additional Naive
cost is introduced.

\begin{table}[htbp]
    \caption{Additional block computations of Naive T-EqF. The listed FLOPs
    include all dominant $f\times f$ products and additions.}
    \label{sup:tab:sup_naive_block_flops}
    \centering
    \small
    \setlength{\tabcolsep}{4pt}
    \begin{tabularx}{\linewidth}{p{0.18\linewidth}Xc}
        \toprule
        \textbf{Stage} & \textbf{Additional block computation} &
        \textbf{FLOPs} \\
        \midrule
        Each IMU step:
        $\bfPhi$ accumulation &
        $\bfD_i\bar{\bfA}$ in the lower-left transition block &
        $O(f)$ \\
        Each IMU step:
        $\bfQ_i^*$ construction &
        $\bfQ_{FF,i}^*
        =(\bar{\bfG}_{F,i}\bfQ_\omega)\bar{\bfG}_{F,i}^\top$ &
        $5f^2+O(f)$ \\
        Each IMU step:
        $\bar{\bfQ}$ accumulation &
        $\bfM_1=\bfD_i\bar{\bfQ}_{IF}$,
        $\bfM_2=\bfD_i(\bar{\bfQ}_{II}\bfD_i^\top)$,
        and four $f\times f$ additions &
        $14f^2+O(f)$ \\
        Once per camera:
        $\bfP$ propagation &
        $\bfN_1=\bfD_k\bfP_{IF}^{+}$,
        $\bfN_2=\bfD_k(\bfP_{II}^{+}\bfD_k^\top)$,
        and four $f\times f$ additions &
        $14f^2+O(f)$ \\
        \bottomrule
    \end{tabularx}
\end{table}

Since the first three stages are repeated over $q$ IMU sub-intervals, the total
additional cost per camera frame is
\begin{align}
    \mathrm{FLOPs}_{\mathrm{Naive}}
    &=
    q\left(5f^2+14f^2\right)+14f^2+O(qf)
    \nonumber\\
    &=
    (19q+14)f^2+O(qf)
    \nonumber\\
    &=
    (171q+126)m^2+O(qm).
    \label{sup:equ:sup_naive_flops}
\end{align}
Thus, the implementation that exploits the known zero and identity blocks is
$O(qm^2)$. The $O(qm^3)$ value is only the unstructured upper bound obtained by
treating every $(15+3m)\times(15+3m)$ matrix as fully dense.

\subsection{TP: Correspondence with the Naive Computations}
\label{sup:sec:sup_tp_detailed}

For an estimate $\hat{\xi}$, define the SD-EqF-to-T-EqF coordinate
transformation by
\begin{equation}
    \bfT_{\text{SD-EqF}}^\text{T-EqF}(\hat{\xi})
    =
    \begin{bmatrix}
        \bfI_d&\bf0\\
        \bfL(\hat{\xi})&\bfI_f
    \end{bmatrix},
    \qquad
    \bfL(\hat{\xi})=
    \begin{bmatrix}
        [\hat{\bff}_1]_\times&\bf0_{3\times12}\\
        \vdots&\vdots\\
        [\hat{\bff}_m]_\times&\bf0_{3\times12}
    \end{bmatrix}.
    \label{sup:equ:sup_tp_l_definition}
\end{equation}
Thus, $\bfL\in\mathbb{R}^{f\times d}$ is not an independent propagation
matrix: it is precisely the lower-left block of the state-dependent coordinate
transformation, and $[\bfa]_\times\bfb=\bfa\times\bfb$. Let
$\bfL_{k-1}=\bfL(\hat{\xi}_{k-1|k-1})$ and
$\bfL_{k|k-1}=\bfL(\hat{\xi}_{k|k-1})$. Line~7 of Algorithm~1 then gives
\begin{equation}
\begin{aligned}
    \bfPhi_k^*
    &=
    \bfT_{\text{SD-EqF}}^\text{T-EqF}(\hat{\xi}_{k|k-1})
    \begin{bmatrix}\bfPhi_I&\bf0\\\bf0&\bfI_f\end{bmatrix}
    \bfT_{\text{SD-EqF}}^\text{T-EqF}(\hat{\xi}_{k-1|k-1})^{-1}\\
    &=
    \begin{bmatrix}
        \bfPhi_I&\bf0\\
        \bfC&\bfI_f
    \end{bmatrix},
    \qquad
    \bfC=\bfL_{k|k-1}\bfPhi_I-\bfL_{k-1}\\
    &=
    \begin{bmatrix}
        \bf0_{3\times9}&
        [\hat{\bff}_1]_\times{\bfPhi_I}_{[1:3,10:12]}&
        \bf0_{3\times3}\\
        \vdots&\vdots&\vdots\\
        \bf0_{3\times9}&
        [\hat{\bff}_m]_\times{\bfPhi_I}_{[1:3,10:12]}&
        \bf0_{3\times3}
    \end{bmatrix}.
\end{aligned}
    \label{sup:equ:sup_tp_c_definition}
\end{equation}
During IMU propagation, the global-frame landmark estimates are unchanged, so
$\hat{\bff}_{j,k|k-1}=\hat{\bff}_{j,k-1|k-1}$; the abbreviated
$\hat{\bff}_j$ in \eqref{sup:equ:sup_tp_c_definition} denotes this common value.
Moreover, the first block row of $\bfPhi_I$ is
\begin{equation*}
    [\bfI_3\ \bf0_{3\times6}\
    {\bfPhi_I}_{[1:3,10:12]}\ \bf0_{3\times3}],
\end{equation*}
where
${\bfPhi_I}_{[1:3,10:12]}\in\mathbb{R}^{3\times3}$ denotes rows 1--3 and
columns 10--12. These two relations yield the last equality in
\eqref{sup:equ:sup_tp_c_definition}. Both $\bfL$ and $\bfC$ therefore have only
one active three-column block.

The remaining two lines have equally direct block interpretations. Line~8
transforms the accumulated auxiliary process noise:
\begin{equation}
\begin{aligned}
    \bfQ_k^*
    &=
    \begin{bmatrix}\bfI_d&\bf0\\\bfL_{k|k-1}&\bfI_f\end{bmatrix}
    \begin{bmatrix}\bfQ_I&\bf0\\\bf0&\bf0\end{bmatrix}
    \begin{bmatrix}\bfI_d&\bfL_{k|k-1}^\top\\\bf0&\bfI_f\end{bmatrix}\\
    &=
    \begin{bmatrix}
        \bfQ_I&\bfQ_I\bfL_{k|k-1}^\top\\
        \bfL_{k|k-1}\bfQ_I&
        \bfL_{k|k-1}\bfQ_I\bfL_{k|k-1}^\top
    \end{bmatrix}.
\end{aligned}
    \label{sup:equ:sup_tp_q_relation}
\end{equation}
For Line~9, partition the preceding posterior target covariance as
$\bfP_{k-1|k-1}^*=\left[\begin{smallmatrix}
\bfP_{II}^*&\bfP_{IF}^*\\
\bfP_{FI}^*&\bfP_{FF}^*
\end{smallmatrix}\right]$. The complete block relation is
\begin{equation}
\begin{aligned}
    \bfP_{k|k-1}^*
    &=\bfPhi_k^*\bfP_{k-1|k-1}^*{\bfPhi_k^*}^\top+\bfQ_k^*\\
    &=
    \begin{bmatrix}
        \bfPhi_I\bfP_{II}^*\bfPhi_I^\top &
        \bfPhi_I(\bfP_{IF}^*+\bfP_{II}^*\bfC^\top)\\
        \star &
        \bfP_{FF}^*+\bfC\bfP_{IF}^*
        +(\bfC\bfP_{IF}^*)^\top
        +\bfC(\bfP_{II}^*\bfC^\top)
    \end{bmatrix}
    +\bfQ_k^*,
\end{aligned}
    \label{sup:equ:sup_tp_pff_relation}
\end{equation}
where $\star$ is the transpose of the upper-right block. In particular, the
lower-right block receives the additional term $\bfQ_{FF,k}^*$ from
\eqref{sup:equ:sup_tp_q_relation}.
Equations~\eqref{sup:equ:sup_tp_c_definition}--\eqref{sup:equ:sup_tp_pff_relation}
make explicit which Naive block computations are deferred and evaluated by
TP. TP first
accumulates the $d\times d$ auxiliary transition and noise matrices over the
$q$ IMU steps. It then performs Lines 7--9 of Algorithm~1 once per camera
frame. The additional block operations are listed in
Table~\ref{sup:tab:sup_tp_block_flops}.

\begin{table}[htbp]
    \caption{Additional block computations of Lines 7--9 in TP and their
    correspondence with Naive.}
    \label{sup:tab:sup_tp_block_flops}
    \centering
    \small
    \setlength{\tabcolsep}{3pt}
    \begin{tabularx}{\linewidth}{cXcX}
        \toprule
        \textbf{Line} & \textbf{Additional block computation} &
        \textbf{FLOPs} & \textbf{Same computation in Naive} \\
        \midrule
        7 &
        Construct the lower-left transition block $\bfC$ from $m$ products of
        two $3\times3$ matrices &
        $O(f)$ &
        Transition accumulation, which is also only $O(f)$ \\
        8 &
        $\bfQ_I\bfL^\top$ and
        $\bfL(\bfQ_I\bfL^\top)$; the second product is an
        $f\times3$ by $3\times f$ product &
        $5f^2+O(f)$ &
        Construction of $\bfQ_{FF,i}^*$ in
        \eqref{sup:equ:sup_naive_interval_qff}; Naive repeats it $q$ times \\
        9 &
        $\bfP_{II}^*\bfC^\top$,
        $\bfC\bfP_{IF}^*$, and
        $\bfC(\bfP_{II}^*\bfC^\top)$, followed by four
        $f\times f$ additions &
        $14f^2+O(f)$ &
        Final covariance propagation
        \eqref{sup:equ:sup_naive_pff_propagation}; the same block kernel is also
        used in every Naive noise accumulation \\
        \bottomrule
    \end{tabularx}
\end{table}

The correspondence is therefore explicit. Line~8 of TP performs the same dominant
process-noise product of an $f\times3$ matrix and a $3\times f$ matrix as the Naive construction of
$\bfQ_{FF,i}^*$, but only once after the auxiliary noise has been accumulated.
Line~9 of TP performs the same two quadratic covariance products and four
$f\times f$ additions as the final Naive covariance propagation. Consequently,
\begin{align}
    \mathrm{FLOPs}_{\mathrm{TP}}
    &=
    \underbrace{5f^2}_{\text{Line 8}}
    +
    \underbrace{14f^2}_{\text{Line 9}}
    +O(f)
    \nonumber\\
    &=
    19f^2+O(f)
    =
    171m^2+O(m).
    \label{sup:equ:sup_tp_flops}
\end{align}
The computational benefit of TP is not a cheaper full-state covariance kernel;
it is that this kernel is evaluated once per camera frame instead of in
each of the $q$ IMU noise-accumulation steps.

\subsection{TC: Correction-Stage Block Computation}
\label{sup:sec:sup_tc_detailed}

For TC, propagation is performed entirely in the auxiliary coordinates and
introduces no additional propagation cost. After a correction, the relative
transformation has the following form, where
$\bfL_{k|k}=\bfL(\hat{\xi}_{k|k})$:
\begin{equation}
\begin{aligned}
    \mathcal{T}
    =
    &\bfT_{\text{SD-EqF}}^\text{T-EqF}(\hat{\xi}_{k|k})^{-1}
    \bfT_{\text{SD-EqF}}^\text{T-EqF}(\hat{\xi}_{k|k-1})\\
    =
    &\begin{bmatrix}
        \bfI_d&\bf0\\
        \bfD_\Delta&\bfI_f
    \end{bmatrix},
    \qquad
    \bfD_\Delta
    =\bfL_{k|k-1}-\bfL_{k|k}.
\end{aligned}
    \label{sup:equ:sup_tc_relative_transform}
\end{equation}
Hence, $\bfD_\Delta$ has one active three-column block; its block row for
landmark $i$ is
$[\hat{\bff}_{i,k|k-1}-\hat{\bff}_{i,k|k}]_\times$.
Writing the corrected auxiliary covariance as
$\bfP=\left[\begin{smallmatrix}\bfP_{II}&\bfP_{IF}\\
\bfP_{FI}&\bfP_{FF}\end{smallmatrix}\right]$, the landmark--landmark block of
$\mathcal{T}\bfP\mathcal{T}^\top$ is
\begin{equation}
    \bfP_{FF}
    +\bfD_\Delta\bfP_{IF}
    +(\bfD_\Delta\bfP_{IF})^\top
    +\bfD_\Delta(\bfP_{II}\bfD_\Delta^\top).
\end{equation}
The two quadratic products cost $2(5f^2)$ FLOPs and summing the four
$f\times f$ terms costs $3f^2$ FLOPs. Therefore,
\begin{equation}
    \mathrm{FLOPs}_{\mathrm{TC}}
    =
    p\left(13f^2+O(f)\right)
    =
    117pm^2+O(pm).
    \label{sup:equ:sup_tc_flops}
\end{equation}

\par\medskip
\noindent\begin{minipage}{\linewidth}
    \centering
    \captionof{table}{Total additional covariance FLOPs per camera frame
    relative to the auxiliary EqF.}
    \label{sup:tab:sup_complexity_flops}
    \small
    \setlength{\tabcolsep}{6pt}
    \begin{tabular}{lccc}
        \toprule
        \textbf{Implementation} & \textbf{Propagation} &
        \textbf{Correction} & \textbf{Scaling} \\
        \midrule
        Naive & $(19q+14)f^2+O(qf)$ & 0 & $O(qm^2)$ \\
        TP & $19f^2+O(f)$ & 0 & $O(m^2)$ \\
        TC & 0 & $13pf^2+O(pf)$ & $O(pm^2)$ \\
        \bottomrule
    \end{tabular}
\end{minipage}
\par\medskip

Table~\ref{sup:tab:sup_complexity_flops} shows why TP is typically faster than TC: TP pays the full-state quadratic cost once per camera frame, whereas
TC pays a slightly smaller quadratic cost after every correction ($p$ averages between 3 and 5 in the simulations and experiments).

\section{Relations to Local Landmark Parameterizations and Virtual States}
\label{sup:sec:sup_recent_eqfs}

This section discusses SOT-EqF \cite{vangoor2023eqvio} and TG-EqF \cite{fornasier2025equivariant,fornasier_equivariant_2024} separately from the
representative EqFs above because they do not share the same state description
used in the preceding direct comparisons. The preceding EqFs all represent
landmarks by their global positions and use the standard inertial state without
virtual variables. In contrast, SOT-EqF uses an $\mathbf{SOT}(3)$-based local
landmark parameterization, while TG-EqF augments the inertial model with a
virtual velocity input and a virtual velocity-bias state. A transformation can
be established only after the underlying physical states and local coordinates
have been aligned. This section performs that alignment and then relates
both formulations to ESKF from the transformation viewpoint. The resulting
relations explain their first-order error dynamics and observability
properties, but do not imply that their native higher-order properties are
created by a first-order coordinate transformation.

\subsection{Relation to SOT-EqF}
\label{sup:sec:sup_sot_relation}

The name SOT-EqF is used here rather than EqVIO\@. EqVIO combines the
$\mathbf{SOT}(3)$-based EqF with an output-equivariant visual measurement
construction. For the matched numerical comparisons in the main manuscript,
we extracted and reimplemented only its $\mathbf{SOT}(3)$-based EqF component
in OpenVINS, using the same frontend, sensor data, landmark management, and
experimental settings as the other filters, without imposing output
equivariance. Calling this implementation SOT-EqF therefore distinguishes the
compared filter component from the complete EqVIO system.

SOT-EqF uses a right-invariant IMU error and an $\mathbf{SOT}(3)$-based local
landmark representation \cite{vangoor2023eqvio}. For clarity, we ignore camera
extrinsic calibration, set the IMU--camera transformation to identity, and
consider one visual landmark. Although its native coordinates may be written as
$(\bfR,\bfv,\bfp,\bfb_\omega,\bfb_a,\bfy)$, the transformation theorem requires
the compared filters to use the same underlying physical state. We therefore
express that state using the global landmark position:
\begin{equation}
    \xi=(\bfR,\bfv,\bfp,\bfb_\omega,\bfb_a,\bff),
\end{equation}
where $\bff$ is the landmark position in the world frame. Let $\bfy$ denote the
SOT local landmark coordinate and let
\begin{equation}
    {^I}\bfp_f
    =
    \bfR^\top(\bff-\bfp)
    =
    \varsigma(\bfy).
    \label{sup:equ:sup_sot_landmark_map}
\end{equation}
Thus, $\bfy=\varsigma^{-1}(\bfR^\top(\bff-\bfp))$ is part of the SOT global-local
map rather than a different physical landmark state. Its local landmark error
is $\tilde{\bfy}_{\mathrm{SOT}}=\bfy-\hat{\bfy}$.

Linearizing \eqref{sup:equ:sup_sot_landmark_map} at the estimate gives
\begin{equation}
\begin{aligned}
    \bfR^\top(\bff-\bfp)
    \simeq{}&
    \hat{\bfR}^{\top}(\hat{\bff}-\hat{\bfp})
    +\hat{\bfR}^{\top}
    (\tilde{\bff}_{\mathrm{ESKF}}-\tilde{\bfp}_{\mathrm{ESKF}})
    \\
    &+
    [\hat{\bfR}^{\top}(\hat{\bff}-\hat{\bfp})]_\times
    \tilde{\bftheta}_{\mathrm{ESKF}},
    \\
    \varsigma(\bfy)
    \simeq{}&
    \varsigma(\hat{\bfy})
    +
    \left.\frac{\partial\varsigma}{\partial\bfy}\right|_{\hat{\bfy}}
    \tilde{\bfy}_{\mathrm{SOT}}.
\end{aligned}
\end{equation}
If the local Jacobian is nonsingular, define
\begin{equation}
    \bfA=
    \left(
    \left.\frac{\partial\varsigma}{\partial\bfy}\right|_{\hat{\bfy}}
    \right)^{-1}.
\end{equation}
The landmark-error relation is therefore
\begin{equation}
    \tilde{\bfy}_{\mathrm{SOT}}
    =
    \bfA\!\left(
    \hat{\bfR}^{\top}
    (\tilde{\bff}_{\mathrm{ESKF}}-\tilde{\bfp}_{\mathrm{ESKF}})
    +[\hat{\bfR}^{\top}(\hat{\bff}-\hat{\bfp})]_{\times}
    \tilde{\bftheta}_{\mathrm{ESKF}}
    \right).
    \label{sup:equ:sup_sot_landmark_error_relation}
\end{equation}
Combining the resulting landmark relation with the standard transformation from
ESKF IMU error to the right-invariant IMU error yields
\begin{equation}
\begin{bmatrix}
\tilde{\bftheta}_{\mathrm{SOT}}\\
\tilde{\bfv}_{\mathrm{SOT}}\\
\tilde{\bfp}_{\mathrm{SOT}}\\
\tilde{\bfb}_{\omega}\\
\tilde{\bfb}_{a}\\
\tilde{\bfy}_{\mathrm{SOT}}
\end{bmatrix}
=
\underbrace{
    \left[
        \begin{array}{ccccc|c}
            \hat{\bfR}&\bf0&\bf0&\bf0&\bf0&\bf0\\
{[\hat{\bfv}]}_\times\hat{\bfR}&\bfI_3&\bf0&\bf0&\bf0&\bf0\\
{[\hat{\bfp}]}_\times\hat{\bfR}&\bf0&\bfI_3&\bf0&\bf0&\bf0\\
\bf0&\bf0&\bf0&\bfI_3&\bf0&\bf0\\
\bf0&\bf0&\bf0&\bf0&\bfI_3&\bf0\\
\hline
\bfA[\hat{\bfR}^{\top}(\hat{\bff}-\hat{\bfp})]_\times&
\bf0&-\bfA\hat{\bfR}^{\top}&\bf0&\bf0&\bfA\hat{\bfR}^{\top}
        \end{array}
    \right]
}_{\bfT_{\mathrm{ESKF}}^{\mathrm{SOT-EqF}}}
\begin{bmatrix}
\tilde{\bftheta}_{\mathrm{ESKF}}\\
\tilde{\bfv}_{\mathrm{ESKF}}\\
\tilde{\bfp}_{\mathrm{ESKF}}\\
\tilde{\bfb}_{\omega}\\
\tilde{\bfb}_{a}\\
\tilde{\bff}_{\mathrm{ESKF}}
\end{bmatrix}.
\label{sup:equ:sup_eskf_sot_transformation}
\end{equation}

Applying the unobservable-subspace transformation to the standard ESKF basis
gives
\begin{equation}
\begin{aligned}
    \operatorname{span}(\bfN_{\mathrm{SOT-EqF}})
    &=\operatorname{span}\!\left(
    \bfT_{\mathrm{ESKF}}^{\mathrm{SOT-EqF}}
    \begin{bmatrix}
        \bf0_{3\times3}&-\hat{\bfR}^{\top}\bfg\\
        \bf0_{3\times3}&[\hat{\bfv}]_\times\bfg\\
        \bfI_3&[\hat{\bfp}]_\times\bfg\\
        \bf0_{3\times3}&\bf0_{3\times1}\\
        \bf0_{3\times3}&\bf0_{3\times1}\\
        \bfI_3&[\hat{\bff}]_\times\bfg
    \end{bmatrix}
    \right)\\
    &=\operatorname{span}\!\left(
    \begin{bmatrix}
        \bf0_{3\times3}&-\bfg\\
        \bf0_{3\times3}&\bf0_{3\times1}\\
        \bfI_3&\bf0_{3\times1}\\
        \bf0_{3\times3}&\bf0_{3\times1}\\
        \bf0_{3\times3}&\bf0_{3\times1}\\
        \bf0_{3\times3}&\bf0_{3\times1}
    \end{bmatrix}
    \right),
\end{aligned}
    \label{sup:equ:sup_sot_unobservable}
\end{equation}
which is state independent. This provides a transformation-based explanation of
SOT-EqF consistency.

There are three additional remarks regarding the relation between SOT-EqF and ESKF:
\begin{itemize}
    \item \emph{Generality.} The derivation is not tied to the specific $\mathbf{SOT}(3)$ parameterization.
Its only local requirement is that
$\bfA=(\partial\varsigma/\partial\bfy)^{-1}\in\mathbb{R}^{3\times3}$ exists.
It therefore extends to other locally invertible landmark representations,
including \path{anchored_3d},
\path{anchored_full_inverse_depth}, and
\path{anchored_msckf_inverse_depth}\footnote{\url{https://docs.openvins.com/update-feat.html}}.
    \item \emph{Equivalent physical state.} Rewriting SOT-EqF with the global physical landmark $\bff$ does not change its
estimation outcome or higher-order properties. It only treats
$\bfy=\varsigma^{-1}(\bfR^\top(\bff-\bfp))$ as part of the global-local map, so
that ESKF and SOT-EqF are expressed on the same physical state and the local
transformation in \eqref{sup:equ:sup_eskf_sot_transformation} becomes well defined.

    \item \emph{Scope.} The native EqVIO construction additionally provides output equivariance, which
can reduce higher-order measurement-linearization error. The present
first-order transformation does not reproduce that property; incorporating
output equivariance remains a direction for future work.
\end{itemize}

\subsection{Relation to TG-EqF}
\label{sup:sec:sup_tg_relation}

Reference~\cite{fornasier2025equivariant} identifies tangent-group (TG),
direct-position (DP), and semi-direct (SD) aided-inertial symmetries. The SD case
is already represented by SD-EqF in the main manuscript. DP introduces a
virtual input, whereas TG introduces both a virtual input and a virtual state.
We therefore focus on TG, the most general of the three. Its error state is
18-dimensional, rather than the 15-dimensional error state of the standard ESKF, so ESKF must
first be augmented to the same physical state:
\begin{equation}
    \xi=(\bfR,\bfv,\bfp,\bfb_\omega,\bfb_a,\bfb_v).
\end{equation}
Its dynamics are
\begin{subequations}
\begin{align}
    \dot{\bfR}&=\bfR[\bfomega_m-\bfb_\omega-\bfn_\omega]_\times,\\
    \dot{\bfv}&=\bfR(\bfa_m-\bfb_a-\bfn_a)+\bfg,\\
    \dot{\bfp}&=\bfR(\bfv_m-\bfb_v-\bfn_v)+\bfv,\\
    \dot{\bfb}_\omega&=\bfn_{\mathrm w\omega},\qquad
    \dot{\bfb}_a=\bfn_{\mathrm wa},\qquad
    \dot{\bfb}_v=\bfn_{\mathrm wv}.
\end{align}
\label{sup:equ:sup_augmented_eskf_dynamics}
\end{subequations}
Here $\bfomega_m$, $\bfa_m$, and $\bfv_m$ are the measured angular velocity,
linear acceleration, and virtual linear velocity; $\bfn_\omega$, $\bfn_a$, and
$\bfn_v$ are their measurement-noise terms; and $\bfn_{\mathrm w\omega}$,
$\bfn_{\mathrm wa}$, and $\bfn_{\mathrm wv}$ are bias random-walk noise terms. The
additional variable $\bfb_v$ is the virtual velocity bias required by TG\@. When
$\bfv_m-\bfb_v-\bfn_v=\bf0$, the model reduces to the standard inertial
kinematic model.

Using ESKF attitude error
$\Log_{\SO3}(\hat{\bfR}^{-1}\bfR)$ and additive errors for the other states, the
augmented ESKF error is
\begin{equation}
\bfvarepsilon_{\mathrm{ESKF}}=
\begin{bmatrix}
\Log_{\SO3}(\hat{\bfR}^{-1}\bfR)\\
\bfv-\hat{\bfv}\\ \bfp-\hat{\bfp}\\
\bfb_\omega-\hat{\bfb}_\omega\\
\bfb_a-\hat{\bfb}_a\\ \bfb_v-\hat{\bfb}_v
\end{bmatrix}.
\label{sup:equ:sup_augmented_eskf_error}
\end{equation}
Let
$\bfn=[\bfn_\omega^\top,\bfn_a^\top,\bfn_v^\top,
\bfn_{\mathrm w\omega}^\top,\bfn_{\mathrm wa}^\top,
\bfn_{\mathrm wv}^\top]^\top$. The linearized error dynamics
$\dot{\bfvarepsilon}_{\mathrm{ESKF}}=\bfF_{\mathrm{ESKF}}
\bfvarepsilon_{\mathrm{ESKF}}+\bfG_{\mathrm{ESKF}}\bfn$ use
\begin{equation}
\bfF_{\mathrm{ESKF}}=
\begin{bmatrix}
-\hat{\bfR}^{\top}\dot{\hat{\bfR}}&\bfZo&\bfZo&-\bfI&\bfZo&\bfZo\\
-[\dot{\hat{\bfv}}-\bfg]_\times\hat{\bfR}&\bfZo&\bfZo&\bfZo&-\hat{\bfR}&\bfZo\\
-[\dot{\hat{\bfp}}-\hat{\bfv}]_\times\hat{\bfR}&\bfI&\bfZo&\bfZo&\bfZo&-\hat{\bfR}\\
\bfZo&\bfZo&\bfZo&\bfZo&\bfZo&\bfZo\\
\bfZo&\bfZo&\bfZo&\bfZo&\bfZo&\bfZo\\
\bfZo&\bfZo&\bfZo&\bfZo&\bfZo&\bfZo
\end{bmatrix},
\label{sup:equ:sup_augmented_eskf_F}
\end{equation}
and
\begin{equation}
\bfG_{\mathrm{ESKF}}=
\begin{bmatrix}
-\bfI&\bfZo&\bfZo&\bfZo&\bfZo&\bfZo\\
\bfZo&-\hat{\bfR}&\bfZo&\bfZo&\bfZo&\bfZo\\
\bfZo&\bfZo&-\hat{\bfR}&\bfZo&\bfZo&\bfZo\\
\bfZo&\bfZo&\bfZo&\bfI&\bfZo&\bfZo\\
\bfZo&\bfZo&\bfZo&\bfZo&\bfI&\bfZo\\
\bfZo&\bfZo&\bfZo&\bfZo&\bfZo&\bfI
\end{bmatrix}.
\label{sup:equ:sup_augmented_eskf_G}
\end{equation}
The local transformation to TG-EqF is
\begin{equation}
\bfT_{\mathrm{ESKF}}^{\mathrm{TG}}
=
\begin{bmatrix}
\hat{\bfR}&\bf0&\bf0&\bf0&\bf0&\bf0\\
[\hat{\bfv}]_\times\hat{\bfR}&\bfI&\bf0&\bf0&\bf0&\bf0\\
[\hat{\bfp}]_\times\hat{\bfR}&\bf0&\bfI&\bf0&\bf0&\bf0\\
\bf0&\bf0&\bf0&\hat{\bfR}&\bf0&\bf0\\
\bf0&\bf0&\bf0&[\hat{\bfv}]_\times\hat{\bfR}&\hat{\bfR}&\bf0\\
\bf0&\bf0&\bf0&[\hat{\bfp}]_\times\hat{\bfR}&\bf0&\hat{\bfR}
\end{bmatrix}.
\label{sup:equ:sup_eskf_tg_transformation}
\end{equation}
The top-left block is the transformation from ESKF to SD-EqF for the navigation state,
and the bottom-right block is the adjoint matrix on
$\mathbf{SE}_2(3)$. Substitution into Corollary~1 gives
\begin{subequations}
\begin{align}
    \bfF_{\mathrm{TG}}
    &=
    \dot{\bfT}_{\mathrm{ESKF}}^{\mathrm{TG}}
    (\bfT_{\mathrm{ESKF}}^{\mathrm{TG}})^{-1}
    +
    \bfT_{\mathrm{ESKF}}^{\mathrm{TG}}
    \bfF_{\mathrm{ESKF}}
    (\bfT_{\mathrm{ESKF}}^{\mathrm{TG}})^{-1},
    \\
    \bfG_{\mathrm{TG}}
    &=
    \bfT_{\mathrm{ESKF}}^{\mathrm{TG}}\bfG_{\mathrm{ESKF}}.
\end{align}
\label{sup:equ:sup_tg_jacobian_relation}
\end{subequations}
For completeness, direct substitution gives
\begin{equation}
\bfF_{\mathrm{TG}}=
\begin{bmatrix}
\bfZo&\bfZo&\bfZo&-\bfI&\bfZo&\bfZo\\
[\bfg]_\times&\bfZo&\bfZo&\bfZo&-\bfI&\bfZo\\
\bfZo&\bfI&\bfZo&\bfZo&\bfZo&-\bfI\\
\bfZo&\bfZo&\bfZo&\boldsymbol{\Omega}&\bfZo&\bfZo\\
\bfZo&\bfZo&\bfZo&\mathbf{V}&\boldsymbol{\Omega}&\bfZo\\
\bfZo&\bfZo&\bfZo&\mathbf{P}&\bfZo&\boldsymbol{\Omega}
\end{bmatrix},
\label{sup:equ:sup_tg_F_explicit}
\end{equation}
where
\begin{align}
\boldsymbol{\Omega}&=\dot{\hat{\bfR}}\hat{\bfR}^{\top},\\
\mathbf{V}&=[\dot{\hat{\bfv}}]_\times
+[\hat{\bfv}]_\times\boldsymbol{\Omega}
-\boldsymbol{\Omega}[\hat{\bfv}]_\times,\\
\mathbf{P}&=[\dot{\hat{\bfp}}]_\times
+[\hat{\bfp}]_\times\boldsymbol{\Omega}
-\boldsymbol{\Omega}[\hat{\bfp}]_\times.
\end{align}
The lower-right $3\times3$ block matrix is the corresponding adjoint dynamics
on $\mathfrak{se}_2(3)$. The transformed noise Jacobian is
\begin{equation}
\bfG_{\mathrm{TG}}=
\begin{bmatrix}
-\hat{\bfR}&\bfZo&\bfZo&\bfZo&\bfZo&\bfZo\\
-[\hat{\bfv}]_\times\hat{\bfR}&-\bfI&\bfZo&\bfZo&\bfZo&\bfZo\\
-[\hat{\bfp}]_\times\hat{\bfR}&\bfZo&-\bfI&\bfZo&\bfZo&\bfZo\\
\bfZo&\bfZo&\bfZo&\hat{\bfR}&\bfZo&\bfZo\\
\bfZo&\bfZo&\bfZo&[\hat{\bfv}]_\times\hat{\bfR}&\hat{\bfR}&\bfZo\\
\bfZo&\bfZo&\bfZo&[\hat{\bfp}]_\times\hat{\bfR}&\bfZo&\hat{\bfR}
\end{bmatrix}.
\label{sup:equ:sup_tg_G_explicit}
\end{equation}
Equations~\eqref{sup:equ:sup_tg_F_explicit} and
\eqref{sup:equ:sup_tg_G_explicit} recover TG-EqF error dynamics reported in
\cite{fornasier2025equivariant}.

An apparent sign difference from \cite{fornasier2025equivariant} follows from
the local coordinate chart. The convention used above is
\begin{equation}
\vartheta(e)=
\big(\Log_{\mathbf{SE}_2(3)}(e_{\mathbf A}),e_{\mathbf b}\big),
\end{equation}
whereas the native TG construction uses
\begin{equation}
\vartheta(e)=
\Log_{\mathbf G_{\mathrm{TG}}}\!\left(\phi_{\xi^\circ}^{-1}(e)\right).
\end{equation}
The latter introduces an additional minus sign in the bias error state; hence
the bottom-right block of $\bfT_{\mathrm{ESKF}}^{\mathrm{TG}}$ and the
top-right block of $\bfF_{\mathrm{TG}}$ change sign consistently. This is only
a coordinate convention and does not affect the filter or the transformation
relationship.

\section{Scope of First-Order Equivalence and Nonlinear Limitations}
\label{sup:sec:sup_scope}

This section clarifies the scope of first-order equivalence by stating exactly
which properties are preserved by the proposed transformation and which
properties require a native nonlinear error construction.

\subsection{Preserved First-Order Properties and Nonlinear Limitations}
\label{sup:sec:sup_first_higher_order}

To avoid ambiguity, we distinguish three formulations in the following discussion:
    \begin{itemize}
        \item \emph{Auxiliary EqF}: This computationally convenient baseline, such as the standard ESKF or SD-EqF considered in the manuscript, uses the global-local map $\varphi_\mathrm{A}$ but does not generally preserve the correct observability structure.
        \item \emph{Reference EqF}: This directly derived invariant/EqF formulation uses the global-local map $\varphi_\mathrm{R}$ and preserves the nonlinear error geometry induced by its invariant/equivariant error definition and, when applicable, reduces the estimate dependence of the error dynamics. In the manuscript, ISD-EqF serves as the reference formulation because it combines the advantages of RI-EKF and SD-EqF.
        \item \emph{T-EqF}: This is the proposed transformed EqF, denoted by $\varphi_\mathrm{T}$. It is constructed by applying the first-order transformation induced by $\varphi_\mathrm{R}$ and $\varphi_\mathrm{A}$
        \begin{equation}
            \varphi_\mathrm{T}(\xi)
            =
            \mathbf{T}^{\varphi_\mathrm{R}}_{\varphi_\mathrm{A}}(\hat{\xi})
            \varphi_\mathrm{A}(\xi),
        \end{equation}
        where
        \begin{equation}
            \mathbf{T}^{\varphi_\mathrm{R}}_{\varphi_\mathrm{A}}(\hat{\xi})
            =
            D_{\xi}|_{\hat{\xi}}\varphi_\mathrm{R}(\xi)
            \,
            D_{\epsilon}|_{\mathbf{0}}\varphi_\mathrm{A}^{-1}(\epsilon).
        \end{equation}
        Therefore, $\varphi_\mathrm{T}$ and $\varphi_\mathrm{R}$ have the same value and the same first-order differential at the current estimate:
        \begin{equation}
            \varphi_\mathrm{T}(\hat{\xi})
            =
            \varphi_\mathrm{R}(\hat{\xi})
            =
            \mathbf{0},
            \qquad
            D_{\xi}|_{\hat{\xi}}\varphi_\mathrm{T}(\xi)
            =
            D_{\xi}|_{\hat{\xi}}\varphi_\mathrm{R}(\xi).
        \end{equation}
        Equivalently,
        \begin{equation}
            \varphi_\mathrm{R}(\xi)
            =
            \varphi_\mathrm{T}(\xi)
            +
            O(\|\varphi_\mathrm{A}(\xi)\|^2).
            \label{sup:equ:sup_first_order_remainder}
        \end{equation}
        Hence, T-EqF matches the reference EqF locally to first order, rather than being nonlinearly identical to it.
    \end{itemize}
The conclusion is thus first-order equivalence at the current estimate, not
nonlinear identity of the two maps.

\begin{table}[htbp]
    \centering
    \caption{Scope of the transformation-based equivalence.}
    \label{sup:tab:sup_preserved_properties}
    \small
    \begin{tabularx}{\linewidth}{p{0.23\linewidth}X}
        \toprule
        \textbf{Status} & \textbf{Properties} \\
        \midrule
        Related by the first-order transformation &
        Continuous-time dynamics and measurement Jacobians; covariance
        transformation; local observability matrix; and the basis and dimension
        of the unobservable subspace. With the same nominal trajectory and
        discretization, the corresponding discrete transition and noise
        Jacobians are also preserved to first order. \\
        \addlinespace
        Not guaranteed &
        The nonlinear inverse map used for finite state corrections, the full
        nonlinear error geometry, output equivariance, and higher-order reduction
        of estimate dependence provided by a directly derived EqF. \\
        \bottomrule
    \end{tabularx}
\end{table}

In particular, the reference and transformed filters may have the same
innovation, Kalman gain, and posterior covariance in their local linearized
coordinates, yet return different state estimates because their inverse maps
inject the same finite correction differently:
\begin{subequations}
\begin{align}
    \hat{\xi}^{+}_{\mathrm R}
    &=
    \left.\varphi_{\mathrm R}^{-1}
    (\bfK_{\mathrm R}\tilde{\bfz})\right|_{\hat{\xi}=\hat{\xi}^{-}},
    \\
    \hat{\xi}^{+}_{\mathrm T}
    &=
    \left.\varphi_{\mathrm T}^{-1}
    (\bfK_{\mathrm T}\tilde{\bfz})\right|_{\hat{\xi}=\hat{\xi}^{-}}.
\end{align}
\end{subequations}
The difference begins at second order according to
\eqref{sup:equ:sup_first_order_remainder}.


\subsection{Role of TP and TC.}
TP and TC accelerate covariance computation without changing the underlying
error definition and linearization-error characteristics. 
When T-EqF and the reference EqF are both implemented with TP, they share the same discrete transition, noise, and measurement Jacobians along the same nominal trajectory. The only difference lies in 
their state correction steps (Line~7 of Algorithm~2) remain different. For the
T-EqF, the correction at
\(\hat{\xi}_{k|k-1}\) satisfies
\begin{subequations}
\begin{align}
    \hat{\xi}_{k|k}
    &=
    \varphi_{\mathrm{T}}^{-1}
    \left(
        \mathbf{T}^{\mathrm{R}}_{\mathrm{A}}
        \mathbf{K}_{\mathrm{A}}\tilde{\mathbf{z}}
    \right)
    \bigg|_{\hat{\xi}=\hat{\xi}_{k|k-1}}
    \\
    &=
    \varphi_{\mathrm{A}}^{-1}
    \left(
        \left(\mathbf{T}^{\mathrm{R}}_{\mathrm{A}}\right)^{-1}
        \mathbf{T}^{\mathrm{R}}_{\mathrm{A}}
        \mathbf{K}_{\mathrm{A}}\tilde{\mathbf{z}}
    \right)
    \bigg|_{\hat{\xi}=\hat{\xi}_{k|k-1}}
    \\
    &=
    \varphi_{\mathrm{A}}^{-1}
    \left(
        \mathbf{K}_{\mathrm{A}}\tilde{\mathbf{z}}
    \right)
    \bigg|_{\hat{\xi}=\hat{\xi}_{k|k-1}} .
\end{align}
\label{sup:equ:tc_equivalence}%
\end{subequations}
Thus, the T-EqF state update under TC can use the inverse global-local map of the
auxiliary formulation. This is not equivalent to independently running the
auxiliary EqF, because
\(\mathbf{K}_{\mathrm{A}}\) is computed from the equivalent covariance associated
with the target T-EqF, rather than from the covariance of an independently run
auxiliary filter.
By contrast, when TC is used to accelerate the directly derived reference EqF,
the state correction should still follow the reference global-local map:
\begin{equation}
    \hat{\xi}_{k|k}
    =
    \varphi_{\mathrm{R}}^{-1}
    \left(
        \mathbf{T}^{\mathrm{R}}_{\mathrm{A}}
        \mathbf{K}_{\mathrm{A}}\tilde{\mathbf{z}}
    \right)
    \bigg|_{\hat{\xi}=\hat{\xi}_{k|k-1}} .
\end{equation}
This correction cannot generally be simplified to an auxiliary-form correction,
because
\begin{equation}
    \varphi_{\mathrm{R}}
    \neq
    \mathbf{T}^{\mathrm{R}}_{\mathrm{A}}\varphi_{\mathrm{A}}
\end{equation}
beyond first order. Therefore, TP and TC preserve the computational and
first-order consistency benefits of the target formulation but do not grant
T-EqF the higher-order reduced estimate-dependence or nonlinear state-space
error geometry of a directly derived reference EqF.

\subsection{Planar-Robot Illustration}
\label{sup:sec:sup_planar_example}

Consider the planar state $\xi=(\theta,\bfp)$ with dynamics
\begin{equation}
    \dot{\theta}=\omega_m-n_\omega,
    \qquad
    \dot{\bfp}=\bfC(\theta)(\bfv_m-\bfn_v),
    \label{sup:equ:sup_planar_system}
\end{equation}
where $\bfC(\theta)\in\mathbf{SO}(2)$. The auxiliary additive map and the
reference right-invariant map are
\begin{align}
    \varphi_{\mathrm A}(\xi)
    &=
    \begin{bmatrix}
        \theta-\hat{\theta}\\
        \bfp-\hat{\bfp}
    \end{bmatrix},
    \\
    \varphi_{\mathrm R}(\xi)
    &=
    \Log_{\mathbf{SE}(2)}
    \left(
        \begin{bmatrix}\bfC(\theta)&\bfp\\\bf0^\top&1\end{bmatrix}
        \begin{bmatrix}\bfC(\hat{\theta})&\hat{\bfp}\\\bf0^\top&1\end{bmatrix}^{-1}
    \right).
\end{align}
With
\begin{equation}
    \bfJ=\begin{bmatrix}0&-1\\1&0\end{bmatrix},
    \qquad
    \bfT_k=
    \begin{bmatrix}
        1&\bf0_{1\times2}\\
        -\bfJ\hat{\bfp}_k&\bfI_2
    \end{bmatrix},
    \label{sup:equ:sup_planar_transformation}
\end{equation}
the transformed map is
$\varphi_{\mathrm T}=\bfT_k\varphi_{\mathrm A}$. The reference and transformed
maps have identical first-order discrete transition and noise Jacobians along
the same nominal trajectory, as proved in
Appendix~\ref{sup:app:sup_planar_derivation}. Their inverse maps nevertheless differ
beyond first order. Consequently, the same local Gaussian is mapped to a
banana-shaped state-space distribution by the reference map but to an
ellipsoidal distribution by the transformed map, as shown in
Fig.~\ref{sup:fig:sup_preintegration_bound}.

\begin{figure}[htbp]
    \centering
    \includegraphics[width=0.58\linewidth]{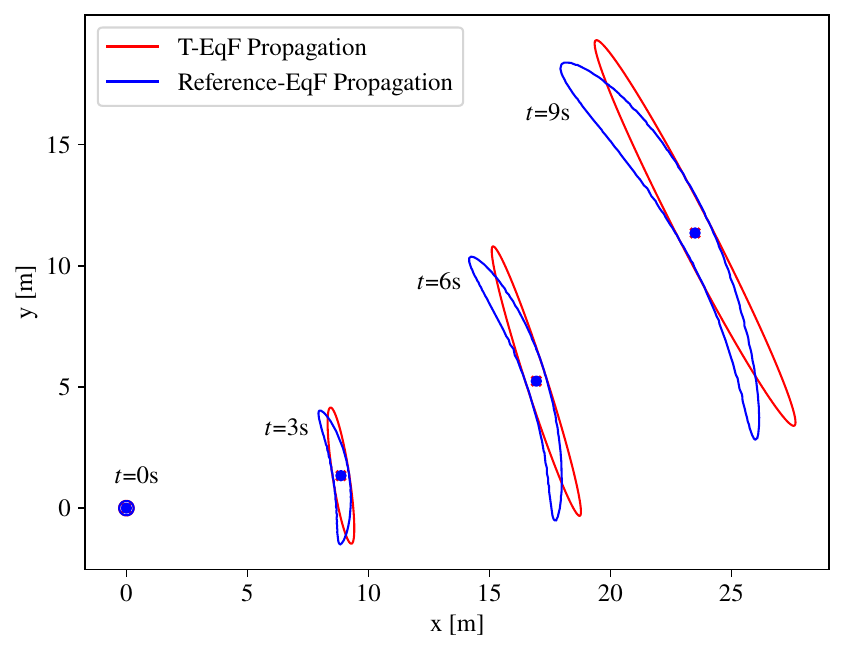}
    \caption{State-space uncertainty for the planar robot with
    $\bfu=[0.1~\mathrm{rad/s},3~\mathrm{m/s},0~\mathrm{m/s}]$,
    $\bfQ=0.01\bfI_3$, and $\delta t=0.1$ s. Both formulations use the same
    initial state and covariance and perform no measurement update. The
    reference invariant map produces the banana-shaped distribution, whereas
    the first-order transformed map produces an ellipsoidal distribution.
    The auxiliary formulation has the same ellipsoidal distribution as the transformed
formulation.
    }
    \label{sup:fig:sup_preintegration_bound}
\end{figure}

\ifdefined\SUPPLEMENTINMAIN
\appendices
\else
\appendix
\fi

\section{Discrete-Time Derivation for the Planar-Robot Example}
\label{sup:app:sup_planar_derivation}

This appendix derives the discrete first-order relationship used in
Section~\ref{sup:sec:sup_planar_example} and then shows explicitly where the
higher-order difference enters.

Assume that $\omega_m$, $\bfv_m$, $n_\omega$, and $\bfn_v$ are constant over
one integration interval of duration $\delta t$. Define
\begin{equation}
    \boldsymbol{\Xi}(\omega_m,\delta t)
    =
    \int_0^{\delta t}\bfC(\omega_m\tau)\,d\tau
    =
    \frac{1}{\omega_m}
    \begin{bmatrix}
        \sin(\omega_m\delta t)&\cos(\omega_m\delta t)-1\\
        1-\cos(\omega_m\delta t)&\sin(\omega_m\delta t)
    \end{bmatrix}.
    \label{sup:equ:sup_planar_xi}
\end{equation}
The nominal increments are
\begin{equation}
    \Delta\hat{\theta}_k=\omega_m\delta t,
    \qquad
    \Delta\hat{\bfp}_k=\boldsymbol{\Xi}\bfv_m,
\end{equation}
and their first-order noise perturbations are
\begin{equation}
    \Delta\tilde{\theta}_k=-\delta t\,n_\omega,
    \qquad
    \Delta\tilde{\bfp}_k=
    \bfs\,n_\omega-\boldsymbol{\Xi}\bfn_v,
    \label{sup:equ:sup_planar_noisy_increments}
\end{equation}
where
\begin{equation}
    \bfs=
    \frac{\boldsymbol{\Xi}-\delta t\bfC(\omega_m\delta t)}
    {\omega_m}\bfv_m.
\end{equation}
The apparent singularity at $\omega_m=0$ is removable:
\begin{equation}
    \lim_{\omega_m\to0}\boldsymbol{\Xi}=\delta t\bfI_2,
    \qquad
    \lim_{\omega_m\to0}\bfs
    =-\frac{\delta t^2}{2}\bfJ\bfv_m.
\end{equation}

The nominal state is integrated by
\begin{subequations}
\begin{align}
    \hat{\theta}_{k+1}&=\hat{\theta}_k+\Delta\hat{\theta}_k,\\
    \hat{\bfp}_{k+1}&=
    \hat{\bfp}_k+\bfC(\hat{\theta}_k)\Delta\hat{\bfp}_k.
\end{align}
\end{subequations}
For the additive auxiliary error
$\bfvarepsilon_{\mathrm A,k}
=[\tilde{\theta}_k,\tilde{\bfp}_k^\top]^\top$, first-order expansion gives
\begin{equation}
    \bfvarepsilon_{\mathrm A,k+1}
    =
    \bfPhi_{\mathrm A,k}\bfvarepsilon_{\mathrm A,k}
    +
    \bfG_{\mathrm A,k}
    \begin{bmatrix}n_\omega\\\bfn_v\end{bmatrix},
    \label{sup:equ:sup_planar_auxiliary_discrete}
\end{equation}
with
\begin{align}
    \bfPhi_{\mathrm A,k}
    &=
    \begin{bmatrix}
        1&\bf0_{1\times2}\\
        \bfJ(\hat{\bfp}_{k+1}-\hat{\bfp}_k)&\bfI_2
    \end{bmatrix},
    \\
    \bfG_{\mathrm A,k}
    &=
    \begin{bmatrix}
        -\delta t&\bf0_{1\times2}\\
        \bfC(\hat{\theta}_k)\bfs&
        -\bfC(\hat{\theta}_k)\boldsymbol{\Xi}
    \end{bmatrix}.
    \label{sup:equ:sup_planar_auxiliary_jacobians}
\end{align}

Now embed the state in $\mathbf{SE}(2)$ as
\begin{equation}
    \bfX_k=
    \begin{bmatrix}
        \bfC(\theta_k)&\bfp_k\\
        \bf0^\top&1
    \end{bmatrix}
\end{equation}
and define the reference error by
\begin{equation}
    \Exp_{\mathbf{SE}(2)}(\bfvarepsilon_{\mathrm R,k})\hat{\bfX}_k
    =
    \bfX_k.
    \label{sup:equ:sup_planar_reference_error}
\end{equation}
Linearizing \eqref{sup:equ:sup_planar_reference_error} at the estimate gives
\begin{equation}
    \bfvarepsilon_{\mathrm R,k}
    =
    \bfT_k\bfvarepsilon_{\mathrm A,k}
    +O(\|\bfvarepsilon_{\mathrm A,k}\|^2),
    \qquad
    \bfT_k=
    \begin{bmatrix}
        1&\bf0_{1\times2}\\
        -\bfJ\hat{\bfp}_k&\bfI_2
    \end{bmatrix}.
    \label{sup:equ:sup_planar_local_relation}
\end{equation}
Substitution of \eqref{sup:equ:sup_planar_auxiliary_discrete} into
\eqref{sup:equ:sup_planar_local_relation} at $k$ and $k+1$ gives
\begin{subequations}
\begin{align}
    \bfPhi_{\mathrm R,k}
    &=
    \bfT_{k+1}\bfPhi_{\mathrm A,k}\bfT_k^{-1}
    =
    \bfI_3,
    \\
    \bfG_{\mathrm R,k}
    &=
    \bfT_{k+1}\bfG_{\mathrm A,k}
    \nonumber\\
    &=
    \begin{bmatrix}
        -\delta t&\bf0_{1\times2}\\
        \bfJ\hat{\bfp}_{k+1}\delta t+
        \bfC(\hat{\theta}_k)\bfs&
        -\bfC(\hat{\theta}_k)\boldsymbol{\Xi}
    \end{bmatrix}.
\end{align}
\label{sup:equ:sup_planar_reference_jacobians}
\end{subequations}
The transformed map
$\varphi_{\mathrm T}=\bfT_k\varphi_{\mathrm A}$ therefore produces exactly the
same first-order discrete Jacobians in
\eqref{sup:equ:sup_planar_reference_jacobians}.

The higher-order distinction follows from the inverse maps. For
$\bfvarepsilon_{\mathrm R}=[\alpha,\boldsymbol{\rho}^\top]^\top$,
\eqref{sup:equ:sup_planar_reference_error} gives
\begin{equation}
    \theta=\hat{\theta}+\alpha,
    \qquad
    \bfp=
    \bfC(\alpha)\hat{\bfp}
    +
    \bfJ_l(\alpha)\boldsymbol{\rho},
    \label{sup:equ:sup_planar_reference_inverse}
\end{equation}
where
\begin{equation}
    \bfJ_l(\alpha)
    =
    \frac{\sin\alpha}{\alpha}\bfI_2
    +
    \frac{1-\cos\alpha}{\alpha}\bfJ,
    \qquad
    \bfJ_l(0)=\bfI_2.
\end{equation}
For the transformed coordinate
$\bfvarepsilon_{\mathrm T}=[\alpha,\boldsymbol{\rho}^\top]^\top$, direct
inversion of $\bfvarepsilon_{\mathrm T}=\bfT\bfvarepsilon_{\mathrm A}$ instead
gives
\begin{equation}
    \theta=\hat{\theta}+\alpha,
    \qquad
    \bfp=
    \hat{\bfp}+\boldsymbol{\rho}+\bfJ\hat{\bfp}\alpha.
    \label{sup:equ:sup_planar_transformed_inverse}
\end{equation}
Equations~\eqref{sup:equ:sup_planar_reference_inverse} and
\eqref{sup:equ:sup_planar_transformed_inverse} have the same first-order expansion,
but the rotation $\bfC(\alpha)$ and the left Jacobian $\bfJ_l(\alpha)$ introduce
different second- and higher-order terms. This proves both the preserved
first-order Jacobians and the distinct nonlinear uncertainty shapes shown in
Fig.~\ref{sup:fig:sup_preintegration_bound}.

\section{Simulation and Experiment Settings}
\label{sup:app:sup_reference_trajectories}

The entries in the simulation RMSE table of the main manuscript are
reference pose trajectories used to generate synthetic measurements, rather
than raw image/IMU datasets. The OpenVINS simulator
\cite{genevaOpenVINSResearchPlatform2020} fits an $\mathbf{SE}(3)$ B-spline to
each time-stamped trajectory and then generates inertial measurements and
visual bearing observations.

\begin{table}[!thp]
    \caption{Basic Simulator Parameters}
    \centering
    \footnotesize
    \setlength\tabcolsep{4.5pt}
    \begin{tabular}[]{cccc}
            \toprule
            { \textbf{Parameter}} & \textbf{Value} & \textbf{Parameter} & \textbf{Value}  \\
            \midrule
            Accel. White Noise    & 2.0e-03  $\text{m} / \text{s}^2 / \sqrt{\text{Hz}}$    &IMU Freq.           & 200 Hz  \\
            Accel. Random Walk    & 3.0e-03  $\text{m} / \text{s}^3 / \sqrt{\text{Hz}}$     & Cam. Freq.          & 10  Hz  \\
            Gyro. White Noise  & 1.7e-04  $\text{rad} / \text{s} / \sqrt{\text{Hz}}$      & Cam. Noise          & 1 pixel \\
             Gyro. Random Walk  & 2.0e-05  $\text{rad} / \text{s}^2 / \sqrt{\text{Hz}}$  &  Cam. Number         & Mono  \\
              Cam. Resolution     & 752 $\times$ 480 pixels & Max Landmarks       & 40 \\
            \bottomrule
    \end{tabular}
    \label{sup:tab:sim_params_sup}
\end{table}

The first four entries in Table~\ref{sup:tab:sim_params_sup} parameterize the
continuous-time process-noise vector in the system model,
\begin{equation}
    \bfn=
    [\bfn_{\omega}^{\top},\bfn_a^{\top},
    \bfn_{\mathrm{w}\omega}^{\top},\bfn_{\mathrm{w}a}^{\top}]^{\top}.
\end{equation}
Specifically, Gyro. White Noise and Accel. White Noise are the noise densities
of $\bfn_{\omega}$ and $\bfn_a$, respectively, which corrupt the raw angular
velocity and acceleration measurements. Gyro. Random Walk and Accel. Random
Walk are the noise densities of $\bfn_{\mathrm{w}\omega}$ and
$\bfn_{\mathrm{w}a}$, respectively, which drive the gyroscope- and
accelerometer-bias random walks. Their units in the table are continuous-time
noise-density units, which the simulator discretizes according to the sensor
sampling interval.

The remaining entries describe the measurement and estimator configuration.
IMU Freq.\ and Cam. Freq.\ are the inertial and visual sampling rates,
respectively. Cam. Noise is the standard deviation of the additive image
measurement noise $\bfepsilon_i$, expressed in pixels, and Cam. Resolution is
the image width and height in pixels. The camera-number and maximum-landmark
parameters are configurable OpenVINS settings. The camera-number parameter
selects the visual input streams: using \texttt{cam0} alone gives a monocular
configuration, whereas using both \texttt{cam0} and \texttt{cam1} gives a
stereo configuration. It therefore determines the visual measurement
configuration and affects the frontend and update costs. Let $m$ denote the
current number of landmark states and $m_{\max}$ its configured upper bound,
so that $m\leq m_{\max}$. The maximum-landmark parameter bounds the maximum
landmark-state dimension and the associated state-estimation cost; the
maintained set changes as landmarks are marginalized and newly initialized.
The simulation configuration in Table~\ref{sup:tab:sim_params_sup} uses monocular
input and $m_{\max}=40$, following the default OpenVINS simulator settings. The EuRoC
MAV experiments instead use the default OpenVINS \textbf{euroc\_mav}
configuration, namely stereo input and $m_{\max}=50$. The self-collected
experiments use the same stereo/50-landmark settings while retaining their
dataset-specific sensor-noise and calibration parameters.

Udel-Gore, Udel-Gara, Udel-Arl-s, and Udel-Neig are UDel/RPNG reference
trajectories distributed with OpenVINS\@. Udel-Gore represents motion around Gore
Hall at the University of Delaware. Udel-Gara, Udel-Arl-s, and Udel-Neig denote
garage, short ARL, and neighborhood-scale trajectories, respectively.
Udel-Gore and Udel-Arl-s mainly represent indoor or campus-like
handheld/platform-carried motion, whereas Udel-Gara and Udel-Neig contain
vehicle-like motion over larger spatial ranges.

TUM-Corr and TUM-Magi are derived from the corridor and magistrale sequences of
the TUM VI benchmark \cite{schubert2018tum} and represent indoor handheld
motion. EuRoC-V11 is based on the EuRoC V1\_01\_easy sequence
\cite{burri2016EuRoCMicroAerial}, which records micro-aerial-vehicle motion and
provides stereo images, 200 Hz IMU measurements, and high-precision ground
truth. Their different spatial scales and motion patterns are visualized in
Fig.~\ref{sup:fig:sup_traj_3d_subplots}.
\begin{figure}[tbp]
    \centering
    \includegraphics[width=0.98\linewidth]{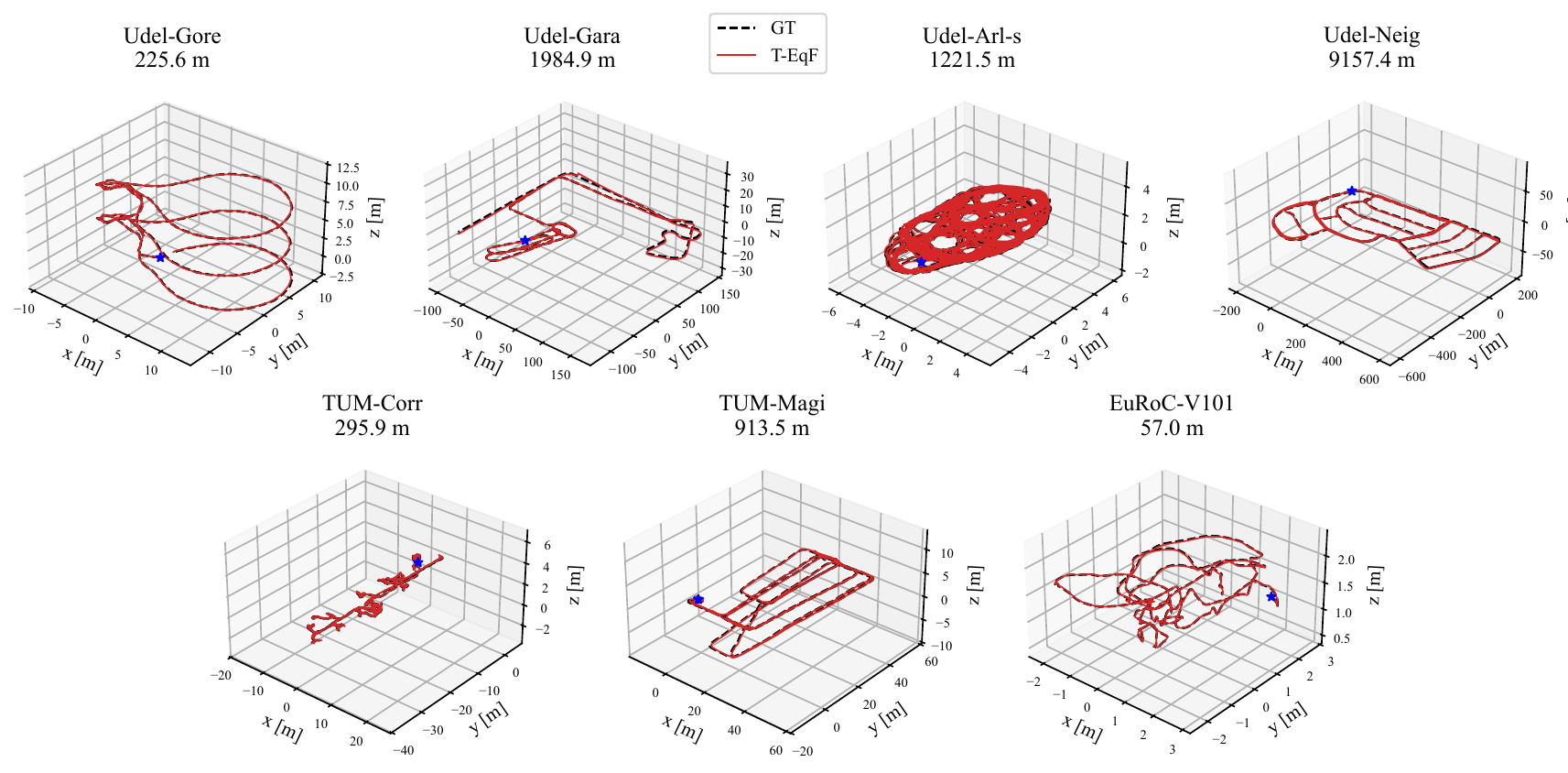}
    \caption{Three-dimensional visualization of the seven reference pose
    trajectories used by the OpenVINS simulator. Axis scales differ across
    subplots so that both short indoor motions and long vehicle-like
    trajectories remain visible.}
    \label{sup:fig:sup_traj_3d_subplots}
\end{figure}

\section{Designing Consistent EqFs with Local Landmark Parameterizations}
Improving estimator consistency is compatible with reducing linearization error. The proposed transformation framework supports local landmark parameterizations and can be used to analyze and improve the consistency of estimators that employ them. To illustrate this compatibility, consider an estimator that combines an ESKF-style IMU error with a local landmark parameterization (e.g., $\mathbf{SOT}(3)$) designed to reduce measurement linearization error. We refer to this estimator as SOT-ESKF and define its error state by
\begin{equation}
    \varphi_{\mathrm{SOT-ESKF}} = \begin{pmatrix}
\Log_{\SO3}(\hat{\bfR}^{\top}\bfR)\\
\bfv-\hat{\bfv}\\
\bfp-\hat{\bfp}\\
\bfb_\omega-\hat{\bfb}_\omega\\
\bfb_a-\hat{\bfb}_a\\
\bfy - \hat{\bfy}
    \end{pmatrix},
    \label{sup:equ:sup_sot_eskf_error}
\end{equation}
where the local landmark parameterization is defined by
\begin{equation}
    \bfy
    =
    \varsigma^{-1}({^I}\bfp_f)
    =
    \varsigma^{-1}\left(\bfR^\top(\bff-\bfp)\right),
\end{equation}
and $\bff$ denotes the landmark position in the global frame. As in
\eqref{sup:equ:sup_eskf_sot_transformation}, the transformation from ESKF
error to SOT-ESKF error is
\begin{equation}
\begin{bmatrix}
\tilde{\bftheta}_{\mathrm{SOT-ESKF}}\\
\tilde{\bfv}_{\mathrm{SOT-ESKF}}\\
\tilde{\bfp}_{\mathrm{SOT-ESKF}}\\
\tilde{\bfb}_{\omega}\\
\tilde{\bfb}_{a}\\
\tilde{\bfy}_{\mathrm{SOT-ESKF}}
\end{bmatrix}
=
\underbrace{
    \left[
        \begin{array}{ccccc|c}
\bfI_3&\bf0&\bf0&\bf0&\bf0&\bf0\\
\bf0&\bfI_3&\bf0&\bf0&\bf0&\bf0\\
\bf0&\bf0&\bfI_3&\bf0&\bf0&\bf0\\
\bf0&\bf0&\bf0&\bfI_3&\bf0&\bf0\\
\bf0&\bf0&\bf0&\bf0&\bfI_3&\bf0\\
\hline
\bfA[\hat{\bfR}^{\top}(\hat{\bff}-\hat{\bfp})]_\times&
\bf0&-\bfA\hat{\bfR}^{\top}&\bf0&\bf0&\bfA\hat{\bfR}^{\top}
        \end{array}
    \right]
}_{\bfT_{\mathrm{ESKF}}^{\mathrm{SOT-ESKF}}}
\begin{bmatrix}
\tilde{\bftheta}_{\mathrm{ESKF}}\\
\tilde{\bfv}_{\mathrm{ESKF}}\\
\tilde{\bfp}_{\mathrm{ESKF}}\\
\tilde{\bfb}_{\omega}\\
\tilde{\bfb}_{a}\\
\tilde{\bff}_{\mathrm{ESKF}}
\end{bmatrix}.
\end{equation}
The resulting unobservable subspace of SOT-ESKF is state-dependent:
\begin{equation}
\begin{aligned}
    \operatorname{span}(\bfN_{\mathrm{SOT-ESKF}})
    &=\operatorname{span}\!\left(
    \bfT_{\mathrm{ESKF}}^{\mathrm{SOT-ESKF}}
    \begin{bmatrix}
        \bf0_{3\times3}&-\hat{\bfR}^{\top}\bfg\\
        \bf0_{3\times3}&[\hat{\bfv}]_\times\bfg\\
        \bfI_3&[\hat{\bfp}]_\times\bfg\\
        \bf0_{3\times3}&\bf0_{3\times1}\\
        \bf0_{3\times3}&\bf0_{3\times1}\\
        \hline
        \bfI_3&[\hat{\bff}]_\times\bfg
    \end{bmatrix}
    \right)\\
    &=\operatorname{span}\!\left(
    \begin{bmatrix}
        \bf0_{3\times3}&-\hat{\bfR}^{\top}\bfg\\
        \bf0_{3\times3}&[\hat{\bfv}]_\times\bfg\\
        \bfI_3&[\hat{\bfp}]_\times\bfg\\
        \bf0_{3\times3}&\bf0_{3\times1}\\
        \bf0_{3\times3}&\bf0_{3\times1}\\
        \hline
        \bf0_{3\times3}&\bf0_{3\times1}
    \end{bmatrix}
    \right),
\end{aligned}
    \label{sup:equ:sup_sot_unobservable2}
\end{equation}

Therefore, SOT-ESKF is inconsistent. To obtain a consistent formulation,
we apply the following transformation to SOT-ESKF error:
\begin{equation}
    \bfT^{\text{T-SOT-ESKF}}_{\text{SOT-ESKF}} =     \left[
        \begin{array}{ccccc|c}
\hat{\bfR}&\bf0&\bf0&\bf0&\bf0&\bf0\\
 {[\hat{\bfv}]_\times} \hat{\bfR}&\bfI_3&\bf0&\bf0&\bf0&\bf0\\
{[\hat{\bfp}]_\times} \hat{\bfR}&\bf0&\bfI_3&\bf0&\bf0&\bf0\\
\bf0&\bf0&\bf0&\bfI_3&\bf0&\bf0\\
\bf0&\bf0&\bf0&\bf0&\bfI_3&\bf0\\
\hline
\bf0&\bf0&\bf0&\bf0&\bf0&\bfI_3\\
        \end{array}
    \right].
\end{equation}
The resulting unobservable subspace of T-SOT-ESKF is state-independent:
\begin{equation}
\begin{aligned}
    \operatorname{span}(\bfN_{\mathrm{T-SOT-ESKF}})
    &=\operatorname{span}\!\left(
    \bfT_{\mathrm{SOT-ESKF}}^{\mathrm{T-SOT-ESKF}}
    \begin{bmatrix}
        \bf0_{3\times3}&-\hat{\bfR}^{\top}\bfg\\
        \bf0_{3\times3}&[\hat{\bfv}]_\times\bfg\\
        \bfI_3&[\hat{\bfp}]_\times\bfg\\
        \bf0_{3\times3}&\bf0_{3\times1}\\
        \bf0_{3\times3}&\bf0_{3\times1}\\
        \hline
        \bf0_{3\times3}&\bf0_{3\times1}
    \end{bmatrix}
    \right)\\
    &=\operatorname{span}\!\left(
    \begin{bmatrix}
        \bf0_{3\times3}&-\bfg\\
        \bf0_{3\times3}&\bf0_{3\times1}\\
        \bfI_3&\bf0_{3\times1}\\
        \bf0_{3\times3}&\bf0_{3\times1}\\
        \bf0_{3\times3}&\bf0_{3\times1}\\
        \hline
        \bf0_{3\times3}&\bf0_{3\times1}
    \end{bmatrix}
    \right),
\end{aligned}
    \label{sup:equ:sup_sot_unobservable3}
\end{equation}
Therefore, T-SOT-ESKF is consistent. This derivation does not require a
specific local landmark parameterization and thus applies to any locally
invertible parameterization (e.g., \path{anchored_3d},
\path{anchored_full_inverse_depth}, or
\path{anchored_msckf_inverse_depth}).

We compare the following EqFs based on local landmark parameterizations:
\begin{itemize}
    \item ESKF: an inconsistent baseline with a global 3D landmark parameterization;
    \item T-EqF: a consistent baseline with a global 3D landmark parameterization;
    \item SOT-EqF: a consistent filter with an $\mathbf{SOT}(3)$ local landmark parameterization;
    \item T-SOT-ESKF: the consistent, transformed version of SOT-ESKF;
    \item T-3D-ESKF: a consistent, transformed ESKF with a local 3D landmark parameterization, designed analogously to T-SOT-ESKF by replacing the $\mathbf{SOT}(3)$ parameterization with a local 3D parameterization;
    \item T-FID-ESKF: a consistent, transformed ESKF with the \path{anchored_full_inverse_depth} parameterization; and
    \item T-MID-ESKF: a consistent, transformed ESKF with the \path{anchored_msckf_inverse_depth} parameterization.
\end{itemize}
The results are reported in Table~\ref{sup:tab:rmse}. All consistent filters improve both orientation and position RMSEs relative to ESKF\@. Compared with T-EqF, the filters employing local landmark parameterizations generally achieve slightly lower RMSEs, while T-SOT-ESKF, T-3D-ESKF, T-FID-ESKF, and T-MID-ESKF perform comparably to SOT-EqF. These results demonstrate that improving consistency does not conflict with reducing linearization error. In particular, the proposed transformation framework can improve observability-related consistency while retaining the local landmark parameterization and its associated reduction in measurement linearization error.
\begin{table}[ht]
    \caption{Orientation (deg) and position (m) RMSEs in simulation.}
    \centering
    \scriptsize
    \setlength\tabcolsep{4pt}
    \setlength\aboverulesep{0pt}
    \setlength\belowrulesep{0pt}
    \begin{tabular}[]{ccc|ccccc}
            \toprule
            {\textbf{Trajectory}} & \textbf{ESKF} & \textbf{T-EqF} & \textbf{SOT-EqF} & \textbf{T-SOT-ESKF} & {\textbf{T-3D-ESKF}} & \textbf{T-MID-ESKF} & \textbf{T-FID-ESKF} \\
            \midrule
            {Udel-Gara} & 0.245 / 1.686 & 0.222 / 1.653 & \textbf{0.219} / 1.568 & \textbf{0.219} / \textbf{1.567} & \textbf{0.219} / 1.572 & \textbf{0.219} / \textbf{1.567} & \textbf{0.219} / \textbf{1.567} \\
            {Udel-Neig} & 1.625 / 10.94 &1.139 / 7.855 & 1.123 / 7.823 & 1.123 / 7.822 & \textbf{1.116} / \textbf{7.786} & 1.127 / 7.841 & 1.129 / 7.855 \\
            {TUM-Magi} &  0.456 / 0.420 & \textbf{0.376} / 0.391 & 0.380 / \textbf{0.372} & 0.380 / \textbf{0.372} & 0.382 / 0.382 & 0.377 / 0.386 & \textbf{0.376} / 0.386 \\
            \bottomrule
    \end{tabular}
    \label{sup:tab:rmse}
\end{table}

\newpage




\ifdefined\SUPPLEMENTINMAIN
\expandafter 
\else
\bibliographystyle{IEEEtran}
\bibliography{ref.bib,ref2.bib}
\fi






\end{document}